\documentclass{article}
\usepackage[accepted]{icml2023}
\usepackage{microtype}
\usepackage{graphicx}
\usepackage{subfigure}
\usepackage{booktabs} % for professional tables

\usepackage{amsmath}
\usepackage{amssymb}
\usepackage{mathtools}
\usepackage{amsthm}

\usepackage[utf8]{inputenc} % allow utf-8 input
\usepackage[T1]{fontenc}    % use 8-bit T1 fonts

\usepackage{url}            % simple URL typesetting
\usepackage{booktabs}       % professional-quality tables
\usepackage{amsfonts}       % blackboard math symbols
\usepackage{nicefrac}       % compact symbols for 1/2, etc.
\usepackage{microtype}      % microtypography
\usepackage{xcolor}         % colors

\usepackage{algorithm}
\usepackage{algorithmic}

\usepackage{wrapfig}
\usepackage{amssymb}
\usepackage[utf8]{inputenc}
\usepackage{caption}
\usepackage{markdown}
\usepackage{multirow}
\usepackage{amsfonts}   
\usepackage{amsmath}
\usepackage{natbib}

\usepackage{url}
\usepackage{tikz}
\usepackage{pgfplots} % used for showing all the figure
\pgfplotsset{
    compat=1.5
} 

\usepgfplotslibrary{groupplots}
\usepackage{CJKutf8}
\usepackage{cancel}
\usepackage{mathtools}
\usepackage{bm}

\usepackage{soul}

\usepackage[textsize=tiny]{todonotes}

\usepackage{graphicx,calc}
\newlength\myheight
\newlength\mydepth
\settototalheight\myheight{Xygp}
\usepackage[backref=page,colorlinks=true,linkcolor=blue,citecolor=green,urlcolor=blue]{hyperref}
\usepackage[capitalize,noabbrev]{cleveref}

\theoremstyle{plain}
\newtheorem{theorem}{Theorem}[section]
\newtheorem{proposition}[theorem]{Proposition}
\newtheorem{lemma}[theorem]{Lemma}

\theoremstyle{definition}

\theoremstyle{remark}

\usepackage{xspace}
\newcommand{\ouralgo}{\textsc{PDML}\xspace}
\newcommand{\ouralgofull}{Policy-adapted Dynamics Model Learning\xspace}

\newcommand{\fh}[1]{{\textcolor{red}{#1}}}

\definecolor{mypink}{RGB}{216,0,115}
\definecolor{myblue}{RGB}{27,161,226}
\usepackage{tikz}

\icmltitlerunning{\hfill Live in the Moment: Learning Dynamics Model Adapted to Evolving Policy\hfill}

\begin{document}

\twocolumn[
\icmltitle{Live in the Moment: Learning Dynamics Model Adapted to Evolving Policy}

\begin{icmlauthorlist}
\icmlauthor{Xiyao Wang}{1}
\icmlauthor{Wichayaporn Wongkamjan}{1}
\icmlauthor{Ruonan Jia}{2}
\icmlauthor{Furong Huang}{1}
\end{icmlauthorlist}

\icmlaffiliation{1}{Department of Computer Science, University of Maryland, College Park, MD 20742, USA}
\icmlaffiliation{2}{Tsinghua University}

\icmlcorrespondingauthor{Xiyao Wang}{xywang@umd.edu}

\icmlkeywords{Machine Learning, ICML}

\vskip 0.3in
]



\printAffiliationsAndNotice{}  % leave blank if no need to mention equal contribution
 % otherwise use the standard text.

\begin{abstract}
Model-based reinforcement learning (RL) often achieves higher sample efficiency in practice than model-free RL by learning a dynamics model to generate samples for policy learning. Previous works learn a dynamics model that fits under the empirical state-action visitation distribution for all historical policies, i.e., the sample replay buffer. However, in this paper, we observe that fitting the dynamics model under the distribution for \emph{all historical policies} does not necessarily benefit model prediction for the \emph{current policy} since the policy in use is constantly evolving over time. The evolving policy during training will cause state-action visitation distribution shifts. We theoretically analyze how this distribution shift over historical policies affects the model learning and model rollouts. We then propose a novel dynamics model learning method, named \textit{Policy-adapted Dynamics Model Learning (PDML)}. PDML dynamically adjusts the historical policy mixture distribution to ensure the learned model can continually adapt to the state-action visitation distribution of the evolving policy. Experiments on a range of continuous control environments in MuJoCo show that PDML achieves significant improvement in sample efficiency and higher asymptotic performance combined with the state-of-the-art model-based RL methods. Our code is released at \href{https://github.com/si0wang/PDML}{https://github.com/si0wang/PDML}.
\end{abstract}

\section{Introduction}
\label{sec_intro}
Recent years have witnessed great successes of Reinforcement Learning (RL) in many complex decision-making tasks, such as robotics \citep{{polydoros2017survey}, {yang2022learning}} and  chess games \citep{{silver2016mastering},{schrittwieser2020mastering}}.
Among RL methods, a wide range of works in model-free RL \citep{{schulman2015trust},{lillicrap2015continuous},{haarnoja2018soft},{fujimoto2018addressing},{hu2021generalizable}} have shown promising performance.
However, model-free methods can be impractical for real-world scenarios \citep{dulac2021challenges}  since massive samples from the real environment are required for policy training, resulting in low sample efficiency.

Model-based RL is considered one of the solutions to improve sample efficiency. 
Most of the model-based RL algorithms first use supervised learning techniques to learn a dynamics model based on the samples obtained from the real environment,  and then use this learned dynamics model to generate massive samples to derive a policy \citep{{luo2018algorithmic},{janner2019trust}}.
Therefore, it is crucial to learn a dynamics model which can accurately simulate the underlying transition dynamics of the real environment since the policy is trained based on the model-generated samples.
If the learned dynamics has a high prediction error, the model-generated samples will be biased, and the policy induced by these samples will be sub-optimal.
To reduce the model prediction error and learn an accurate dynamics model,
some advanced architectures such as model ensemble \citep{{kurutach2018model},{chua2018deep}} and multi-step model \citep{asadi2019combating} have been proposed to improve the multi-step prediction accuracy of the learned dynamics model.
Besides, the idea of a generative adversarial network (GAN) \citep{goodfellow2014generative} is used to design the training process of a dynamics model \citep{{shen2020model}, {eysenbach2021mismatched}} to reduce the distribution mismatch between model-generated samples and real samples.
% Instead of accurately predicting the transitions, the value equivalence model \citep{{farahmand2017value},{grimm2020value},{voelcker2022value}} is proposed to generate the samples that can yield the same Bellman updates like the real samples.
Those previous works mentioned above aim to learn a dynamics model that can fit all historical policies. 
To be precise, when training the dynamics model, they randomly select the training data from the real samples obtained by all historical policies in the replay buffer.
This learned dynamics model needs to adapt to the state-action visitation distribution of all historical policies to obtain a dynamics model that predicts transitions accurately under different policies.
% However, learning such an everywhere accurate model may be unnecessary.

However, since we only use the current newest policy to interact with the learned model to generate samples for policy learning during model rollouts, 
learning such a dynamics model that fits under (highly likely sub-optimal) historical policies may be unnecessary.
Due to the state-action visitation distribution shift during policy updating, the state-action pairs visited by historical policies may not appear in the state-action visitation distribution of the current policy, and vice versa. 
Thus, learning these samples may not benefit model rollouts.
Besides, in many complex tasks, it is hard to predict all samples from all historical policies due to limited model capacity \citep{abbas2020selective},
and as shown later in our paper,  trying to learn every sample from historical policies can even hurt the accuracy when predicting the transitions induced by the current policy.
Therefore, there is an objective mismatch between model learning and model rollouts --- 
model learning tries to fit samples from state-action visitation distribution for all historical policies, whereas model rollouts require accurate prediction of the transitions induced by the current policy.

% Therefore, the learned dynamics model only needs to make accurate predictions for the current policy.
% Learning a global\fhc{change} dynamics model may have a relatively good prediction for all historical policies, but unfortunately, it may have a high prediction error for current policy during model rollouts.

In this paper, we investigate how to learn an accurate dynamics model for model rollouts based on existing samples.
\textbf{(a)} To begin with, we confirm through experiments that although the dynamics model learned by the previous methods has a low overall prediction error on all transitions obtained by historical policies, its prediction error for the current newest policy can still be very high.
This leads to inaccurate model-generated samples which can hurt the sample efficiency and asymptotic performance of the policy. %\joy{does this sound better? which can hurt the sample efficiency and asymptotic performance of the policy.}
\textbf{(b)} We then derive an upper bound of the expected performance gap between the model rollouts and real environment rollouts.
According to this upper bound, we analyze how the distribution of historical policies affects model learning and model rollouts.
The theoretical result suggests that the historical policy distribution used for model learning should be more inclined towards policies that are closer to the current policy rather than a uniform distribution over all historical policies to ensure the model prediction accuracy for model rollouts. 
% Motivated by this insight, we propose a novel model-based RL method named \emph{policy-adapted models for model-based actor-critic (PMAC)}.
\textbf{(c)} Motivated by this insight, we propose a novel dynamics model learning method named \emph{Policy-adapted Dynamics Model Learning (PDML)}.
Instead of learning a dynamics model that fits under a uniform mixture of all historical policies, PDML adjusts the historical policy distribution by reducing the total variation distance between the historical policy mixture and the current policy, then learns a policy-adapted dynamics model according to this adjusted historical policy distribution.
% adapts the learned dynamics model to the evolving policy to ensure the accuracy of the model-generated samples.
\textbf{(d)} We conduct systematic and extensive experiments on a range of continuous control benchmark MuJoCo environments~\citep{todorov2012mujoco}.
Experimental results show that PDML significantly improves the sample efficiency and asymptotic performance of the state-of-the-art model-based RL methods.
% Experiment results show that PMAC achieves higher sample efficiency and better asymptotic performance than previous state-of-the-art model-based RL methods. 

\textbf{Summary of contributions:}
\textbf{(1)} Through detailed experimental results, we establish that learning a dynamics model that fits a uniform mixture of all historical policies may not be accurate enough for model rollouts.
\textbf{(2)} We propose an upper bound of an expected performance gap between the model rollouts and the real environment rollouts, and theoretically analyze how the distribution over historical policies affects model learning and model rollouts.
\textbf{(3)} We propose \emph{Policy-adapted Dynamics Model Learning (PDML)}, which dynamically adjusts the distribution over the historical policy sequence and allows the learned model to continuously adapt to the evolving policy.
\textbf{(4)} Experimental results on a range of MuJoCo environments demonstrate that PDML can achieve significant improvement in sample efficiency and higher asymptotic performance combined with the state-of-the-art model-based RL methods.
\section{Background}\label{sec:bac}
\subsection{Preliminaries}
\textbf{Reinforcement learning.} \quad
Consider a Markov Decision Process (MDP) defined by the tuple $(\mathcal{S}, \mathcal{A}, T, r, \gamma)$, 
where $\mathcal{S}$ is the state space, $\mathcal{A}$ is the action space, and
$T(s^{\prime}|s,a)$ is the transition dynamics in the real world.
The reward function is denoted as $r(s,a)$ and $\gamma$ is the discount factor.
Reinforcement learning aims to find an optimal policy $\pi$ which can maximize the expected sum of discounted rewards

\vspace{-20pt}
\begin{equation}
\label{eq_BRLgoal}
\pi = \mathop{\arg\!\max}_{\pi}\mathbb{E}_{{s_{t}\sim T(\cdot|s_{t-1},a_{t-1})} \atop{a_t \sim \pi(a|s_t)}}\left[\sum_{t=0}^{\infty}\gamma^t r(s_t, a_t)\right].
\end{equation}
\vspace{-15pt}

In model-based RL, the transition dynamics $T$ in the real world is unknown, and we aim to construct a model $\hat{T}(s^{\prime}|s,a)$ of transition dynamics and use it to improve the policy. 
In this paper, we concentrate on the Dyna-style \citep{sutton1990integrated} model-based RL, which uses the learned dynamics model to generate samples and train the policy. 
% Model-based RL in this paper refers to Dyna-style. \joy{this is repeated to the prev sentence}% model-based RL. 
% \joy{give a fullstop and connect to the last sentence with other words might be better, thiis is a bit long sentence}

\textbf{Policy mixture.} \quad 
During policy learning, we consider the historical policies at iteration step $k$ as a historical policy sequence $\Pi^k = \{\pi_1, \pi_2, ..., \pi_k\}$.
For each policy in the policy sequence, we denote its state-action visitation distribution as $\rho^{\pi_i}(s,a)$, and the policy mixture distribution over the policy sequence as $\bm{w}^k=[w_1^k\ldots,w_k^k]$.
Then the state-action visitation distribution of the policy mixture $\pi_{\text{mix}, k} = (\Pi^k, \bm{w}^k)$ is $\rho^{\pi_{\text{mix}, k}}(s,a) = \sum^{k}_{i=1} w^k_i \rho^{\pi_i}(s,a)$ \citep{{hazan2019provably}, {zhang2021made}}.

\subsection{Dynamics Model Learning in Model-based RL} \label{subsec_bac_glo}
Learning a dynamics model is the most crucial part of model-based RL since the ground-truth transition dynamics is unknown and the policy must be updated based on the samples generated by the learned dynamics model.
Previous works learn the dynamics model by randomly selecting training data from the samples obtained by the historical policy sequence $\Pi^k$, which means the distribution of policy mixture is a random distribution: $w^k_i = \frac{1}{k}$. 
The learned dynamics model is trained based on the following state-action visitation distribution

\vspace{-20pt}
\begin{equation}
\label{global_density}
\rho^{\pi_{\text{mix}, k}}(s,a) = \sum^{k}_{i=1} \frac{1}{k} \rho^{\pi_i}(s,a).
\end{equation}

\begin{figure}[!htbp]
\centering
\subfigure[HalfCheetah]{
\includegraphics[width=0.2\textwidth]{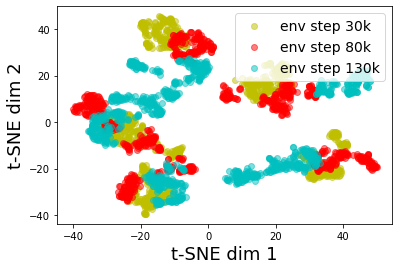}
\label{dis_c}
}
\hfil
\subfigure[Hopper]{
\includegraphics[width=0.2\textwidth]{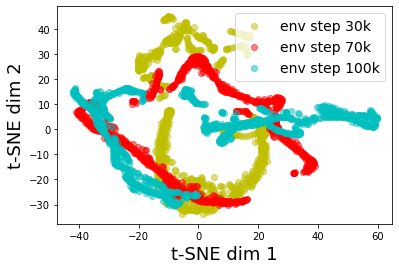}
\label{dis_d}
}
\hfil
\subfigure[HalfCheetah]{
\resizebox{0.215\textwidth}{!}{\input{exp_results/PMAC_half_local_error}}
\label{dis_a}}
\hfil
\subfigure[Hopper]{
\resizebox{0.235\textwidth}{!}{\input{exp_results/PMAC_hopper_local_error}}
\label{dis_b}}
\vspace{-10pt}
\caption{
\textbf{(a)} and \textbf{(b)}: visualization of the state-action visitation distribution of different historical policies and the current policy using t-SNE. Env step 130k and env step 100k are the current policy. More details are shown in Appendix~\ref{app: more distribution shift}.
\textbf{(c)} and \textbf{(d)}: the overall error curves and current error curves of MBPO on HalfCheetah and Hopper, respectively.
}
\vspace{-20pt}
\label{Fig_delta_distribution}
\end{figure}

This model tries to fit all the samples obtained by sampling the state-action visitation distribution corresponding to all policies in the historical policy sequence,
so the learned dynamics model is (hopefully) able to predict the transition for any state-action input.

% However, it is difficult to learn such a dynamics model in many real-world tasks since the number of real samples from the environment is very limited and obtaining real samples is expensive.
However, as shown in Figure~\ref{dis_c} and \ref{dis_d}, since the policy is constantly evolving, the state-action visitation distribution of historical policies may have a huge shift from the current policy.
There is little overlap between the state-action visitation distribution of policies at different environment steps.
The state-action pairs visited by historical policies may not appear in the state-action visitation distribution of the current policy. 
During model rollouts, we only use the current policy to interact with the learned dynamics model to generate samples.
Thus, learning these samples may not benefit model rollouts.
When the model capacity is not large enough, learning these samples may even be detrimental to the learning of the samples collected by the current policy.

We conduct an experiment using a state-of-the-art model-based RL method called MBPO \citep{janner2019trust} on four MuJoCo \citep{todorov2012mujoco} environments HalfCheetah, Hopper, Walker2d, and Ant.
MBPO first trains a model based on the real samples and then uses the model to roll out multiple samples for policy learning.
% name: model learning error curves and the model rollouts error curves
The architecture of the dynamics model is a 4-layer neural network with a hidden size of 200, which is a very common architecture used in many recent model-based methods \citep{{yao2021sample}, {froehlich2022onpolicy}, {li2022gradient}}.
We present the overall error curves and the current error curves during learning steps on HalfCheetah and Hopper in Figure~\ref{dis_a} and \ref{dis_b}. 
Here the overall error means the model prediction error for all historical policies during training.
It is evaluated on an evaluation dataset which contains 1000 $\times N$ samples from the real environment.
$N$ is the number of historical policies in the historical policy sequence.
The current error is the model prediction error for the current policy, which is evaluated using L2 error on the 1000 samples obtained by the current policy from the real environment.
The error curves for more environments can be found in Appendix~\ref{app: more_local}.

From Figure~\ref{dis_a} and \ref{dis_b}, we observe that there is a gap between the overall error and the current error.
This means although the agent can learn a dynamics model which is good enough for all samples obtained by historical policies, 
this is at the expense of the prediction accuracy for the samples induced by current policy. 
Since we only use the current policy during model rollouts, this will lead to inaccurate model-generated samples and misleading policy learning.
To demonstrate that this error gap is not caused by the model not having converged on the recent data, we also conduct another experiment.
We checkpoint the replay buffer and model at multiple points during training, then train the dynamics model for a long time until convergence at these checkpointed locations, and test the prediction error on newly generated data.
We find that even when the dynamics model has converged on all data, the prediction error on newly generated data does not reduce obviously.
Experiment results and more details can be found in Appendix~\ref{app: converge_model}.
% \fhst{This further proves our claim.}}

Therefore, learning a dynamics model that adapts the state-action visitation distribution for all historical policies, in other words, a random historical policy mixture distribution used for model learning, is not the most efficient way for model-based RL (especially for task-specific problems).
In the next section, we will analyze how the policy mixture distribution affects the performance of model-based RL.

% \section{Dynamics Model Adapted to the Evolving Policy}\label{section:motivation}

% As shown in the last section, the model-generated samples are always generated by the current policy through interacting with the learned dynamics model.
% Therefore, to ensure the accuracy of model-generated samples, we present a dynamics model which aims to adapt the state-action visitation distribution of the current policy.
% To motivate our method, we first analyze how the distribution of policy mixture affects the model rollout performance theoretically. 
% We then analyze the state-action visitation distribution of the historical policies to guide us on how to adjust the distribution of policy mixture and learn the policy-adapted dynamics model.
% Then we give a model-based RL algorithm which combined with a practical local dynamics model implementation

\section{Performance Gap Influenced by Policy Mixture Distribution}

% As mentioned in Section \ref{subsec_bac_glo}, learning a dynamics model that adapts the state-action visitation distribution for all historical policies is not efficient for model-based RL.
In this section, we provide a theoretical analysis of how the policy mixture distribution affects the performance of model-based RL.
First, we derive a theorem that upper bounds the performance gap between the real environment rollouts and the model rollouts under any current policy $\pi$.

\begin{theorem}\label{main_theory}
Given the historical policy mixture $\pi_{\text{mix}, k} = (\Pi^k, \bm{w}^k)$ at iteration step $k$, we denote $\xi_{\rho_i} = D_{TV}(\rho^{\pi}_{T}(s,a) || \rho^{\pi_i}_{T}(s,a))$ as the state-action visitation distribution shift and $\xi_{\pi_i} = \mathbb{E}_{s \sim v^{\pi_{\text{mix}}}_{\hat{T}}} \left[ D_{TV}(\pi(a|s)|| \pi_i(a|s)) \right]$ as the policy distribution shift between the historical policy $\pi_i$ and current policy $\pi$ respectively, where $v^{\pi_{\text{mix}}}_{\hat{T}}$ is the state visitation distribution of the policy mixture under the learned dynamics model $\hat{T}$.
$r_{\text{max}}$ is the maximum reward the policy can get from the real environment, $\gamma$ is the discount factor, and ${\rm Vol}(\mathcal{S})$ is the volume of state space. 
Then the performance gap between the real environment rollout $J(\pi, T)$ and the model rollout $J(\pi, \hat{T})$ can be bounded as:
\begin{equation}\label{main_theory_eq}
\resizebox{.9\hsize}{!}{$
\begin{aligned}
J(\pi, T)  - J(\pi, \hat{T}) \leq & 2 \gamma r_{\text{max}} \mathbb{E}_{(s,a) \sim \rho^{\pi}_{T} }[D_{TV} (T(s^\prime|s,a) || \hat{T}(s^\prime|s,a))] \\
&  + r_{\text{max}} \sum^{k}_{i=1} w^k_i (\gamma {\rm Vol}(\mathcal{S}) \xi_{\rho_i} + 2 \xi_{\pi_i}) \\
& + 2 r_{\text{max}} D_{TV}(\rho^{\pi_{\text{mix}}}_{\hat{T}}(s,a) || \rho^{\pi}_{\hat{T}}(s,a)) 
\end{aligned}$}
\end{equation}
\end{theorem}

\begin{proof}
 See Appendix~\ref{proof_main}.
\end{proof}

\textbf{Remarks.} \\
%We have several remarks according to this bound.
\textbf{(1)} The first term is about \textbf{model prediction error}. 
This term suggests that the model needs to adapt to the state-action visitation distribution of the \emph{current policy} to reduce the model prediction error, 
since this term is the expectation of prediction error of the learned dynamics model $\hat{T}$ under the current policy state-action visitation distribution $\rho^{\pi}_{T}$.\\
\textbf{(2)} The second term shows the effect of the policy mixture distribution on \textbf{model rollout}. 
This item contains two distribution shifts: \textbf{(2a)} state-action visitation distribution shift $\xi_{\rho_i}$ and \textbf{(2b)} policy distribution shift $\xi_{\pi_i}$ between the historical policy and current policy. 
It should be noted that $\xi_{\rho_i}$ is induced by $\xi_{\pi_i}$, so it is reasonable to believe that a historical policy with a larger $\xi_{\pi_i}$ will have a larger $\xi_{\rho_i}$.
Both $\xi_{\rho_i}$ and $\xi_{\pi_i}$ are fixed since historical policies and the current policy are immutable during model learning and model rollout.
Therefore, to reduce this term, we can only adjust the policy mixture distribution $\bm{w}^k$.
Since the distribution shift varies across historical policies and the current policy, it is obvious that the random distribution $w^k_i=\frac{1}{k}$ is not the best choice.\\
\textbf{(3)} The last term is related to the model sample buffer, which is used for \textbf{policy learning}. 
To maximize sample utilization, the model-generated samples obtained by the historical policies will be maintained in the model sample buffer until they are replaced by the new samples generated by the current policy.
Therefore, the distribution of simulated samples in the model buffer is not exactly the same as the simulated sample distribution of the current policy, but is often mixed with the simulated sample distribution of the historical policies.
This makes it necessary to adjust the sample distribution in the model sample buffer to make it close to the simulated sample distribution of the current policy during the policy learning process.
This has been studied in many model-based and model-free methods \citep{{schaul2016prioritized}, {liu2021regret}, {huang2021learning}, {mu2021model}} and is out of the scope of this paper, and we focus on reducing the first two terms related to model learning.

The first two items on the right-hand side of Equation (\ref{main_theory_eq}) provide 
%\st{us some} \joy{I think it's fine to remove} 
useful insights on model learning.
This first term points out the goal of model learning: to make accurate predictions for the current policy.
The second item further demonstrates that to achieve this goal, we should adjust the policy mixture distribution to reduce the distribution shift between the historical policy mixture and the current policy.
According to the second term, we have the following proposition.

\begin{proposition}\label{prop_3}
The performance gap can be reduced if the weight $w_i^k$ of each policy $\pi_i$ in the historical policy sequence $\Pi^k$ is negatively related to state action visitation distribution shift $\xi_{\rho_i}$ and the policy distribution shift $\xi_{\pi_i}$ between the historical policy $\pi_i$ and current policy $\pi$ instead of an average weight $w_i^k = \frac{1}{k}$.
\end{proposition}

The proof is in Appendix~\ref{connections}.
Proposition~\ref{prop_3} illustrates how we should adjust the policy distribution to help the learned dynamics model adapt to the current policy.
This naturally motivates our method, which is described in the next section.
% The first two items on the right-hand side of Eq. \ref{main_theory_eq} both point out that using a random distribution over historical policies is not the most efficient way for model learning.
% The first term states that model learning should adapt as much as possible to the state-action visitation distribution of the current policy.
% The second term proposes that the state-action visitation distribution in the training data should be closer to the current policy by adjusting the policy mixture distribution. 
% In the next subsection, we present a specific method for adjusting the policy mixture distribution to learn a policy-adapted dynamics model.

\section{Policy-adapted Dynamics Model Learning}
\label{sec4}

In this section, we introduce our model learning method called \emph{\ouralgofull} (\emph{\ouralgo}).
\ouralgo is designed to reduce the model prediction error during model rollouts, and it contains two parts.
The first part is adjusting the policy mixture distribution into a non-uniform distribution, and the second part is learning the dynamics model based on this non-uniform distribution.
The pseudo-code is in Algorithm~\ref{Alg_PDML}.

% \vspace{-8pt}
\begin{algorithm}[!htbp]
\caption{\ouralgofull (\ouralgo)}
\label{Alg_PDML}
\begin{algorithmic}[1]
% \REQUIRE \fhst{Dynamics model $\hat{T}_\theta$} \fhc{you don't need this as an input, since your algorithm learns it}, \fhst{current policy $\pi_c$}\fhc{if you illustrate in the algo how current policy is learned, you don't need it as input here}, \fhst{real sample buffer $\mathbb{D}_e$} \fhc{You could also not use it as an input for the algorithm, but rather illustrate how the sample buffer gets initialized and updated}, model sample buffer $\mathbb{D}_m$\fhc{same here}, \fhst{historical policy sequence $\Pi^{k}$,} \fhst{policy mixture distribution $\bm{w}^k$,} current policy proportion \fh{hyperparameter} $\alpha$, interaction epochs $I$ \
\REQUIRE current policy proportion hyperparameter $\alpha$, interaction epochs $I$ \
\STATE Initialize historical policy sequence $k\gets 0, \Pi^{k} \gets \emptyset$
\FOR{$I$ epochs}
\STATE Interact with the environment using current policy $\pi_c$, add samples into real sample buffer $\mathbb{D}_e$ \ 
\STATE Add current policy $\pi_c$ into historical policy sequence: $\pi_{k}\gets \pi_c$, $\Pi^{k} \gets \{\Pi^{k-1}, \pi_k\}$ \
\STATE Adjust the historical policy mixture distribution $\bm{w}^k=[w_1^k,\ldots,w_k^k]$ via Equation~(\ref{policy_weight}) and (\ref{current_policy_weight}) \
\STATE Normalize $\bm{w}_k \gets {\bm{w}_k}/{\lVert \bm{w}_k \rVert}$
\STATE Sample a training data batch of $(s_n,a_n,r,s_{n+1})$ from $\mathbb{D}_e$ according to $\bm{w}^k$ \ %\fh{lifetime-decay weight defined in Definition~\ref{xx}} \ 
\STATE Train dynamics model $\hat{T}_\theta$ via Equation~(\ref{model_loss}), use current policy $\pi_c$ to perform model rollouts \
\STATE $k\gets k+1$
% \STATE \fhc{Do you want to illustrate how the policy gets updated like in the wrapper algorithm in $a0_wrapper_algo.tex$?}
\ENDFOR
\end{algorithmic}
\end{algorithm}
% \vspace{-10pt}

\subsection{Policy Mixture Distribution Adjustment}
In this section, we introduce a mechanism to adjust the policy mixture distribution.
According to our Theorem~\ref{main_theory}, to minimize the performance gap, one may set the weight of the policy with the smallest $\xi_{\rho_i}$ and $\xi_{\pi_i}$ to be 1 and the weights of other policies in the historical policy sequence to be 0.
However, this is not the best approach in practice since each policy can only interact with the environment for very few steps in model-based RL.
This means each policy can provide very limited samples for model learning.
If we only use a small number of samples from just one policy, it is difficult to learn accurate transition dynamics for the current policy.

\paragraph{Weights design for historical policies.} In order to maximize the use of limited samples to estimate the transition dynamics, inspired by Proposition~\ref{prop_3}, we design the weight of each policy in the historical policy sequence  $\Pi^k = \{\pi_1, \pi_2, ..., \pi_k\}$ except for the current policy $\pi_c$ (i.e., $\pi_k \in \Pi^k$)  as follows:

\vspace{-15pt}
\begin{equation}\label{policy_weight}
% \resizebox{.9\hsize}{!}{$
\begin{aligned}
& w^k_i = \frac{{\xi_{\pi_i}}^{-1}}{\sum_{n=1}^{k} {\xi_{\pi_n}}^{-1}}, \\
& \xi_{\pi_i} = \mathbb{E}_{s \sim v^{\pi_{\text{mix}}}_{\hat{T}}} \left[ D_{TV}(\pi_c(\cdot|s)|| \pi_i(\cdot|s)) \right], \quad  \forall i \in [k-1],
\end{aligned}
% $}
\end{equation}

where $\xi_{\pi_i}$ is the policy distribution shift between historical policy $\pi^k_i$ and the current policy $\pi_c$; it is also one of the distribution shifts in the second term of Equation~(\ref{main_theory_eq}).
We use $[k-1] := \{1,\ldots,k-1\}$ to denote the integers from 1 to $k-1$.
We only use the policy distribution shift $\xi_{\pi_i}$ (and not the state-action visitation distribution shift $\xi_{\rho_i}$)
because estimating the state-action visitation distribution shift using limited real samples is difficult, and thus the estimation may be inaccurate.
Besides, as mentioned in the remarks of Theorem~\ref{main_theory}, state-action visitation distribution is induced by the policy, so it is reasonable to believe that a historical policy with a larger $\xi_{\pi_i}$ will have a larger $\xi_{\rho_i}$.

\paragraph{Weight design for the current policy.} %\st{In addition,} 
In model-based RL, the current policy becomes a historical policy after interacting with the environment and is added to the historical policy sequence (see Algorithm~\ref{Alg_PDML}).
The total variation distance between the current policy and itself is 0, so Equation~(\ref{policy_weight}) cannot be used to calculate the weight of the current policy. 
For the weight of the current policy $w_k^k$, we use the following equation:

\vspace{-12pt}
\begin{equation}\label{current_policy_weight}
%\tilde{w}_k^k= \alpha\sum_{i=1}^{k-1}w^{k}_{i}, \quad 
w_k^k = \left\{  
             \begin{matrix}
             %\tilde{w}_k^k, 
             \alpha\sum_{i=1}^{k-1}w^{k}_{i},
             & \quad \text{ if }  \quad %\tilde{w}_k^k 
             \alpha\sum_{i=1}^{k-1}w^{k}_{i}
             > \mathop{max}\limits_{i\in[k-1]}\{w^{k}_{i}\}   \\  
              \mathop{max}\limits_{i\in[k-1]}\{w^{k}_{i}\}, & \quad \text{ if } \quad 
              %\tilde{w}_k^k 
              \alpha\sum_{i=1}^{k-1}w^{k}_{i}
              \leq \mathop{max}\limits_{i\in[k-1]}\{w^{k}_{i}\}    
             \end{matrix}  
\right. 
\end{equation}
where $\alpha$ is a hyperparameter to control the proportion of the weight of the current policy to the total weight over the historical policy sequence.
Equation~(\ref{current_policy_weight}) ensures that the weight of the current policy $w_k^k$ is always the largest in the historical policy sequence.
Before each model learning iteration, we adjust the policy mixture distribution according to Equation~(\ref{policy_weight}) and Equation~(\ref{current_policy_weight}) and normalize the weights $\bm{w}^k =[w_1^k, ..., w_k^k]$ to make sure they sum to 1. The details are illustrated in Algorithm~\ref{Alg_PDML}.

\paragraph{Estimation of the policy distribution shift $\xi_{\pi_i}\ \forall i\in[k-1]$.} Given a state $s_n$, we define the output of policy $\pi_i$ as a multivariate Gaussian distribution $\mathcal{N}(\mu_{\pi^n_i}, \Sigma_{\pi^n_i})$.
In order to make the empirical estimation more accurate, we use each historical policy to traverse all $N$ samples in the real sample buffer and output the action distribution corresponding to each state.
Then we use the inequality between KL divergence and total variation distance to estimate $\xi_{\pi_i}$:

\vspace{-15pt}
\begin{equation}\label{TV_estimation}
\resizebox{\hsize}{!}{$
\begin{aligned}
\xi_{\pi_i}
& = \frac{1}{N} \sum_{n=1}^{N} D_{TV}(\pi_c(\cdot|s_n)||\pi_i(\cdot|s_n)) \\
& \leq \frac{1}{2N} \sum_{n=1}^{N} \sqrt{tr(\Sigma^{-1}_{\pi^n_c}\Sigma_{\pi^n_i}-I) + (\mu_{\pi^n_c}-\mu_{\pi^n_i})^\mathsf{T}\Sigma^{-1}_{\pi^n_i}(\mu_{\pi^n_c}-\mu_{\pi^n_i}) - \log\det(\Sigma^{-1}_{\pi^n_c}\Sigma_{\pi^n_i})}
\end{aligned}
$}
\end{equation}

\textbf{Novelty of \ouralgo compared to prioritized experience replay proposed in model-free RL.} \quad In model-free RL, prioritized experience replay methods only need to consider how to improve the policy based on existing samples.
Therefore, it is only necessary to select the sample that can bring the greatest improvement to the policy, and a weighting is designed for each sample. 
In model-based RL, the policy is learned based on model-generated samples, and the accuracy of these model-generated samples determines the sub-optimality of the policy. 
Thus, in the model-learning part, we focus on the model prediction accuracy.
Our theoretical analysis shows that we should consider whether the state-action visitation distribution that generates the samples is close to the current policy when reweighting samples. 
Although a sample can bring a great improvement to the current policy (the TD value is high), if this sample is not in the state-action visitation distribution of the current policy,
this sample will not be encountered during model rollouts. 
Then learning this sample will not bring any benefit to model learning and policy learning.
Therefore, we reweight the state-action visitation distribution that generates a batch of samples according to $\xi_{\pi_i}$, rather than a single sample as in model-free RL.

\subsection{Dynamics Model Learning}

After adjusting the policy mixture distribution, we learn the dynamics model based on this adjusted distribution.
Although our method can be applied to learn any type of dynamic model, here we choose to use the current state-of-the-art structure probabilistic dynamics model ensemble \citet{chua2018deep}: $\{\hat{T}^1_{\theta}, ..., \hat{T}^B_{\theta} \}$.
$\theta$ is the parameters of each dynamics model in the ensemble, and $B$ is the ensemble size.
Given a $(s_n, a_n)$ pair as an input, the output $\hat{T}^b_{\theta}$ of each network $b$ in the ensemble is the Multivariate Gaussian Distribution of the next state: $\hat{T}^b_{\theta}(s_{n+1}|s_n,a_n) = \mathcal{N}(\mu^b_{\theta}(s_n,a_n), \Sigma^b_{\theta}(s_n,a_n))$
Before each model learning iteration, we sample the training data batch from the real sample buffer according to the adjusted policy mixture distribution $\bm{w}^k$, and train the dynamics model using maximum likelihood:

\vspace{-22pt}
\begin{equation}\label{model_loss}
\resizebox{\hsize}{!}{$
\mathcal{L}(\theta) =  \sum_{n=1}^N [\mu^b_{\theta}(s_n,a_n)-s_{n+1}]^\top {\Sigma^b_{\theta}}^{-1}(s_n,a_n)[\mu^b_{\theta}(s_n,a_n)-s_{n+1}] + \log\det\Sigma^b_{\theta}(s_n,a_n)
$}
\end{equation}
\vspace{-13pt}

During model rollouts, we use the current policy $\pi_c$ as the rollout policy and sample the initial states from the real sample buffer according to the adjusted policy mixture distribution $\bm{w}^k$.
\section{Experiment}\label{section:experiment}
\vspace{-0.5em}

In this section, we will first compare our method with the previous state-of-the-art  (including both model-free and model-based) baselines. 
We demonstrate that after combining with SOTA model-based method, \ouralgo improves SOTA sample efficiency and SOTA asymptotic performance for model-based RL.
Then we compare our method with three SOTA prioritized experience replay methods to indicate the advantage of our distribution adjustment method for model learning.
Lastly, we conduct a systematic ablation study to analyze the model errors of \ouralgo.
% , PER~\citep{goyal2018recall}, RECALL~\citep{schaul2016prioritized} and MaPER~\citep{oh2022modelaugmented}, 
% Due to the space limitation, we provide more ablation studies in Appendix \ref{app:exp} including the discussion of policy adaptation rate.

% \vspace{-0.5em}
\subsection{Comparison with State-of-the-arts}
\vspace{-0.5em}

In this section, we compare our method with several previous state-of-the-art (SOTA) baselines.
For model-based methods, we choose MBPO~\citep{janner2019trust}, AMPO~\citep{shen2020model}, and VaGraM~\citep{voelcker2022value}.
MBPO is the SOTA model-based method, and our method is combined with MBPO for the model learning part.
We name our method \ouralgo-MBPO and we provide the pseudo code in Appendix~\ref{app:pdml_mbpo}.
AMPO is another SOTA model-based method that uses unsupervised model adaptation during model learning to reduce the prediction error.
VaGraM is a SOTA value equivalence model-based method. 
Instead of accurately learning each dimension in the dynamics, it aims to learn the dimensions which impact policy learning most.
In other words, this method also learns a locally accurate model.
Both AMPO and VaGraM are implemented based on MBPO.
\ouralgo-MBPO, AMPO, and VaGraM share the same model architecture and policy part; only the model learning part is different.
For model-free methods, we compare with two methods.
The first one is SAC \citep{haarnoja2018soft}, which is the policy part of all model-based and model-free baselines we used and is one of the SOTA model-free methods.
The second one is REDQ \citep{chen2021randomized}, which improves the Update-To-Data (UTD) ratio of the model-free method and achieves higher sample efficiency than SAC. 
The implementation details of our method are in Appendix~\ref{app:subsec:imp-details}.
We conduct experiment on six complex MoJoCo-v2 \citep{todorov2012mujoco} environments, the performance curves are shown in Figure~\ref{benchmark_exp}.

\begin{figure}[!htbp]
\centering
\resizebox{\columnwidth}{!}{\input{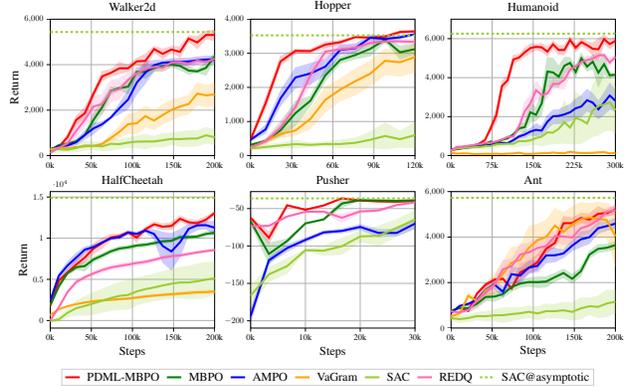}}
\caption{
Performance curves for our method (\ouralgo-MBPO) and other baseline methods on six MuJoCo environments. Our method, AMPO, MBPO and VaGram are model-based methods, while SAC and REDQ are model-free methods. The dashed line indicates the asymptotic performance of SAC. 
The solid lines indicate the mean over 8 seeds and shaded regions correspond to the $95\%$ confidence interval among seeds.
We evaluate the performance every 1k interaction steps. }
\label{benchmark_exp}
\end{figure}

\textbf{Results:} \textbf{(1) Improving SOTA sample efficiency.}
\ouralgo-MBPO outperforms all existing state-of-the-art methods, including model-based and model-free, in sample efficiency in first five environments, and achieves competitive sample efficiency in Ant.
In Hopper, Walker2d, and Humanoid, \ouralgo-MBPO achieves very impressive sample efficiency improvements, up to a 2$\times$ improvement  in Hopper and Humanoid compared to the SOTA model-based methods.
For example, our method using only 30k steps to achieve 3000 while other model-based methods need almost 60k steps.
Besides, its sample efficiency is also higher than REDQ which is modified for sample-efficient model-free RL.
\textbf{(2) Improving SOTA asymptotic performance for Model-based RL.} 
In addition, \ouralgo-MBPO obtains significantly better asymptotic performance compared to other state-of-the-art model-based methods.
It is worth noting that the asymptotic performance of \ouralgo-MBPO is very close to SAC in four environments (Hopper, Walker2d, Humanoid, and Pusher) and is even better than SAC occasionally.
Furthermore, our method achieves impressive improvement in the most complex environment Humanoid.
These indicate the effectiveness of our proposed model learning method.

\textbf{Discussion of computational cost.} 
\ouralgo requires saving all historical policies as well as computing their distances to the current policy for adjusting their weights (as shown in Equation~\ref{TV_estimation}). 
This creates an additional memory overhead of storing historical policy networks ($k \times $policy network size) and an additional computational overhead of computing the distances, for each iteration $k$. 
In PDML-MBPO, we observe storing historical policy networks costs a memory overhead of no more than $k \times$1 MB, compared with the high memory occupied by the model sample buffer, this cost is very small.
Besides, compared to MBPO, the training time of PDML-MBPO does not increase significantly.
We present the training time of PDML-MBPO and MBPO in six different environments. 
As shown in Table~\ref{tab: training time}, after using PDML, the training time doesn't increase significantly.
In the most complex environment Humanoid, the training time for 300k steps increases by only one hour. 
In other environments, the training time of PDML-MBPO is almost the same as that of MBPO.

\begin{table}[!htb]
\vspace{-5pt}
\setlength{\abovecaptionskip}{0.2cm}
\centering
\caption{Training time of PDML-MBPO and MBPO in different environments. The results are averaged over 8 random seeds.}
\begin{tabular}{c|cc}
\toprule 
 &MBPO&PDML-MBPO \\
\midrule
Walker2d & 58.6 h & 59.2 h\\
Hopper & 35.5 h & 35.7 h \\
Humanoid & 70.8 h & 72.0 h  \\
HalfCheetah & 60.2h  & 60.9 h \\
Pusher & 4.2 h & 4.3 h \\
Ant & 55.6 h & 55.9 h \\

\bottomrule
\end{tabular}
\label{tab: training time}
\end{table}
\vspace{-5pt}

\subsection{Comparison with Model-free Experience Replay Methods}\label{exp5.3}

We compare with the other three prioritized experience replay methods in model-free RL to indicate the advantage of our \ouralgo.
The first one is Prioritized Experience Replay (PER) \citep{schaul2016prioritized}, which weighs the samples according to their TD-error.
The second method is RECALL \citep{goyal2018recall}, which chooses the top $k$ highest value sample.
They use this to recall the samples that induce the high-value trajectories and train the policy.
We implement this by choosing the top $25\%$ highest $Q$ value samples to train the model and as model rollout initial states.
The third method is Model-augmented Prioritized Experience Replay (MaPER) \citep{oh2022modelaugmented}, which is an extension of PER using both TD-error and model prediction error to weight the samples for model learning.

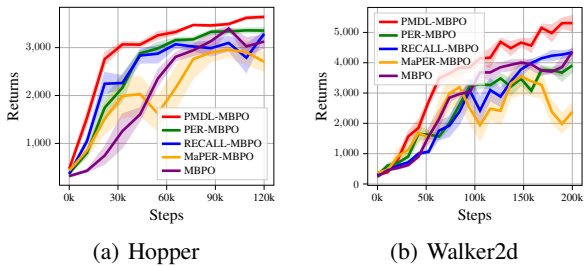
\begin{figure}[!htbp]
\centering
\subfigure[Hopper]{
\resizebox{0.215\textwidth}{!}{% This file was created with tikzplotlib v0.10.1.
\begin{tikzpicture}

\definecolor{darkgray176}{RGB}{176,176,176}
\definecolor{green}{RGB}{0,128,0}
\definecolor{lightgray204}{RGB}{204,204,204}
\definecolor{orange}{RGB}{255,165,0}
\definecolor{purple}{RGB}{128,0,128}

\begin{axis}[
legend cell align={left},
legend style={
  fill opacity=0.8,
  draw opacity=1,
  text opacity=1,
  at={(0.97,0.03)},
  anchor=south east,
  draw=lightgray204
},
tick align=outside,
tick pos=left,
x grid style={darkgray176},
xlabel={\Large{Steps}},
xmajorgrids,
xmin=-0.55, xmax=11.55,
xtick style={color=black},
xtick={0,2.75,5.5,8.25,11},
xticklabels={0k,30k,60k,90k,120k},
y grid style={darkgray176},
ylabel={\Large{Returns}},
ymajorgrids,
ymin=126.877985787876, ymax=3841.97610955141,
ytick style={color=black}
]
\path [draw=red, fill=red, opacity=0.2]
(axis cs:0,524.069090729243)
--(axis cs:0,424.469742762824)
--(axis cs:1,1392.32110725313)
--(axis cs:2,2638.68862400117)
--(axis cs:3,2993.63013246451)
--(axis cs:4,3007.27114802064)
--(axis cs:5,3171.88384149205)
--(axis cs:6,3274.36481001235)
--(axis cs:7,3458.67496428729)
--(axis cs:8,3432.37123261142)
--(axis cs:9,3460.88458598964)
--(axis cs:10,3589.7762996872)
--(axis cs:11,3609.33936961825)
--(axis cs:11,3673.1080130167)
--(axis cs:11,3673.1080130167)
--(axis cs:10,3659.62320104442)
--(axis cs:9,3522.55764161087)
--(axis cs:8,3489.06955781452)
--(axis cs:7,3486.37678663579)
--(axis cs:6,3377.25675221651)
--(axis cs:5,3335.03644292654)
--(axis cs:4,3118.29457493737)
--(axis cs:3,3126.93798037639)
--(axis cs:2,2901.9018158295)
--(axis cs:1,1680.38491899096)
--(axis cs:0,524.069090729243)
--cycle;

\path [draw=green, fill=green, opacity=0.2]
(axis cs:0,409.195856331342)
--(axis cs:0,347.604633888648)
--(axis cs:1,712.966287507385)
--(axis cs:2,1515.51327485693)
--(axis cs:3,1871.72785991194)
--(axis cs:4,2828.09066890951)
--(axis cs:5,2853.59256905678)
--(axis cs:6,3113.61231122919)
--(axis cs:7,3126.37448833977)
--(axis cs:8,3280.88631087716)
--(axis cs:9,3282.56323452107)
--(axis cs:10,3309.20785909855)
--(axis cs:11,3301.80441647243)
--(axis cs:11,3412.20821398142)
--(axis cs:11,3412.20821398142)
--(axis cs:10,3411.66116661514)
--(axis cs:9,3393.52054203527)
--(axis cs:8,3372.82645577195)
--(axis cs:7,3222.85557717593)
--(axis cs:6,3207.09771877727)
--(axis cs:5,3110.18375328259)
--(axis cs:4,2924.29516551101)
--(axis cs:3,2447.00597517351)
--(axis cs:2,1961.7080783628)
--(axis cs:1,884.076558205413)
--(axis cs:0,409.195856331342)
--cycle;

\path [draw=blue, fill=blue, opacity=0.2]
(axis cs:0,387.756374646298)
--(axis cs:0,334.468761953757)
--(axis cs:1,973.853545214208)
--(axis cs:2,2041.07716242756)
--(axis cs:3,2135.41359397869)
--(axis cs:4,2677.93732982089)
--(axis cs:5,2681.08931161748)
--(axis cs:6,2986.74122995583)
--(axis cs:7,2942.22314600752)
--(axis cs:8,2822.12351125943)
--(axis cs:9,2930.04673086393)
--(axis cs:10,2482.8595736311)
--(axis cs:11,3241.77171626091)
--(axis cs:11,3340.60992213103)
--(axis cs:11,3340.60992213103)
--(axis cs:10,3071.27242506606)
--(axis cs:9,3244.28271047968)
--(axis cs:8,3142.40715993979)
--(axis cs:7,3110.15903889684)
--(axis cs:6,3146.46003933559)
--(axis cs:5,3042.12336270044)
--(axis cs:4,2969.45218111355)
--(axis cs:3,2413.84667544309)
--(axis cs:2,2482.89964454103)
--(axis cs:1,1135.20246830226)
--(axis cs:0,387.756374646298)
--cycle;

\path [draw=orange, fill=orange, opacity=0.2]
(axis cs:0,473.576878814573)
--(axis cs:0,392.324318577407)
--(axis cs:1,722.910794777355)
--(axis cs:2,1192.21776453006)
--(axis cs:3,1535.92478352166)
--(axis cs:4,1574.59403367818)
--(axis cs:5,1262.42887253643)
--(axis cs:6,1715.16231499622)
--(axis cs:7,2348.6164299533)
--(axis cs:8,2701.27475817536)
--(axis cs:9,2802.78137750636)
--(axis cs:10,2639.82010601418)
--(axis cs:11,2503.04525109003)
--(axis cs:11,2890.87999494942)
--(axis cs:11,2890.87999494942)
--(axis cs:10,3136.03531318584)
--(axis cs:9,3087.47282411519)
--(axis cs:8,3082.342074447)
--(axis cs:7,3111.521935891)
--(axis cs:6,2587.28807090215)
--(axis cs:5,1969.74728553913)
--(axis cs:4,2408.28703116475)
--(axis cs:3,2404.67463125612)
--(axis cs:2,1809.13381285636)
--(axis cs:1,977.304137759031)
--(axis cs:0,473.576878814573)
--cycle;

\path [draw=purple, fill=purple, opacity=0.2]
(axis cs:0,339.972579588968)
--(axis cs:0,295.746082322582)
--(axis cs:1,400.519965741246)
--(axis cs:2,564.864196032257)
--(axis cs:3,902.529847831912)
--(axis cs:4,1473.51505171033)
--(axis cs:5,2193.01774168523)
--(axis cs:6,2720.57731207071)
--(axis cs:7,2835.31391123476)
--(axis cs:8,2913.98607861179)
--(axis cs:9,3352.73182316601)
--(axis cs:10,2835.56533106089)
--(axis cs:11,2988.32434405311)
--(axis cs:11,3266.64159864914)
--(axis cs:11,3266.64159864914)
--(axis cs:10,3209.8786664619)
--(axis cs:9,3438.32108330912)
--(axis cs:8,3343.41414410734)
--(axis cs:7,3063.59540174373)
--(axis cs:6,2897.72176180611)
--(axis cs:5,2543.22154778328)
--(axis cs:4,1744.95846983322)
--(axis cs:3,1617.22041537425)
--(axis cs:2,936.880602529529)
--(axis cs:1,460.106434202795)
--(axis cs:0,339.972579588968)
--cycle;

\addplot [line width=2.4pt, red]
table {%
0 471.089963694722
1 1532.56998894136
2 2761.26849040953
3 3067.62203033963
4 3062.95333272059
5 3256.19638478024
6 3326.40277057023
7 3472.54954923827
8 3462.64393825511
9 3492.56411339738
10 3622.56898451809
11 3641.32432551514
};
\addlegendentry{PMDL-MBPO}
\addplot [line width=2.4pt, green]
table {%
0 376.714442611452
1 803.857670013724
2 1747.48448574483
3 2161.58293439144
4 2880.68741366188
5 2983.29726044789
6 3158.63716620244
7 3172.51776804274
8 3329.74224233818
9 3339.51947620983
10 3360.57297286238
11 3356.26210670612
};
\addlegendentry{PER-MBPO}
\addplot [line width=2.4pt, blue]
table {%
0 362.461381822472
1 1047.9871887112
2 2246.48099669052
3 2260.59190471406
4 2835.54416504777
5 2874.31296567561
6 3069.73965062722
7 3024.6194890541
8 2989.35048108383
9 3098.42827841806
10 2789.05055087198
11 3287.98583969722
};
\addlegendentry{RECALL-MBPO}
\addplot [line width=2.4pt, orange]
table {%
0 430.678000103028
1 855.773312058494
2 1512.30937292727
3 1989.19033124918
4 2030.8023284377
5 1626.05135694962
6 2154.62511872032
7 2759.22774394523
8 2897.99971713904
9 2956.61521647972
10 2903.44158190959
11 2705.08505173789
};
\addlegendentry{MaPER-MBPO}
\addplot [line width=2.4pt, purple]
table {%
0 317.898530472971
1 428.743190593075
2 745.279637758621
3 1256.99048239969
4 1606.0990534834
5 2353.78492243818
6 2806.38044623282
7 2944.00105858782
8 3136.00963927379
9 3397.57586127361
10 3026.46923253685
11 3126.13677739073
};
\addlegendentry{MBPO}
\end{axis}

\end{tikzpicture}}
\label{hopper_per}
}
\hfil
\subfigure[Walker2d]{
\resizebox{0.215\textwidth}{!}{% This file was created with tikzplotlib v0.10.1.
\begin{tikzpicture}

\definecolor{darkgray176}{RGB}{176,176,176}
\definecolor{green}{RGB}{0,128,0}
\definecolor{lightgray204}{RGB}{204,204,204}
\definecolor{orange}{RGB}{255,165,0}
\definecolor{purple}{RGB}{128,0,128}

\begin{axis}[
legend cell align={left},
legend style={
  fill opacity=0.8,
  draw opacity=1,
  text opacity=1,
  at={(0.03,0.97)},
  anchor=north west,
  draw=lightgray204
},
tick align=outside,
tick pos=left,
x grid style={darkgray176},
xlabel={\Large{Steps}},
xmajorgrids,
xmin=-0.95, xmax=19.95,
xtick style={color=black},
xtick={0,4.75,9.5,14.25,19},
xticklabels={0k,50k,100k,150k,200k},
y grid style={darkgray176},
ylabel={\Large{Returns}},
ymajorgrids,
ymin=-45.8110688563009, ymax=5824.07983702833,
ytick style={color=black}
]
\path [draw=red, fill=red, opacity=0.2]
(axis cs:0,310.830550089836)
--(axis cs:0,272.854401780311)
--(axis cs:1,373.334428159517)
--(axis cs:2,679.589206648156)
--(axis cs:3,1402.65087744195)
--(axis cs:4,1610.96680162714)
--(axis cs:5,2591.79893612864)
--(axis cs:6,3236.28740991097)
--(axis cs:7,3490.04561520236)
--(axis cs:8,3586.96303636819)
--(axis cs:9,3700.7246066198)
--(axis cs:10,3990.1790472297)
--(axis cs:11,3830.77014966605)
--(axis cs:12,4541.50839418894)
--(axis cs:13,4241.10155563128)
--(axis cs:14,4535.67264001557)
--(axis cs:15,4339.50074722111)
--(axis cs:16,5020.8007366783)
--(axis cs:17,4693.35385495834)
--(axis cs:18,5141.80632094195)
--(axis cs:19,5036.70165011735)
--(axis cs:19,5557.26661403358)
--(axis cs:19,5557.26661403358)
--(axis cs:18,5473.29540903158)
--(axis cs:17,5255.8377128452)
--(axis cs:16,5276.63120963503)
--(axis cs:15,4804.20323095772)
--(axis cs:14,4802.80691154034)
--(axis cs:13,4700.41802904981)
--(axis cs:12,4810.61684569456)
--(axis cs:11,4470.30929036581)
--(axis cs:10,4276.35786927132)
--(axis cs:9,3979.24122612549)
--(axis cs:8,4119.37202166008)
--(axis cs:7,3780.73018612708)
--(axis cs:6,3710.15799516396)
--(axis cs:5,2868.73347012784)
--(axis cs:4,2118.09272673743)
--(axis cs:3,1742.22871642554)
--(axis cs:2,955.425591023964)
--(axis cs:1,397.340916165709)
--(axis cs:0,310.830550089836)
--cycle;

\path [draw=green, fill=green, opacity=0.2]
(axis cs:0,317.019049839726)
--(axis cs:0,274.727353004765)
--(axis cs:1,535.074443118566)
--(axis cs:2,656.633510972536)
--(axis cs:3,746.312549128292)
--(axis cs:4,1438.7311912241)
--(axis cs:5,1502.89262524681)
--(axis cs:6,1455.15895257688)
--(axis cs:7,1825.27355638602)
--(axis cs:8,2372.78064878015)
--(axis cs:9,3041.25844159893)
--(axis cs:10,3158.12700832489)
--(axis cs:11,3154.17582271126)
--(axis cs:12,3329.89643584159)
--(axis cs:13,2909.96013744011)
--(axis cs:14,3240.47004866467)
--(axis cs:15,2973.68244200398)
--(axis cs:16,3624.11146060211)
--(axis cs:17,3661.66776287296)
--(axis cs:18,3392.04554369605)
--(axis cs:19,3685.99131854143)
--(axis cs:19,4151.45522311415)
--(axis cs:19,4151.45522311415)
--(axis cs:18,3921.69556389106)
--(axis cs:17,3850.16696431225)
--(axis cs:16,3903.96699660825)
--(axis cs:15,3181.50046917754)
--(axis cs:14,3695.12913907325)
--(axis cs:13,3472.87969451086)
--(axis cs:12,3638.80952339543)
--(axis cs:11,3365.89676648014)
--(axis cs:10,3429.07999090041)
--(axis cs:9,3404.81108378693)
--(axis cs:8,2856.32252124574)
--(axis cs:7,2184.30601817428)
--(axis cs:6,1677.09724184029)
--(axis cs:5,1750.6254726433)
--(axis cs:4,1907.17769195034)
--(axis cs:3,1075.65132973337)
--(axis cs:2,739.633852729188)
--(axis cs:1,687.627561919525)
--(axis cs:0,317.019049839726)
--cycle;

\path [draw=blue, fill=blue, opacity=0.2]
(axis cs:0,315.667325993333)
--(axis cs:0,278.508931899421)
--(axis cs:1,425.321963841805)
--(axis cs:2,571.955507658693)
--(axis cs:3,686.689420910461)
--(axis cs:4,942.493516004444)
--(axis cs:5,1008.70690590407)
--(axis cs:6,1456.68952199857)
--(axis cs:7,1558.77302467293)
--(axis cs:8,1963.61312491929)
--(axis cs:9,2678.99939027615)
--(axis cs:10,2025.47742624466)
--(axis cs:11,2745.53782648025)
--(axis cs:12,2556.08948908627)
--(axis cs:13,3162.08224697362)
--(axis cs:14,3543.05186948128)
--(axis cs:15,3870.0909655185)
--(axis cs:16,4061.83925185107)
--(axis cs:17,4061.33554267235)
--(axis cs:18,4144.75938391817)
--(axis cs:19,4223.49007300043)
--(axis cs:19,4449.02916187939)
--(axis cs:19,4449.02916187939)
--(axis cs:18,4395.38871773529)
--(axis cs:17,4409.1778969368)
--(axis cs:16,4233.36925878463)
--(axis cs:15,4109.64236183704)
--(axis cs:14,4010.55312016175)
--(axis cs:13,3613.24811701234)
--(axis cs:12,3169.86661224724)
--(axis cs:11,3479.83999049184)
--(axis cs:10,2871.16542176524)
--(axis cs:9,3496.45079935016)
--(axis cs:8,2808.86861545341)
--(axis cs:7,2228.34699885204)
--(axis cs:6,2060.44109952198)
--(axis cs:5,1128.48015582549)
--(axis cs:4,1072.47753627247)
--(axis cs:3,743.549169983274)
--(axis cs:2,644.058008719437)
--(axis cs:1,476.952103658752)
--(axis cs:0,315.667325993333)
--cycle;

\path [draw=orange, fill=orange, opacity=0.2]
(axis cs:0,416.786922317573)
--(axis cs:0,337.157153389551)
--(axis cs:1,440.40353938867)
--(axis cs:2,706.462767023301)
--(axis cs:3,934.878991167328)
--(axis cs:4,1357.69683796338)
--(axis cs:5,1312.96549890378)
--(axis cs:6,1701.69077589263)
--(axis cs:7,2463.76192394257)
--(axis cs:8,2620.27599039918)
--(axis cs:9,1977.44709910504)
--(axis cs:10,1421.77272339951)
--(axis cs:11,2107.34465657431)
--(axis cs:12,1957.44247598932)
--(axis cs:13,2818.98598820004)
--(axis cs:14,2964.17504500921)
--(axis cs:15,2905.68152478869)
--(axis cs:16,2831.45402635911)
--(axis cs:17,1995.74392479687)
--(axis cs:18,1678.76025625361)
--(axis cs:19,2034.1321632111)
--(axis cs:19,2671.78675697139)
--(axis cs:19,2671.78675697139)
--(axis cs:18,2339.46916177861)
--(axis cs:17,2753.22424144047)
--(axis cs:16,3671.37758904207)
--(axis cs:15,3853.30802407887)
--(axis cs:14,4086.18833457748)
--(axis cs:13,3728.97487628367)
--(axis cs:12,2809.79097427314)
--(axis cs:11,2916.51815494367)
--(axis cs:10,2514.75820421423)
--(axis cs:9,3004.0934028769)
--(axis cs:8,3718.51367357913)
--(axis cs:7,3566.80919978271)
--(axis cs:6,2745.73270035633)
--(axis cs:5,1759.74505919426)
--(axis cs:4,2106.10751030199)
--(axis cs:3,1336.04207600111)
--(axis cs:2,1213.17718718444)
--(axis cs:1,492.916451115113)
--(axis cs:0,416.786922317573)
--cycle;

\path [draw=purple, fill=purple, opacity=0.2]
(axis cs:0,255.66183395849)
--(axis cs:0,221.002154138455)
--(axis cs:1,424.458998966034)
--(axis cs:2,479.467590141296)
--(axis cs:3,574.509440538923)
--(axis cs:4,846.504928251731)
--(axis cs:5,1346.29526525113)
--(axis cs:6,1855.56570286306)
--(axis cs:7,2687.53997036052)
--(axis cs:8,2905.4578799687)
--(axis cs:9,2772.81512693713)
--(axis cs:10,3545.9711005254)
--(axis cs:11,3398.05500946957)
--(axis cs:12,3697.76395208573)
--(axis cs:13,3734.68043393137)
--(axis cs:14,3906.20217559042)
--(axis cs:15,3766.67741500411)
--(axis cs:16,3580.93194345776)
--(axis cs:17,3639.40336307502)
--(axis cs:18,3564.04468282485)
--(axis cs:19,4216.83888831274)
--(axis cs:19,4494.6498627389)
--(axis cs:19,4494.6498627389)
--(axis cs:18,4179.89770859862)
--(axis cs:17,3769.6802402328)
--(axis cs:16,3879.83656598043)
--(axis cs:15,4119.62422235372)
--(axis cs:14,4112.18665697337)
--(axis cs:13,4096.11856590747)
--(axis cs:12,3991.16656272738)
--(axis cs:11,3939.84524975219)
--(axis cs:10,3836.53855406551)
--(axis cs:9,3304.85017612933)
--(axis cs:8,3029.82210734516)
--(axis cs:7,3005.18406001441)
--(axis cs:6,2439.41323577081)
--(axis cs:5,1865.58134249194)
--(axis cs:4,955.012075848352)
--(axis cs:3,683.270244686492)
--(axis cs:2,579.101701743367)
--(axis cs:1,497.655891518741)
--(axis cs:0,255.66183395849)
--cycle;

\addplot [line width=2.4pt, red]
table {%
0 293.145767028379
1 385.319823395862
2 815.240276427423
3 1570.03012304247
4 1854.69304332409
5 2725.36147395747
6 3480.88699550507
7 3644.25712843865
8 3838.64182531279
9 3845.49749084624
10 4145.53383593358
11 4165.74091733639
12 4688.18311046203
13 4479.68161128147
14 4665.26409458896
15 4566.25108058516
16 5151.12356435023
17 4975.66193607887
18 5306.58619043586
19 5307.87194895502
};
\addlegendentry{PMDL-MBPO}
\addplot [line width=2.4pt, green]
table {%
0 295.809940044257
1 610.287882936222
2 699.352655255748
3 904.798193573448
4 1654.64657427699
5 1627.81196682711
6 1564.49235203566
7 2006.7820866683
8 2606.08197872273
9 3233.25994864767
10 3288.78918049899
11 3261.42787250448
12 3482.1252012132
13 3213.52207923043
14 3469.33745132579
15 3079.33439424815
16 3761.94190141642
17 3753.37073271655
18 3674.6477206485
19 3912.82045828881
};
\addlegendentry{PER-MBPO}
\addplot [line width=2.4pt, blue]
table {%
0 296.697437176067
1 451.288293842834
2 610.446583866168
3 714.352571429727
4 1001.16945318658
5 1066.88195253643
6 1760.39842207089
7 1923.06621201632
8 2389.99687765374
9 3082.46690358139
10 2422.64689442319
11 3101.32483016986
12 2881.47835881297
13 3376.5130517439
14 3781.04999623225
15 3988.31090599839
16 4146.45168463152
17 4231.82656147448
18 4269.63559971198
19 4339.84650375995
};
\addlegendentry{RECALL-MBPO}
\addplot [line width=2.4pt, orange]
table {%
0 374.886592403782
1 467.778615030257
2 916.672675958681
3 1130.68950316787
4 1687.85490937643
5 1533.7382168465
6 2230.39408186882
7 2981.44016641824
8 3193.09956180543
9 2496.8133349408
10 1938.54184395034
11 2482.52078948845
12 2412.86984527831
13 3311.30394659504
14 3546.37283895352
15 3386.88928775997
16 3271.02540188075
17 2380.6175432744
18 1982.07110696965
19 2379.26106415477
};
\addlegendentry{MaPER-MBPO}
\addplot [line width=2.4pt, purple]
table {%
0 237.756184186918
1 460.009386931806
2 529.894135830404
3 622.994436878792
4 902.484668439908
5 1573.15167495849
6 2135.10236193063
7 2842.8485840294
8 2968.12256469933
9 3037.49917042036
10 3677.28697016184
11 3682.79395882905
12 3850.6207721261
13 3930.66016017183
14 4004.43281846588
15 3947.29235098245
16 3741.48541290752
17 3706.02864715436
18 3855.39285736112
19 4362.40445380108
};
\addlegendentry{MBPO}
\end{axis}

\end{tikzpicture}}
\label{walker_per}
}
\vspace{-10pt}
\caption{
The comparison of model-free experience replay methods on Hopper and Walker2d. The experiments are run for 8 random seeds.
}
\label{MF_comparison}
\end{figure}

The experiment results are shown in Figure~\ref{hopper_per} and \ref{walker_per}. 
Our \ouralgo significantly outperforms all three methods on both sample efficiency and asymptotic performance.
We believe these methods adjust the weights for each sample in the training data rather than each policy. 
This will cause the samples belonging to the same state-action visitation distribution to have different weights,
and the samples with higher weights may not necessarily appear in the state-action visitation distribution of current policy. 
Therefore, the learned model cannot be adapted to current policy's state-action visitation distribution, and the model prediction error during model rollouts cannot be reduced. 
In the model learning process, it is crucial to adapt to current policy's state-action visitation distribution according to our theory.
This experiment result indicates our theory's correctness and our method's effectiveness.
We also compare with an exponentially decay method to demonstrate the effectiveness of our method.
In this exponentially decay method, the sample's weight is exponentially decay based on a decay rate as its age increases.
The results and details are shown in Appendix~\ref{app: exponentially decay}.

% \vspace{-0.5em}
\subsection{Model Error Analysis}

To further verify the impact of \ouralgo, 
we compare the one-step prediction error and the compounding error of the policy-adapted model learned by \ouralgo-MBPO and the original dynamics model learned by MBPO.

\textbf{One-step prediction error.}\quad
As shown in Figure~\ref{hopper_onestep_error}, \ref{half_onestep_error} and \ref{walker_onestep_error}, we evaluate the model prediction error for the current policy on Hopper, HalfCheetah, and Walker2d.
We evaluate the learned model every 1000 environment steps using L2 error on the 1000 samples obtained by the current policy from the real environment.
The error curves show that the one-step prediction error for the current policy of the policy-adapted model is much smaller than that of the original dynamics model,
which means the model-generated samples of \ouralgo-MBPO are more accurate than MBPO, so the policy induced by \ouralgo-MBPO can perform better.

% In Figure \ref{hopper_onestep_error}, we plot the curves of one-step model prediction error on Hopper.
% We evaluate the learned dynamics model of MBPO and LMAC on 10k samples obtained by interacting with the real environment using the newest policy, and calculate the L1 loss.
% It shows that after training based on lifetime-decay mechanism, the learned model becomes more accurate. 
% Consequently, the policy induced by the improved dynamics model can perform better.

\textbf{Compounding error.}\quad
We also compare the multi-step model rollouts compounding error of the policy-adapted model and the original dynamics model.
This directly determines the accuracy of the model-generated samples in each model rollout trajectory.
Figure~\ref{Compounding_error} shows the compounding error curves of the policy-adapted model and the original dynamics model on Hopper.
We calculate the $h$-step compounding error as the difference between the state at each rollout step $h$ in the model rollout trajectory and the real environment rollout trajectory using L2 error.
The results demonstrate that the policy-adapted model has much a smaller compounding error than the original dynamics model, which means 
the policy-adapted model has a more robust multi-step planning capability than the original dynamics model learned by MBPO.

\begin{figure}[!htbp]
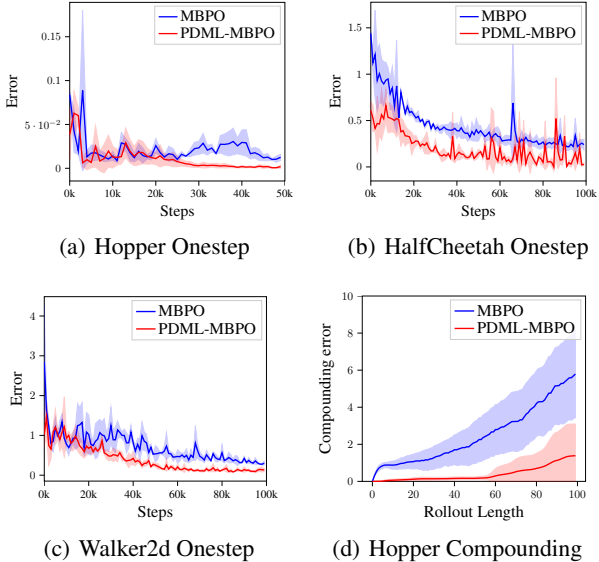

% \vspace{-5pt}
\centering
\subfigure[Hopper Onestep]{
\resizebox{0.23\textwidth}{!}{\input{exp_results/PDML_hopper_one_step}}
\label{hopper_onestep_error}
}
% \vspace{-10pt}
\hfil
\subfigure[HalfCheetah Onestep]{
\resizebox{0.215\textwidth}{!}{\input{exp_results/PDML_half_one_step}}
\label{half_onestep_error}
}
\hfil
\subfigure[Walker2d Onestep]{
\resizebox{0.215\textwidth}{!}{\input{exp_results/PDML_walker_one_step}}
\label{walker_onestep_error}
}
\hfil
\subfigure[Hopper Compounding]{
\resizebox{0.215\textwidth}{!}{\input{exp_results/PDML_Compounding_error}}
\label{Compounding_error}
}
\vspace{-10pt}
\caption{
\textbf{(a), (b) and (c)} display one-step (model-prediction) error for \ouralgo-MBPO and MBPO. 
\textbf{(d)} demonstrates the compounding error (i.e., the difference between the $h$-step state in the model rollout trajectory and the real environment rollout trajectory) of \ouralgo-MBPO and MBPO over 20 model rollout trajectories.
}
\vspace{-10pt}
\label{model error fig}
\end{figure}

\subsection{Comparison with Simple Exponentially Decay Prioritization}\label{app: exponentially decay}

To further demonstrate the effectiveness of our method, we compare with an exponentially decay method.
The weight of the historical policy exponentially decays as it lifetime increases.
To ensure a fair comparison, the weight of current policy is also compute using Eq.~\ref{current_policy_weight}.
The hyperparameter $\alpha$ of the current policy in exponentially decay method is the same as PDML which is given in Appendix~\ref{app:subsec:imp-details}.
The exponentially decay rate of exponentially decay method in Figure~\ref{ablation study: exponentially decay} is 0.98.
We conduct the experiment on three MoJoCo environments: Hopper, Walker2d, and Humanoid.
The performance curves are given in Figure~\ref{ablation study: exponentially decay}.
Moreover, to demonstrate the effective of our method, we provide more results of well-tuned exponentially decay rates in Table~\ref{tab: performance exponentially decay}.
We can see that after using exponentially decay method, the performance in three environments is slightly improved, but it is much lower than PDML.
Besides, the model error of exponentially decay method is higher than PDML.
Combined with the analysis of distribution visualization in Appendix~\ref{app: distribution visualization}, this further demonstrates that our method is non-trivial and effective.

\begin{figure}[!htbp]
\centering
\subfigure[Hopper]{
\resizebox{0.22\textwidth}{!}{% This file was created with tikzplotlib v0.10.1.
\begin{tikzpicture}

\definecolor{darkgray176}{RGB}{176,176,176}
\definecolor{green}{RGB}{0,128,0}
\definecolor{lightgray204}{RGB}{204,204,204}

\begin{axis}[
legend cell align={left},
legend style={
  fill opacity=0.8,
  draw opacity=1,
  text opacity=1,
  at={(0.97,0.03)},
  anchor=south east,
  draw=lightgray204
},
tick align=outside,
tick pos=left,
x grid style={darkgray176},
xlabel={Steps},
xmajorgrids,
xmin=-0.55, xmax=11.55,
xtick style={color=black},
xtick={0,2.75,5.5,8.25,11},
xticklabels={0k,30k,60k,90k,120k},
y grid style={darkgray176},
ylabel={Returns},
ymajorgrids,
ymin=129.669711809802, ymax=3841.8373167974,
ytick style={color=black}
]
\path [draw=red, fill=red, opacity=0.2]
(axis cs:0,517.296683125515)
--(axis cs:0,425.347691238838)
--(axis cs:1,1376.69168407417)
--(axis cs:2,2632.94806941622)
--(axis cs:3,2997.86890212284)
--(axis cs:4,3010.93448435969)
--(axis cs:5,3171.91631808057)
--(axis cs:6,3274.33963434043)
--(axis cs:7,3458.97551007875)
--(axis cs:8,3428.14106030587)
--(axis cs:9,3460.58839139814)
--(axis cs:10,3586.7532783166)
--(axis cs:11,3610.4852125345)
--(axis cs:11,3673.1024256616)
--(axis cs:11,3673.1024256616)
--(axis cs:10,3660.12289320726)
--(axis cs:9,3518.98845752374)
--(axis cs:8,3489.44583022376)
--(axis cs:7,3485.34008173515)
--(axis cs:6,3377.03382731917)
--(axis cs:5,3332.74825998519)
--(axis cs:4,3118.9117226497)
--(axis cs:3,3126.00329283869)
--(axis cs:2,2899.10228613653)
--(axis cs:1,1685.19294058818)
--(axis cs:0,517.296683125515)
--cycle;

\path [draw=blue, fill=blue, opacity=0.2]
(axis cs:0,344.280140067163)
--(axis cs:0,304.421440304486)
--(axis cs:1,1073.18095990671)
--(axis cs:2,2179.38455830193)
--(axis cs:3,2603.83896415301)
--(axis cs:4,2490.39453400589)
--(axis cs:5,2702.90558836896)
--(axis cs:6,2693.92480730185)
--(axis cs:7,2995.88807087828)
--(axis cs:8,3165.66124626525)
--(axis cs:9,3101.00238868539)
--(axis cs:10,3167.20287023334)
--(axis cs:11,3302.75026618459)
--(axis cs:11,3338.58361187453)
--(axis cs:11,3338.58361187453)
--(axis cs:10,3184.24745949078)
--(axis cs:9,3281.4438088699)
--(axis cs:8,3278.34263973574)
--(axis cs:7,3062.24180624806)
--(axis cs:6,2775.22951101342)
--(axis cs:5,2980.06920775468)
--(axis cs:4,2601.59590889263)
--(axis cs:3,2670.63114181856)
--(axis cs:2,2286.58034254622)
--(axis cs:1,1419.33412206494)
--(axis cs:0,344.280140067163)
--cycle;

\path [draw=green, fill=green, opacity=0.2]
(axis cs:0,340.688074049299)
--(axis cs:0,298.404602945602)
--(axis cs:1,400.516249679289)
--(axis cs:2,581.648138459044)
--(axis cs:3,923.1466079283)
--(axis cs:4,1471.35612539944)
--(axis cs:5,2191.61358840189)
--(axis cs:6,2716.6323560744)
--(axis cs:7,2838.7236657148)
--(axis cs:8,2917.18139620777)
--(axis cs:9,3350.52588865511)
--(axis cs:10,2820.05389950922)
--(axis cs:11,2983.76419293711)
--(axis cs:11,3282.34388767606)
--(axis cs:11,3282.34388767606)
--(axis cs:10,3235.15138745673)
--(axis cs:9,3440.71911213682)
--(axis cs:8,3311.23131480895)
--(axis cs:7,3055.43169476955)
--(axis cs:6,2898.7461563137)
--(axis cs:5,2546.23969815017)
--(axis cs:4,1736.45671926268)
--(axis cs:3,1632.50808273817)
--(axis cs:2,940.15850766482)
--(axis cs:1,459.651408871956)
--(axis cs:0,340.688074049299)
--cycle;

\addplot [very thick, red]
table {%
0 472.2825980442
1 1533.07804732684
2 2757.95326151378
3 3066.45024549037
4 3063.59092831097
5 3255.02745881312
6 3324.81653038239
7 3472.6105495737
8 3462.1781851166
9 3491.78016219375
10 3622.25170099892
11 3640.75695517332
};
\addlegendentry{PMDL-MBPO}
\addplot [very thick, blue]
table {%
0 324.953667415584
1 1245.38302115653
2 2227.13223542776
3 2638.6268944124
4 2550.13562540757
5 2841.70473831988
6 2734.95718186977
7 3030.50088036312
8 3227.09736601941
9 3199.32482866506
10 3175.80403202457
11 3320.15584675522
};
\addlegendentry{Exponentially Decay-MBPO}
\addplot [very thick, green]
table {%
0 318.2979419574
1 429.184866187345
2 744.743127737697
3 1248.15759330845
4 1604.77416448329
5 2352.15159098185
6 2805.94148943797
7 2945.15131486378
8 3129.48391751262
9 3397.14831191856
10 3032.83288535635
11 3123.36001414721
};
\addlegendentry{MBPO}
\end{axis}

\end{tikzpicture}}
}
\hfil
\subfigure[Walker2d]{
\resizebox{0.22\textwidth}{!}{% This file was created with tikzplotlib v0.10.1.
\begin{tikzpicture}

\definecolor{darkgray176}{RGB}{176,176,176}
\definecolor{green}{RGB}{0,128,0}
\definecolor{lightgray204}{RGB}{204,204,204}

\begin{axis}[
legend cell align={left},
legend style={
  fill opacity=0.8,
  draw opacity=1,
  text opacity=1,
  at={(0.97,0.03)},
  anchor=south east,
  draw=lightgray204
},
tick align=outside,
tick pos=left,
x grid style={darkgray176},
xlabel={Steps},
xmajorgrids,
xmin=-0.95, xmax=19.95,
xtick style={color=black},
xtick={0,4.75,9.5,14.25,19},
xticklabels={0k,50k,100k,150k,200k},
y grid style={darkgray176},
ylabel={Returns},
ymajorgrids,
ymin=-43.5714763533834, ymax=5786.889873084,
ytick style={color=black}
]
\path [draw=red, fill=red, opacity=0.2]
(axis cs:0,312.137142082192)
--(axis cs:0,273.670032943372)
--(axis cs:1,373.773005438413)
--(axis cs:2,673.396739223287)
--(axis cs:3,1395.76563425731)
--(axis cs:4,1630.38126687919)
--(axis cs:5,2591.75623559294)
--(axis cs:6,3257.79492800077)
--(axis cs:7,3491.94312997644)
--(axis cs:8,3575.85259941692)
--(axis cs:9,3719.3942054446)
--(axis cs:10,4003.23125458285)
--(axis cs:11,3836.86906648947)
--(axis cs:12,4533.90367598524)
--(axis cs:13,4237.31201260413)
--(axis cs:14,4534.3579807437)
--(axis cs:15,4312.54939448837)
--(axis cs:16,5023.5647187668)
--(axis cs:17,4677.82829546429)
--(axis cs:18,5130.801971786)
--(axis cs:19,5062.67712270133)
--(axis cs:19,5521.86890265503)
--(axis cs:19,5521.86890265503)
--(axis cs:18,5467.6980768069)
--(axis cs:17,5263.01569913763)
--(axis cs:16,5280.33722397144)
--(axis cs:15,4771.61362775632)
--(axis cs:14,4801.30288392466)
--(axis cs:13,4690.3529820724)
--(axis cs:12,4817.69443851314)
--(axis cs:11,4432.54968022565)
--(axis cs:10,4289.50629900668)
--(axis cs:9,3971.64601022665)
--(axis cs:8,4103.63094915172)
--(axis cs:7,3786.96161653939)
--(axis cs:6,3703.48365982382)
--(axis cs:5,2859.52354881791)
--(axis cs:4,2137.67878944342)
--(axis cs:3,1725.60546511501)
--(axis cs:2,955.674285706917)
--(axis cs:1,396.459525425475)
--(axis cs:0,312.137142082192)
--cycle;

\path [draw=blue, fill=blue, opacity=0.2]
(axis cs:0,365.381261931525)
--(axis cs:0,319.000510325668)
--(axis cs:1,376.750543755037)
--(axis cs:2,548.931069030236)
--(axis cs:3,879.141946594633)
--(axis cs:4,1592.11227016026)
--(axis cs:5,1785.81634484709)
--(axis cs:6,2330.92110405421)
--(axis cs:7,2761.95273964219)
--(axis cs:8,3465.57834909026)
--(axis cs:9,3165.38509454145)
--(axis cs:10,3789.13554846103)
--(axis cs:11,3834.33183162128)
--(axis cs:12,4028.62655068414)
--(axis cs:13,4118.80861258522)
--(axis cs:14,4175.44658445664)
--(axis cs:15,3737.28366509353)
--(axis cs:16,4399.4541669786)
--(axis cs:17,4419.75301415505)
--(axis cs:18,4547.17738060075)
--(axis cs:19,4494.36377908154)
--(axis cs:19,4725.58863892362)
--(axis cs:19,4725.58863892362)
--(axis cs:18,4878.77796428104)
--(axis cs:17,4610.13028838152)
--(axis cs:16,4671.33926572658)
--(axis cs:15,4417.54366901333)
--(axis cs:14,4553.39714727358)
--(axis cs:13,4236.98754348942)
--(axis cs:12,4240.6331781381)
--(axis cs:11,3977.53115611081)
--(axis cs:10,4044.76827041997)
--(axis cs:9,3862.85823073977)
--(axis cs:8,3702.97126211967)
--(axis cs:7,3384.10713339545)
--(axis cs:6,3096.67095177143)
--(axis cs:5,2589.91730243538)
--(axis cs:4,2105.79073704291)
--(axis cs:3,1300.22991197108)
--(axis cs:2,674.100411680055)
--(axis cs:1,462.733658328953)
--(axis cs:0,365.381261931525)
--cycle;

\path [draw=green, fill=green, opacity=0.2]
(axis cs:0,255.365786959903)
--(axis cs:0,221.449494075589)
--(axis cs:1,422.021214749422)
--(axis cs:2,485.385748469539)
--(axis cs:3,566.885348635053)
--(axis cs:4,847.184918778841)
--(axis cs:5,1320.02886002872)
--(axis cs:6,1851.68322528062)
--(axis cs:7,2678.27432703222)
--(axis cs:8,2903.08336044723)
--(axis cs:9,2750.46605216057)
--(axis cs:10,3561.03145010824)
--(axis cs:11,3399.70429198482)
--(axis cs:12,3686.2874599239)
--(axis cs:13,3747.981466727)
--(axis cs:14,3916.69402976281)
--(axis cs:15,3767.91392147466)
--(axis cs:16,3604.19925706867)
--(axis cs:17,3647.5095886664)
--(axis cs:18,3556.63194480067)
--(axis cs:19,4218.50223546647)
--(axis cs:19,4517.61752102023)
--(axis cs:19,4517.61752102023)
--(axis cs:18,4132.41782681475)
--(axis cs:17,3773.561675977)
--(axis cs:16,3884.58928587413)
--(axis cs:15,4126.03660106658)
--(axis cs:14,4113.04919788532)
--(axis cs:13,4096.08894536977)
--(axis cs:12,4005.90529630844)
--(axis cs:11,3965.66028384058)
--(axis cs:10,3839.04345927852)
--(axis cs:9,3297.91226719749)
--(axis cs:8,3035.32527608424)
--(axis cs:7,2995.31023442028)
--(axis cs:6,2411.14883532218)
--(axis cs:5,1831.3724199482)
--(axis cs:4,958.053857353776)
--(axis cs:3,682.552709779187)
--(axis cs:2,581.22924632974)
--(axis cs:1,495.129580811793)
--(axis cs:0,255.365786959903)
--cycle;

\addplot [very thick, red]
table {%
0 293.868787849442
1 385.600345329877
2 812.854661627481
3 1570.28699425435
4 1859.50146459058
5 2721.57997378863
6 3477.26112599465
7 3642.80335008836
8 3837.71058707689
9 3845.89287093634
10 4145.75553370011
11 4158.75818683123
12 4685.0787405485
13 4474.32275734441
14 4667.63568773962
15 4561.04326250719
16 5149.31749243704
17 4975.21747862893
18 5300.25662541497
19 5301.88663031536
};
\addlegendentry{PMDL-MBPO}
\addplot [very thick, blue]
table {%
0 340.318195815136
1 421.212344287821
2 607.47238174684
3 1072.00549866971
4 1840.60653756508
5 2175.4622062368
6 2712.70768637385
7 3077.25425782354
8 3588.02334510192
9 3534.92457342795
10 3913.92696296236
11 3903.77314908988
12 4135.31114002786
13 4174.31403967967
14 4359.84808433989
15 4091.0496064983
16 4528.5053004295
17 4508.97991952851
18 4702.31813217644
19 4610.1504280341
};
\addlegendentry{Exponentially Decay-MBPO}
\addplot [very thick, green]
table {%
0 237.483873944119
1 459.907794654865
2 531.490392307569
3 622.836952422731
4 902.66296102875
5 1570.36992100984
6 2136.14816434008
7 2842.43054346432
8 2969.20299365593
9 3031.12209508637
10 3680.97011596772
11 3685.61688063076
12 3854.34453300844
13 3930.97403439106
14 4006.30947955272
15 3947.05472241867
16 3743.55785830321
17 3707.04761587335
18 3860.14333754743
19 4365.85138074733
};
\addlegendentry{MBPO}
\end{axis}

\end{tikzpicture}}
}
\hfil
\subfigure[Humanoid]{
\resizebox{0.22\textwidth}{!}{% This file was created with tikzplotlib v0.10.1.
\begin{tikzpicture}

\definecolor{darkgray176}{RGB}{176,176,176}
\definecolor{green}{RGB}{0,128,0}
\definecolor{lightgray204}{RGB}{204,204,204}

\begin{axis}[
legend cell align={left},
legend style={
  fill opacity=0.8,
  draw opacity=1,
  text opacity=1,
  at={(0.97,0.03)},
  anchor=south east,
  draw=lightgray204
},
tick align=outside,
tick pos=left,
x grid style={darkgray176},
xlabel={Steps},
xmajorgrids,
xmin=-1.45, xmax=30.45,
xtick style={color=black},
xtick={0,7.25,14.5,21.75,29},
xticklabels={0k,75k,150k,225k,300k},
y grid style={darkgray176},
ylabel={Returns},
ymajorgrids,
ymin=-33.9407561874706, ymax=6482.24794238218,
ytick style={color=black}
]
\path [draw=red, fill=red, opacity=0.2]
(axis cs:0,293.379759144462)
--(axis cs:0,262.249639202059)
--(axis cs:1,312.028911904168)
--(axis cs:2,352.296337344251)
--(axis cs:3,422.791058891066)
--(axis cs:4,486.406471405407)
--(axis cs:5,583.295073777022)
--(axis cs:6,778.579857961613)
--(axis cs:7,1228.21960035805)
--(axis cs:8,1702.98346170086)
--(axis cs:9,3277.84013854808)
--(axis cs:10,3634.42620414552)
--(axis cs:11,4704.07890292871)
--(axis cs:12,4884.35343595972)
--(axis cs:13,5082.07393905957)
--(axis cs:14,5268.38884676143)
--(axis cs:15,5433.33962764952)
--(axis cs:16,5315.10972451102)
--(axis cs:17,5118.49900215657)
--(axis cs:18,5503.88897652363)
--(axis cs:19,5200.47017190126)
--(axis cs:20,5203.18048886353)
--(axis cs:21,5771.17091453071)
--(axis cs:22,5342.08967444363)
--(axis cs:23,5665.86452175394)
--(axis cs:24,5468.72151860379)
--(axis cs:25,5124.49767200019)
--(axis cs:26,4886.5527230931)
--(axis cs:27,5770.43700716772)
--(axis cs:28,5499.14408104311)
--(axis cs:29,5717.78284163001)
--(axis cs:29,6047.84766156622)
--(axis cs:29,6047.84766156622)
--(axis cs:28,5961.38771723592)
--(axis cs:27,6186.05754699265)
--(axis cs:26,5927.56431352753)
--(axis cs:25,5790.42699149185)
--(axis cs:24,5831.53930627517)
--(axis cs:23,6001.24537991605)
--(axis cs:22,5741.87900368354)
--(axis cs:21,6094.53748990904)
--(axis cs:20,5631.08839924593)
--(axis cs:19,5617.63008411092)
--(axis cs:18,5903.43393403658)
--(axis cs:17,5582.4523606005)
--(axis cs:16,5768.45901127122)
--(axis cs:15,5707.31002799894)
--(axis cs:14,5765.76399807783)
--(axis cs:13,5514.85157720803)
--(axis cs:12,5190.1333497251)
--(axis cs:11,5169.3285706655)
--(axis cs:10,4102.88232263952)
--(axis cs:9,3842.13642491002)
--(axis cs:8,2451.76427665433)
--(axis cs:7,1599.01568480955)
--(axis cs:6,831.714367228558)
--(axis cs:5,648.959595524236)
--(axis cs:4,524.678762466443)
--(axis cs:3,465.063317970949)
--(axis cs:2,401.412982804575)
--(axis cs:1,351.259713470131)
--(axis cs:0,293.379759144462)
--cycle;

\path [draw=blue, fill=blue, opacity=0.2]
(axis cs:0,314.527623528434)
--(axis cs:0,278.499962614377)
--(axis cs:1,373.034927947535)
--(axis cs:2,435.690216655486)
--(axis cs:3,449.587152949852)
--(axis cs:4,527.68045562028)
--(axis cs:5,533.097080595914)
--(axis cs:6,596.390937572372)
--(axis cs:7,659.474270566417)
--(axis cs:8,593.12550799406)
--(axis cs:9,787.058330735919)
--(axis cs:10,867.666574539196)
--(axis cs:11,1041.40506284171)
--(axis cs:12,1353.6055186822)
--(axis cs:13,2067.06975732248)
--(axis cs:14,2234.33668154917)
--(axis cs:15,2584.15761499214)
--(axis cs:16,3026.55013259395)
--(axis cs:17,3798.4470540346)
--(axis cs:18,3613.03710847296)
--(axis cs:19,3779.21383255856)
--(axis cs:20,4127.67960800709)
--(axis cs:21,4838.41578640681)
--(axis cs:22,4880.9127111498)
--(axis cs:23,4797.49811232122)
--(axis cs:24,4875.95874132537)
--(axis cs:25,4338.03774811518)
--(axis cs:26,4453.29389743861)
--(axis cs:27,4870.51172015314)
--(axis cs:28,4609.73464280111)
--(axis cs:29,4811.15207775413)
--(axis cs:29,5357.43489955611)
--(axis cs:29,5357.43489955611)
--(axis cs:28,5282.87637094031)
--(axis cs:27,5787.25534236962)
--(axis cs:26,5715.05947569199)
--(axis cs:25,4995.35258322874)
--(axis cs:24,5375.5454553413)
--(axis cs:23,5366.44260403647)
--(axis cs:22,5398.81277529893)
--(axis cs:21,5324.81796318112)
--(axis cs:20,5043.92130273482)
--(axis cs:19,4987.71882671442)
--(axis cs:18,5115.81231314118)
--(axis cs:17,5356.55332123452)
--(axis cs:16,4016.43903231823)
--(axis cs:15,3245.07096244378)
--(axis cs:14,3224.76124893215)
--(axis cs:13,2535.61355321215)
--(axis cs:12,1699.9784198877)
--(axis cs:11,1192.64799337637)
--(axis cs:10,945.529747492206)
--(axis cs:9,902.05402128784)
--(axis cs:8,694.346127448603)
--(axis cs:7,712.709193796616)
--(axis cs:6,599.237397774394)
--(axis cs:5,569.236458101749)
--(axis cs:4,539.853112733289)
--(axis cs:3,463.254189687401)
--(axis cs:2,477.597440616901)
--(axis cs:1,389.864156867523)
--(axis cs:0,314.527623528434)
--cycle;

\path [draw=green, fill=green, opacity=0.2]
(axis cs:0,312.766488438832)
--(axis cs:0,276.444943517278)
--(axis cs:1,359.368445726084)
--(axis cs:2,419.772490057525)
--(axis cs:3,451.766892418066)
--(axis cs:4,538.06555199475)
--(axis cs:5,522.852014235481)
--(axis cs:6,598.365332583554)
--(axis cs:7,712.780809066453)
--(axis cs:8,574.69836686474)
--(axis cs:9,752.606655331269)
--(axis cs:10,843.409120443478)
--(axis cs:11,1006.58291295922)
--(axis cs:12,1202.44743027391)
--(axis cs:13,1230.25198763004)
--(axis cs:14,1740.37082661736)
--(axis cs:15,1761.47204710578)
--(axis cs:16,2094.066042739)
--(axis cs:17,3457.1807992324)
--(axis cs:18,3508.69373425103)
--(axis cs:19,3509.5330915972)
--(axis cs:20,4162.58688324627)
--(axis cs:21,4499.89398386713)
--(axis cs:22,4015.49053655406)
--(axis cs:23,4734.29766765579)
--(axis cs:24,4320.67516538324)
--(axis cs:25,4477.63206296084)
--(axis cs:26,3573.9879484993)
--(axis cs:27,4222.47951642472)
--(axis cs:28,3974.77050377151)
--(axis cs:29,4100.6442739705)
--(axis cs:29,4197.9404680367)
--(axis cs:29,4197.9404680367)
--(axis cs:28,4256.56545088219)
--(axis cs:27,5107.47447555125)
--(axis cs:26,4880.01174288247)
--(axis cs:25,5117.32969291002)
--(axis cs:24,4771.49232979211)
--(axis cs:23,5253.75632750508)
--(axis cs:22,4586.67899940075)
--(axis cs:21,5011.01808232051)
--(axis cs:20,5023.34091827134)
--(axis cs:19,4533.94167135461)
--(axis cs:18,4987.9879159872)
--(axis cs:17,5128.40965299155)
--(axis cs:16,3208.25695501193)
--(axis cs:15,2438.30577714736)
--(axis cs:14,2751.98700354815)
--(axis cs:13,1686.24727437964)
--(axis cs:12,1592.57858305173)
--(axis cs:11,1176.13479816358)
--(axis cs:10,931.152496918828)
--(axis cs:9,881.597146604548)
--(axis cs:8,685.396983723111)
--(axis cs:7,739.56711048735)
--(axis cs:6,601.21795852049)
--(axis cs:5,563.389492475073)
--(axis cs:4,549.251323471306)
--(axis cs:3,466.546747478691)
--(axis cs:2,463.427793271885)
--(axis cs:1,376.897565349574)
--(axis cs:0,312.766488438832)
--cycle;

\addplot [very thick, red]
table {%
0 277.273965210218
1 331.78764919218
2 375.375264751006
3 441.984303036159
4 506.7975588084
5 616.836778758905
6 805.09630061874
7 1403.85265546498
8 2086.66752802226
9 3569.1468895311
10 3869.78446787532
11 4932.01378993114
12 5034.35245616136
13 5298.49932752546
14 5520.85054882001
15 5574.04172938236
16 5551.62212665899
17 5338.52027279999
18 5713.14618512254
19 5426.93431091233
20 5424.4829989319
21 5933.6357935447
22 5537.43276573512
23 5824.96147828715
24 5664.41802978711
25 5477.36799046138
26 5451.96618287671
27 5984.83059145379
28 5723.62845109407
29 5881.78537019506
};
\addlegendentry{PMDL-MBPO}
\addplot [very thick, blue]
table {%
0 296.264040961684
1 381.76860518396
2 455.11739069058
3 456.806349983734
4 533.497142674959
5 550.842942320419
6 597.759690531354
7 686.0099307131
8 645.966570784109
9 839.514945074799
10 905.58691158792
11 1118.48493094903
12 1533.48303525386
13 2318.381932867
14 2723.13985071113
15 2907.99450725292
16 3546.38546005511
17 4612.66167998108
18 4401.26637775178
19 4429.26348705497
20 4630.49060714554
21 5088.01431797298
22 5142.59477055921
23 5100.65998244261
24 5133.42333021567
25 4665.05040856224
26 5077.56795642977
27 5309.48819254701
28 4935.90262383725
29 5071.39352300142
};
\addlegendentry{Exponentially Decay-MBPO}
\addplot [very thick, green]
table {%
0 294.044125116195
1 368.122999623854
2 440.056372909609
3 459.570806287108
4 543.790823284733
5 542.734149881577
6 599.7700114774
7 727.448378246857
8 623.306819527295
9 814.179002131876
10 887.916293246013
11 1096.03035029619
12 1398.38941254313
13 1457.77051922359
14 2223.60306602676
15 2085.05669242267
16 2684.51035712371
17 4317.05520669832
18 4269.36878076535
19 4042.19984051478
20 4607.49916533563
21 4749.13513150322
22 4295.13702234173
23 4994.28156404408
24 4562.19722957675
25 4813.55759584313
26 4263.2072220877
27 4689.30136282126
28 4124.25389479369
29 4148.00533279534
};
\addlegendentry{MBPO}
\end{axis}

\end{tikzpicture}}
}
\hfil
\subfigure[Walker2d one-step error]{
\resizebox{0.22\textwidth}{!}{\input{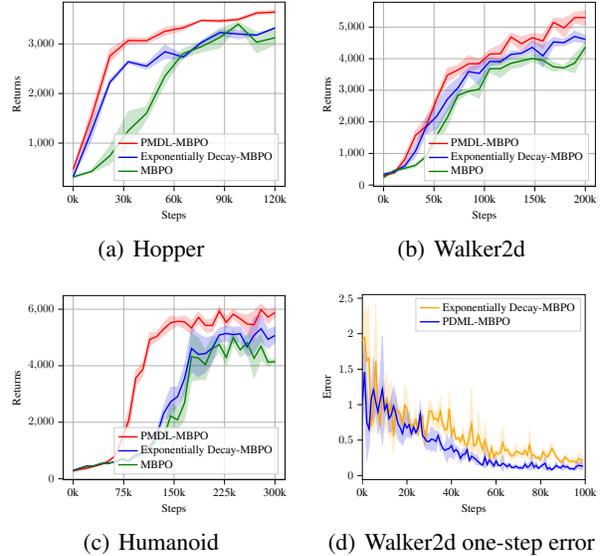}}
}
\caption{Comparison with exponentially decay prioritization.}
\label{ablation study: exponentially decay}
\end{figure}

\begin{table*}[!htb]
\setlength{\abovecaptionskip}{0.2cm}
\centering
\caption{Asymptotic performance of different exponentially decay rate.}
\begin{tabular}{c|cccccc}
\toprule 
 & decay rate 0.98 & decay rate 0.995 & decay rate 0.997 & decay rate 0.999 & MBPO & PDML\\
\midrule
Hopper & 3320.04 & 3291.93 & 3382.63 & 3374.52 & 3125.56 & \textbf{3641.07}  \\
Walker2d & 4609.29 & 4571.64 & 4643.15 & 4595.31 & 4366.37 & \textbf{5304.42} \\
Humanoid & 5070.58 & 5198.34 & 5092.56 & 5149.71 & 4148.15 & \textbf{5885.14} \\

\bottomrule
\end{tabular}
\label{tab: performance exponentially decay}
\end{table*}
\section{Related Work}

\textbf{Model adaptation.} \quad
Several adaptive control approaches \citep{{sastry1989adaptive},{pastor2011online},{meier2016towards}} aim to train a dynamics model that can adapt online. 
However, scaling such methods to complex tasks is exponentially difficult. 
Adaptive learning in the dynamics model has also been studied in inverse dynamics learning tasks.  
A drifting Gaussian process (GP) keeps a history of a constant number of recently observed data points and updates its hyper-parameters at each time step \citep{meier2016drifting}. 
The drifting Gaussian process (GP) predicts the local dynamics errors to control the learning rate \citep{meier2016towards}, resulting in more online hyperparameter learning and adaptive function approximator robustness. 
Our method is different from these works that we learn a forward model which can always adapt to the evolving policy.
% These works contrast \fh{differ from} our work in adapting the policy to learn a local forward model \fhc{what do you mean here? they adapt the policy but we don't?}. \joy{this was my mistake. These works differ from our work in locally adaptive learning of inverse dynamics}
Some studies focus on an adaptive model predictive control for constrained linear systems \citep{tanaskovic2013adaptive} and guaranteeing safety, robustness, and convergence in a quadrotor helicopter testbed \citep{aswani2012extensions}. 
% Such an adaptive model is also important in the natural language process. 
% Dynamic evaluation algorithms \citep{{rei2015online},{krause2018dynamic},{krause2016multiplicative},{fortunato2017bayesian}} are concerned with learning a model from recent history via a gradient descent-based mechanism. 
Our work closely relates to a model adaptation in forward models from \cite{{fu2016one}, {nagabandi2018learning}, {nagabandi2018deep}, {lee2020context}, {guo2022relational}}. 
These methods use meta-learning to train a dynamics model as a prior and then combine it with recent data to rapidly adapt to the new task. 
However, these works are mainly about model transfer under different dynamics.
Different from their works, we study learning an accurate dynamics model for policy learning under a fixed transition dynamics, and we also provide theoretical analysis to motivate our method.
More related works about model-based RL are provided in Appendix~\ref{app:additional-relatedworks}.

\textbf{Prioritized experience replay.} \quad 
Another related line of work is prioritized experience replay in reinforcement learning.
This solves a classic issue in model-free RL.
Previous work \cite{katharopoulos2018not} claimed that emphasizing essential samples in the replay buffer can benefit off-policy RL algorithms.
Prioritized Experience Replay (PER) \cite{schaul2016prioritized} measured the importance of sample by temporal-difference (TD) error.
Based on this work, many methods are proposed to perform prioritized sampling.
% \cite{{novati2019remember},{jiang2021prioritized},{liu2021regret}, {wang2020striving}, {fedus2020revisiting}}
Some methods \citep{{brittain2019prioritized},{lee2019sample},{fujimoto2020equivalence},{jiang2021prioritized},{liu2021regret},{lahire2021large},{oh2022modelaugmented}} extend or explain PER from different perspectives, and others \citep{{novati2019remember},{fedus2020revisiting}} propose to prioritize samples according to their age.
Our work is different from experience replay works in model-free RL in the following points:
\textbf{(1)} In model learning, we re-weight the state-action visitation distribution that generates a batch of samples, rather than a single sample as in model-free RL.
\textbf{(2)} During weighting, we use the distance between the policy distribution that each sample generated from and the policy distribution of the current policy as a metric, rather than how much improvement each sample can bring to the policy.
\textbf{(3)} We provide very detailed theoretical result to analyze how to reweight the samples for model learning.

\section{Conclusion and Discussion}
\label{sec:dis}
In this paper, we introduce a novel dynamics model learning method for model-based RL called PDML, which learns a policy-adapted dynamics model based on a dynamically adjusted historical policy mixture distribution.
This policy-adapted dynamics model can continually adapt to the state-action visitation distribution of the evolving policy.
This makes it more accurate than the previous dynamics model when making predictions during model rollouts.
We also provide theoretical analysis and experimental results to motivate our method.
After combining with the state-of-the-art model-based method MBPO, PDML achieves better asymptotic performance and higher sample efficiency than previous state-of-the-art model-based methods in MuJoCo.
We believe our work takes an important step toward more sample-efficient RL.
One limitation of our work is that the generalization ability of the policy-adapted dynamics model may not be strong enough because we focus on fitting the samples induced by the evolving policy to improve the convergence speed of the policy. 
Therefore, our method is efficient for task-specific problems but may not perform well for some exploration-oriented tasks.
We leave this direction to future work.

\section*{Acknowledgement}
Wang, Wongkamjan and Huang are supported by National Science Foundation NSF-IIS-FAI program, DOD-ONR-Office of Naval Research, DOD Air Force Office of Scientific Research, DOD-DARPA-Defense Advanced Research Projects Agency Guaranteeing AI Robustness against Deception (GARD), Adobe, Capital One and JP Morgan faculty fellowships.

\bibliography{include/reference}
\bibliographystyle{icml2023}

%%%%%%%%%%%%%%%%%%%%%%%%%%%%%%%%%%%%%%%%%%%%%%%%%%%%%%%%%%%%%%%%%%%%%%%%%%%%%%%
%%%%%%%%%%%%%%%%%%%%%%%%%%%%%%%%%%%%%%%%%%%%%%%%%%%%%%%%%%%%%%%%%%%%%%%%%%%%%%%
% APPENDIX
%%%%%%%%%%%%%%%%%%%%%%%%%%%%%%%%%%%%%%%%%%%%%%%%%%%%%%%%%%%%%%%%%%%%%%%%%%%%%%%
%%%%%%%%%%%%%%%%%%%%%%%%%%%%%%%%%%%%%%%%%%%%%%%%%%%%%%%%%%%%%%%%%%%%%%%%%%%%%%%
\newpage
\appendix
\onecolumn
\begin{center}
\icmltitle{Appendix}
\end{center}
\section{Pseudo Code of PDML-MBPO}\label{app:pdml_mbpo}
In Algorithm \ref{Alg_PDML-MBPO}, we demonstrate the pseudo code of PDML-MBPO.
\begin{algorithm}[htbp]
\caption{PDML-MBPO}
\label{Alg_PDML-MBPO}
\begin{algorithmic}[1]
\REQUIRE current policy proportion hyperparameter $\alpha$, interaction epochs $I$, rollout horizon $h$ \
\STATE Initialize historical policy sequence $k\gets 0, \Pi^{k} \gets \emptyset$
\FOR{$I$ epochs}
    \STATE Interact with the environment using current policy $\pi_c$, add samples into real sample buffer $\mathbb{D}_e$ \ 
    \STATE Add current policy $\pi_c$ into historical policy sequence: $\pi_{k}\gets \pi_c$, $\Pi^{k} \gets \{\Pi^{k-1}, \pi_k\}$ \
    \STATE Adjust the historical policy mixture distribution $\bm{w}^k=[w_1^k,\ldots,w_k^k]$ via Equation~(\ref{policy_weight}) and (\ref{current_policy_weight}) \
    \STATE Normalize $\bm{w}_k \gets {\bm{w}_k}/{\lVert \bm{w}_k \rVert}$
    \STATE Sample a training data batch of $(s_n,a_n,r,s_{n+1})$ from $\mathbb{D}_e$ according to $\bm{w}^k$ \ 
    \STATE Train dynamics model $\hat{T}_\theta$ via Equation~(\ref{model_loss})
    \FOR{$M$ model rollouts}
        \STATE Sample initial rollout states from real sample buffer $\mathbb{D}_e$ according to $\bm{w}^k$ \
        \STATE Use current policy $\pi_c$ to perform $h$-step model rollouts, add model-generated samples into model sample buffer $\mathbb{D}_m$ \
    \ENDFOR
    \FOR{$G$ gradient updates}
        \STATE Update current policy $\pi_c$ using model-generated samples from model sample buffer $\mathbb{D}_m$ \
    \ENDFOR
    \STATE $k\gets k+1$
\ENDFOR
\end{algorithmic}
\end{algorithm}
\section{Additional Related Work}\label{app:additional-relatedworks}

\textbf{Model-based reinforcement learning.}
Model-based RL is proposed as a solution to reduce the sample complexity of model-free RL by learning a dynamics model.
Current model-based RL mainly focuses on better model learning and better model usage.
To learn a model with more accuracy, many model architectures have been proposed, such as linear models \cite{{parr2008analysis},{sutton2008dyna},{kumar2016optimal}} and nonparametric Gaussian processes \cite{{rasmussen2004gaussian},{deisenroth2011pilco}}.
With the rapid development of deep learning, neural networks have become a popular choice of model architecture in recent years \cite{{kurutach2018model},{chua2018deep}}.
% To maintain the uncertainty of the learned dynamics model and regularize the learning process, \citet{kurutach2018model} proposed to use an ensemble of neural networks.
% Further, \citet{chua2018deep} suggested to use the probabilistic neural network ensemble instead of the deterministic neural network ensemble to capture both aleatoric uncertainty and epistemic uncertainty.
Moreover, to reduce the model error, a multi-step model \cite{asadi2019combating} was designed to directly predict the transition of an action sequence input, 
and \citet{shen2020model} used unsupervised model adaptation to reduce the potential data distribution mismatch.
For better model usage, \citet{janner2019trust} proved that short model rollouts could avoid the model error and improve the quality of model samples. 
Based on this, \citet{29} proposed a bidirectional model rollout scheme to avoid the model error further. 
Furthermore, model disagreement was used to decide when to trust the model \cite{{pan2020trust}} and regularize the model samples \cite{{yu2020mopo}}.
Besides, 
% there are some other model-based methods that have made good progress.
\citet{luo2018algorithmic} provided a theoretical guarantee of monotone expected reward improvement of model-based RL.
\citet{rajeswaran2020game} cast model-based RL as a game-theoretic framework by formulating the optimization of model and policy as a two-player game.
To save time tuning hyperparameters, \citet{lai2021effective} designed an automatic scheduling framework.
% the hyperparameters of MBPO \cite{janner2019trust}.
\citet{abbas2020selective} systematically studied how the model capacity affects the model-based methods.
\citet{sun2022transfer} investigated how to use dynamics models to improve the sample efficiency of policy learning when observation space changes.

\textbf{Value-equivalence dynamics model.}
Value-equivalence dynamics model has been noted by several authors in recent years.
% It is proposed by \citet{farahmand2017value} to solve the objective mismatch between model learning and policy learning \citep{lambert2020objective} in model-based RL.
Since learning an accurate dynamics model of the world remains challenging and often requires computationally costly and data-hungry models \citep{lovatto2020decision}, 
 \citet{farahmand2017value} proposed value-aware model learning which aims to learn a value-equivalence model that induces the same Bellman operator as the real environment, rather than accurately predicting transitions.
However, they replaced the value function with the supremum over a function space, and it is difficult to find a supremum for a function space parameterized by complex function approximators like neural networks.
Based on this work, \citet{farahmand2018iterative} proposed Iterative Value-Aware Model Learning (IterVAML) which replaced the supremum over a value function space with the value function at current iteration.
Besides, \citet{grimm2020value} introduced value equivalence principle and analysed how the space of possible solutions on model learning is impacted by the choice of policies and functions, 
and \citet{zheng2023model} provides a deep insight on why model ensemble performs well based on value equivalence principle.
However, despite very detailed theoretical guarantees, there is still a performance gap between the value-equivalence dynamics model in the practical implementation and the model trained by the maximum likelihood estimate \citep{lovatto2020decision}.
\citet{eysenbach2021mismatched} introduced a novel objective to jointly train the model and the policy.
\citet{voelcker2022value} proposed Value-Gradient weighted Model loss (VaGraM) which approximated the value-aware model loss function with a Taylor expansion of value function and achieved SOTA performance across all value-aware model learning methods.
Like our method, VaGraM also tries to learn a locally accurate dynamics model. 
The difference is that our method aims to learn the samples that the current policy may encounter as accurately as possible,
while VaGraM is to learn the dimensions in the state that can bring the greatest improvement to policy learning. 
Experimental results demonstrate that our method outperforms VaGraM in practice.

\section{Useful lemma}

\begin{lemma}\label{lemma_1}\cite{shen2020model}
Assume the initial state distributions of the real dynamics $T$ and the learned dynamics model $\hat{T}$ are the same.
For any state $s^{\prime}$, assume $\mathcal{F}_{s^{\prime}}$ is a class of real-valued bounded measurable functions on state-action space,
such that $\hat{T}(s^{\prime}| \cdot, \cdot): \mathcal{S} \times \mathcal{A} \rightarrow \mathbb{R} $ is in $\mathcal{F}_{s^{\prime}}$.
Then the gap between two different state visitation distributions $v^{\pi_1}_T(s^\prime)$ and $v^{\pi_2}_{\hat{T}}(s^\prime)$ can be bounded as follows:

\begin{equation}
|v^{\pi_1}_T(s^\prime)- v^{\pi_2}_{\hat{T}}(s^\prime)| \leq \gamma \mathbb{E}_{(s,a) \sim \rho^{\pi_1}_{T}}{ |T(s^\prime|s,a) - \hat{T}(s^\prime|s,a)|} + \gamma d_{\mathcal{F}_{s^{\prime}}}(\rho^{\pi_1}_{T}, \rho^{\pi_2}_{\hat{T}})
\end{equation}
\end{lemma}
% \joy{subscript $|T - \hat{T}|$ has typo as in eq. 9}
% \joy{question: will $\pi^*$ be confused with optimal policy? (I did when I read carelessly) }

\begin{proof}
For any state visitation distribution $v^{\pi}_T$, we have:

\begin{equation}
v^{\pi}_T(s^\prime) = (1-\gamma)v_0(s^\prime) + \gamma \int_{(s,a)} \rho^{\pi}_T(s,a) T(s^\prime|s,a) {\rm d}s{\rm d}a,
\end{equation}

where $v_0$ is the probability of the initial state being the state $s^\prime$. Then the gap between two different state visitation distributions is:

\begin{equation}
\begin{aligned}
& |v^{\pi_1}_T(s^\prime)- v^{\pi_2}_{\hat{T}}(s^\prime)| \\
= &  \gamma \left|\int_{(s,a)} \rho^{\pi_1}_T(s,a) T(s^\prime|s,a) - \rho^{\pi_2}_{\hat{T}(s,a)} \hat{T}(s^\prime|s,a) {\rm d}s{\rm d}a \right| \\
= &  \gamma \left|\mathbb{E}_{(s,a) \sim \rho^{\pi_1}_{T}}[T(s^\prime|s,a)] - \mathbb{E}_{(s,a) \sim \rho^{\pi_2}_{\hat{T}}}[\hat{T}(s^\prime|s,a)] \right| \\
\leq &  \gamma \left|\mathbb{E}_{(s,a) \sim \rho^{\pi_1}_{T}}[T(s^\prime|s,a) - \hat{T}(s^\prime|s,a)] \right| + \gamma \left| \mathbb{E}_{(s,a) \sim \rho^{\pi_1}_{T}}[\hat{T}(s^\prime|s,a)] - \mathbb{E}_{(s,a) \sim \rho^{\pi_2}_{\hat{T}}}[\hat{T}(s^\prime|s,a)] \right| \\
\leq &  \gamma \mathbb{E}_{(s,a) \sim \rho^{\pi_1}_{T}} {|T(s^\prime|s,a) - \hat{T}(s^\prime|s,a)|} + \gamma d_{\mathcal{F}_{s^{\prime}}}(\rho^{\pi_1}_{T}, \rho^{\pi_2}_{\hat{T}})
\end{aligned}
\end{equation}

\end{proof}

\section{Proof of main theorem}
\label{proof_main}

\begin{theorem}
Given the historical policy mixture $\pi_{\text{mix}, k} = (\Pi^k, \bm{w}^k)$ at iteration step $k$, we denote $\xi_{\rho_i} = D_{TV}(\rho^{\pi}_{T}(s,a) || \rho^{\pi_i}_{T}(s,a))$ and $\xi_{\pi_i} = \mathbb{E}_{s \sim v^{\pi_{\text{mix}}}_{\hat{T}}} \left[ D_{TV}(\pi(a|s)|| \pi_i(a|s)) \right]$ as the state-action visitation distribution shift and the policy distribution shift between the historical policy $\pi_i$ and current policy $\pi$ respectively, where $v^{\pi_{\text{mix}}}_{\hat{T}}$ is the state visitation distribution of policy mixture under the learned dynamics model.
$r_{\text{max}}$ is the maximum reward the policy can get from the real environment, $\gamma$ is the discount factor, and ${\rm Vol}(\mathcal{S})$ is the volume of state space. 
Then the performance gap between the real environment rollout $J(\pi, T)$ and the model rollout $J(\pi, \hat{T})$ can be bounded as follows:

\begin{equation}\label{main_theory_appendix}
\begin{aligned}
 J(\pi, T)  - J(\pi, \hat{T}) \leq \ & 2 \gamma r_{\text{max}} \mathbb{E}_{(s,a) \sim \rho^{\pi}_{T} }[D_{TV} (T(s^\prime|s,a) || \hat{T}(s^\prime|s,a))] \\
&  + r_{\text{max}} \sum^{k}_{i=0} w^k_i (\gamma {\rm Vol}(\mathcal{S}) \xi_{\rho_i} + 2 \xi_{\pi_i}) \\
 &  + 2 r_{\text{max}} D_{TV}(\rho^{\pi_{\text{mix}}}_{\hat{T}}(s,a) || \rho^{\pi}_{\hat{T}}(s,a))\\
\end{aligned}
\end{equation}

\end{theorem}

\begin{proof}
\begin{equation}
\begin{aligned}
& \left| J(\pi, T) - J(\pi, \hat{T}) \right| \\
= &  \left| J(\pi, T) - J(\pi_{\text{mix}}, \hat{T}) + J(\pi_{\text{mix}}, \hat{T}) - J(\pi, \hat{T}) \right| \\
\leq &  \underbrace{\left| \int_{(s,a)} (\rho^{\pi}_{T}(s,a) - \rho^{\pi_{\text{mix}}}_{\hat{T}}(s,a)) r(s,a) {\rm d}s{\rm d}a \right|}_{term 1} + \underbrace{\left| \int_{(s,a)} (\rho^{\pi_{\text{mix}}}_{\hat{T}}(s,a) - \rho^{\pi}_{\hat{T}}(s,a)) r(s,a) {\rm d}s{\rm d}a \right|}_{term 2} \\
\end{aligned}
\end{equation}

For term 1:
\begin{equation}\label{proof_term1}
\begin{aligned}
& \left| \int_{(s,a)} (\rho^{\pi}_{T}(s,a) - \rho^{\pi_{\text{mix}}}_{\hat{T}}(s,a)) r(s,a) {\rm d}s{\rm d}a \right| \\
= & \left| \int_{(s,a)} (v^{\pi}_{T}(s) \pi(a|s) - v^{\pi_{\text{mix}}}_{\hat{T}}(s) \pi_{\text{mix}}(a|s)) r(s,a) {\rm d}s{\rm d}a \right| \\
= & \left| \int_{(s,a)} (v^{\pi}_{T}(s) \pi(a|s) - v^{\pi_{\text{mix}}}_{\hat{T}}(s) \pi(a|s) + v^{\pi_{\text{mix}}}_{\hat{T}}(s) \pi(a|s) - v^{\pi_{\text{mix}}}_{\hat{T}}(s) \pi_{\text{mix}}(a|s)) r(s,a) {\rm d}s{\rm d}a \right| \\
\leq & \left| \int_{(s,a)} (v^{\pi}_{T}(s) - v^{\pi_{\text{mix}}}_{\hat{T}}(s)) \pi(a|s) r(s,a) {\rm d}s{\rm d}a \right| + \left| \int_{(s,a)} (v^{\pi_{\text{mix}}}_{\hat{T}}(s) (\pi(a|s) - \pi_{\text{mix}}(a|s)) r(s,a) {\rm d}s{\rm d}a \right| \\
\leq & r_{\text{max}} \int_s \left| v^{\pi}_{T}(s) - v^{\pi_{\text{mix}}}_{\hat{T}}(s) \right| {\rm d}s + 2 r_{\text{max}} \mathbb{E}_{s \sim v^{\pi_{\text{mix}}}_{\hat{T}}} \left[ D_{TV}(\pi(a|s)||\pi_{\text{mix}}(a|s)) \right] \\
\end{aligned}
\end{equation}

For the first term of last inequality in Eq. \ref{proof_term1}, according to Lemma. \ref{lemma_1} we have:
\begin{equation}\label{eq_mid1}
\begin{aligned}
& r_{\text{max}} \int_s \left| v^{\pi}_{T}(s) - v^{\pi_{\text{mix}}}_{\hat{T}}(s) \right| {\rm d}s \\
\leq & r_{\text{max}} \gamma \mathbb{E}_{(s,a) \sim \rho^{\pi}_{T} } \int_{s^\prime} \left| T(s^\prime|s,a) - \hat{T}(s^\prime|s,a) \right| {\rm d}{s^\prime} + r_{\text{max}} \gamma \int_{s^\prime} d_{\mathcal{F}_{s^{\prime}}}(\rho^{\pi}_{T}, \rho^{\pi^*}_{\hat{T}}) {\rm d}{s^\prime} \\
\end{aligned}
\end{equation}

We use total variance distance as the $\mathcal{F}_{s^{\prime}}$ to measure the distance between $\rho^{\pi}_{T}$ and  $\rho^{\pi_{\text{mix}}}_{\hat{T}}$. 
Suppose we can learn a dynamics model that can perfectly adapt the state-action visitation distribution of $\pi_{\text{mix}}$, which means the difference between the model prediction and the environment next state $s^\prime$ is very small, and the state-action visitation density induced by the learned dynamics model $\rho^{\pi_{\text{mix}}}_{\hat{T}}$ is approximately equal to $\rho^{\pi_{\text{mix}}}_{{T}}$.
This assumption is required by many model-based RL methods \cite{voelcker2022value}.
Then Eq. \ref{eq_mid1} can be expressed as:

\begin{equation}\label{eq_mid2}
\begin{aligned}
& r_{\text{max}} \int_s \left| v^{\pi}_{T}(s) - v^{\pi_{\text{mix}}}_{\hat{T}}(s) \right| {\rm d}s \\
\leq & r_{\text{max}} \gamma \mathbb{E}_{(s,a) \sim \rho^{\pi}_{T} } \int_{s^\prime} \left| T(s^\prime|s,a) - \hat{T}(s^\prime|s,a) \right| {\rm d}{s^\prime} + r_{\text{max}} \gamma \int_{s^\prime} D_{TV}(\rho^{\pi}_{T} || \rho^{\pi_{\text{mix}}}_{T}) {\rm d}{s^\prime} \\
\leq & 2 \gamma r_{\text{max}} \mathbb{E}_{(s,a) \sim \rho^{\pi}_{T} }[D_{TV} (T(s^\prime|s,a) || \hat{T}(s^\prime|s,a))] + \gamma {\rm Vol}(\mathcal{S}) r_{\text{max}}  D_{TV}(\rho^{\pi}_{T} || \rho^{\pi_{\text{mix}}}_{T})
\end{aligned}
\end{equation}

Combined Eq. \ref{proof_term1} with Eq. \ref{eq_mid2}, we can get:

\begin{equation}\label{eq_mid3}
\begin{aligned}
& \left| \int_{(s,a)} (\rho^{\pi}_{T}(s,a) - \rho^{\pi_{\text{mix}}}_{\hat{T}}(s,a)) r(s,a) {\rm d}s{\rm d}a \right| \\
\leq & 2 \gamma r_{\text{max}} \mathbb{E}_{(s,a) \sim \rho^{\pi}_{T} }[D_{TV} (T(s^\prime|s,a) || \hat{T}(s^\prime|s,a))] + \gamma {\rm Vol}(\mathcal{S}) r_{\text{max}}  D_{TV}(\rho^{\pi}_{T}(s,a) || \rho^{\pi_{\text{mix}}}_{T}(s,a)) \\
\quad & + 2 r_{\text{max}} \mathbb{E}_{s \sim v^{\pi_{\text{mix}}}_{\hat{T}}} \left[ D_{TV}(\pi(a|s)||\pi_{\text{mix}}(a|s)) \right] \\
= & 2 \gamma r_{\text{max}} \mathbb{E}_{(s,a) \sim \rho^{\pi}_{T} }[D_{TV} (T(s^\prime|s,a) || \hat{T}(s^\prime|s,a))] + \gamma {\rm Vol}(\mathcal{S}) r_{\text{max}} D_{TV}(\rho^{\pi}_{T}(s,a) || \sum^{k}_{i=0} w_i \rho^{\pi_i}_{T}(s,a)) \\
\quad & + 2 r_{\text{max}} \mathbb{E}_{s \sim v^{\pi_{\text{mix}}}_{\hat{T}}} \left[ D_{TV}(\pi(a|s)||\sum^{k}_{i=0} w_i \pi_i(a|s)) \right] \\
= &  2 \gamma r_{\text{max}} \mathbb{E}_{(s,a) \sim \rho^{\pi}_{T} }[D_{TV} (T(s^\prime|s,a) || \hat{T}(s^\prime|s,a))] + \gamma {\rm Vol}(\mathcal{S}) r_{\text{max}} \sum^{k}_{i=0} w_i D_{TV}(\rho^{\pi}_{T}(s,a) || \rho^{\pi_i}_{T}(s,a)) \\
\quad & + 2 r_{\text{max}} \sum^{k}_{i=0} w_i  \mathbb{E}_{s \sim v^{\pi_{\text{mix}}}_{\hat{T}}} \left[ D_{TV}(\pi(a|s)|| \pi_i(a|s)) \right] \\
\end{aligned}
\end{equation}

Finally, based on Eq. \ref{eq_mid3}, we get:

\begin{equation}\label{final_proof}
\begin{aligned}
& \left| J(\pi, T) - J(\pi, \hat{T}) \right| \\
\leq & 2 \gamma r_{\text{max}} \mathbb{E}_{(s,a) \sim \rho^{\pi}_{T} }[D_{TV} (T(s^\prime|s,a) || \hat{T}(s^\prime|s,a))] + \gamma {\rm Vol}(\mathcal{S}) r_{\text{max}} \sum^{k}_{i=0} w^k_i D_{TV}(\rho^{\pi}_{T}(s,a) || \rho^{\pi_i}_{T}(s,a)) \\
\quad & + 2 r_{\text{max}} \sum^{k}_{i=0} w^k_i  \mathbb{E}_{s \sim v^{\pi_{\text{mix}}}_{\hat{T}}} \left[ D_{TV}(\pi(a|s)|| \pi_i(a|s)) \right] + 2 r_{\text{max}} D_{TV}(\rho^{\pi_{\text{mix}}}_{\hat{T}}(s,a) || \rho^{\pi}_{\hat{T}}(s,a))\\
\leq & 2 \gamma r_{\text{max}} \mathbb{E}_{(s,a) \sim \rho^{\pi}_{T} }[D_{TV} (T(s^\prime|s,a) || \hat{T}(s^\prime|s,a))] + 
r_{\text{max}} \sum^{k}_{i=0} w^k_i (\gamma {\rm Vol}(\mathcal{S}) \xi_{\rho_i} + 2 \xi_{\pi_i}) \\
\quad & + 2 r_{\text{max}} D_{TV}(\rho^{\pi_{\text{mix}}}_{\hat{T}}(s,a) || \rho^{\pi}_{\hat{T}}(s,a)),\\
\end{aligned}
\end{equation}

and the proof is completed.
\end{proof}

\section{Proof of Proposition~\ref{prop_3}}\label{connections}

% In this section, we provide more discussions about the connection between our method and Theorem~\ref{main_theory}.
% First, we give the following theorem:

\begin{proposition} \label{theorem2}
% If the weight $w_i^k$ of each policy $\pi_i$ in the historical policy sequence $\Pi^k$ is negatively related to state action visitation distribution shift $\xi_{\rho_i}$ and the policy distribution shift $\xi_{\pi_i}$ between the historical policy $\pi_i$ and current policy $\pi$, the error of the performance bound can be reduced:
The performance gap can be reduced if the weight $w_i^k$ of each policy $\pi_i$ in the historical policy sequence $\Pi^k$ is negatively related to state action visitation distribution shift $\xi_{\rho_i}$ and the policy distribution shift $\xi_{\pi_i}$ between the historical policy $\pi_i$ and current policy $\pi$ instead of an average weight $w_i^k = \frac{1}{k}$:

\begin{equation}
\begin{aligned}
\sum^{k}_{i=1} w^k_i (\gamma {\rm Vol}(\mathcal{S}) \xi_{\rho_i} + 2 \xi_{\pi_i}) \leq \sum^{k}_{i=1} \frac{1}{k} (\gamma {\rm Vol}(\mathcal{S}) \xi_{\rho_i} + 2 \xi_{\pi_i})
\end{aligned}
\end{equation}

\end{proposition}

\begin{proof}
Each policy $\pi_i$ in the historical policy sequence $\Pi^k$ corresponds to a distribution shift pair $(\xi_{\rho_i}, \xi_{\pi_i})$, and these pairs form a distribution shift sequence $\{(\xi_{\rho_1}, \xi_{\pi_1}),(\xi_{\rho_2}, \xi_{\pi_2}), ......, (\xi_{\rho_k}, \xi_{\pi_k})\}$, assuming that this sequence decreases as $i$ increases (this is a reasonable assumption, because we can always arrange the historical policy sequence into a distribution shift decreasing sequence according to the magnitude of the shift).
As the weight of each policy is negatively related to state action visitation distribution shift $\xi_{\rho_i}$ and the policy distribution shift $\xi_{\pi_i}$, $w_i^k$ increases with $k$.

Since $\sum\limits_{i=1}^k w_i^k=\sum\limits_{i=1}^k \frac{1}{k}=1$, there exists a $k_0$ that for all  $i>k_0, w_i^k> \frac{1}{k}$.

 Then we have:
   \begin{equation}
    \begin{split}
      0\leqslant & \sum\limits_{i=k_0}^{k}(w_i^k-\frac{1}{k})(\gamma {\rm Vol}(\mathcal{S}) \xi_{\rho_i} + 2 \xi_{\pi_i})\\
      \leqslant &\sum\limits_{i=k_0}^{k}(w_i^k-\frac{1}{k})(\gamma {\rm Vol}(\mathcal{S}) \xi_{\rho_{k_1}} + 2 \xi_{\pi_{k_1}}),
    \end{split}
  \end{equation}
  
  where $k_1 \in [k_0, k]$
  
   \begin{equation}
    \begin{split}
      0\geqslant & \sum\limits_{i=1}^{k_0 - 1}(w_i^k-\frac{1}{k})(\gamma {\rm Vol}(\mathcal{S}) \xi_{\rho_{k_2}} + 2 \xi_{\pi_{k_2}})\\
      \geqslant &\sum\limits_{i=1}^{k_0 - 1}(w_i^k-\frac{1}{k})(\gamma {\rm Vol}(\mathcal{S}) \xi_{\rho_i} + 2 \xi_{\pi_i}),
    \end{split}
  \end{equation}
 
  where $k_2 \in [0, k_0)$
 
 Based on these two equations:
 
  \begin{equation}
    \begin{split}
      & \sum\limits_{i=1}^{k}(w_i^k-\frac{1}{k})(\gamma {\rm Vol}(\mathcal{S}) \xi_{\rho_i} + 2 \xi_{\pi_i})\\
      = & \sum\limits_{i=1}^{k_0 - 1}(w_i^k-\frac{1}{k})(\gamma {\rm Vol}(\mathcal{S}) \xi_{\rho_i} + 2 \xi_{\pi_i})+ \sum\limits_{i=k_0}^{k}(w_i^k-\frac{1}{k})(\gamma {\rm Vol}(\mathcal{S}) \xi_{\rho_i} + 2 \xi_{\pi_i})\\
      \leqslant & \sum\limits_{i=1}^{k_0 - 1}(w_i^k-\frac{1}{k})(\gamma {\rm Vol}(\mathcal{S}) \xi_{\rho_{k_2}} + 2 \xi_{\pi_{k_2}})+ \sum\limits_{i=k_0}^{k}(w_i^k-\frac{1}{k})(\gamma {\rm Vol}(\mathcal{S}) \xi_{\rho_{k_1}} + 2 \xi_{\pi_{k_1}})\\
      = & \sum\limits_{i=1}^{k_0 - 1}(w_i^k-\frac{1}{k})(\gamma {\rm Vol}(\mathcal{S}) \xi_{\rho_{k_2}} + 2 \xi_{\pi_{k_2}}) - \sum\limits_{i=1}^{k_0 - 1}(w_i^k-\frac{1}{k})(\gamma {\rm Vol}(\mathcal{S}) \xi_{\rho_{k_1}} + 2 \xi_{\pi_{k_1}})\\
      \quad &+ \sum\limits_{i=1}^{k_0 - 1}(w_i^k-\frac{1}{k})(\gamma {\rm Vol}(\mathcal{S}) \xi_{\rho_{k_1}} + 2 \xi_{\pi_{k_1}}) + \sum\limits_{i=k_0}^{k}(w_i^k-\frac{1}{k})(\gamma {\rm Vol}(\mathcal{S}) \xi_{\rho_{k_1}} + 2 \xi_{\pi_{k_1}})\\ 
      = & \sum\limits_{i=1}^{k_0 - 1}(w_i^k-\frac{1}{k})[(\gamma {\rm Vol}(\mathcal{S}) \xi_{\rho_{k_2}} + 2 \xi_{\pi_{k_2}}) - (\gamma {\rm Vol}(\mathcal{S}) \xi_{\rho_{k_1}} + 2 \xi_{\pi_{k_1}})] + \sum\limits_{i=1}^{k}(w_i^k-\frac{1}{k})(\gamma {\rm Vol}(\mathcal{S}) \xi_{\rho_{k_1}} + 2 \xi_{\pi_{k_1}})
    \end{split}
  \end{equation}
 
 Since distribution shift sequence $\{(\xi_{\rho_1}, \xi_{\pi_1}),(\xi_{\rho_2}, \xi_{\pi_2}), ......, (\xi_{\rho_k}, \xi_{\pi_k})\}$ decreases as $i$ increases, and $k_2 < k_1$, the first term will be less than 0. 
 Meanwhile, the second term will be equal to 0 because $\sum\limits_{i=1}^k w_i^k=\sum\limits_{i=1}^k \frac{1}{k}=1$.
 Therefore, we can get:
 
   \begin{equation}
    \begin{split}
      & \sum\limits_{i=1}^{k}(w_i^k-\frac{1}{k})(\gamma {\rm Vol}(\mathcal{S}) \xi_{\rho_i} + 2 \xi_{\pi_i})\\
      &\leqslant  \sum\limits_{i=1}^{k_0 - 1}(w_i^k-\frac{1}{k})[(\gamma {\rm Vol}(\mathcal{S}) \xi_{\rho_{k_2}} + 2 \xi_{\pi_{k_2}}) - (\gamma {\rm Vol}(\mathcal{S}) \xi_{\rho_{k_1}} + 2 \xi_{\pi_{k_1}})] + \sum\limits_{i=1}^{k}(w_i^k-\frac{1}{k})(\gamma {\rm Vol}(\mathcal{S}) \xi_{\rho_{k_1}} + 2 \xi_{\pi_{k_1}}) \\
     &\leqslant 0
    \end{split}
  \end{equation}
 
 The proof is finished.
\end{proof}

Proposition~\ref{prop_3} illustrate that after adjusting the policy mixture distribution according to the distribution shifts, the performance bound will be tighter than learning a global dynamics model ($w_i^k = \frac{1}{k}$).
This provides a guidance for our proposed method, that the weight $w_i^k$ of each policy $\pi_i$ in the historical policy sequence $\Pi^k$ should be negatively related to its state action visitation distribution shift $\xi_{\rho_i}$ and the policy distribution shift $\xi_{\pi_i}$.
% However, it is difficult to estimation those distribution shifts using total-variance distance and translate these distances into the weights of each policy, because we can only get very few samples of each policy from the environment.
% Therefore, based on our empirical results in Figure\ref{dis_c} and \ref{dis_d} and the overall trend of policy optimization is constantly getting better,
% we can assume that the state action visitation distribution shift between the historical policy and the current policy increases as the age of the historical policy increases.
% This motivates our method that using an exponential weighting of $\rho^{\pi_i}$ in Section~\ref{section:algorithm} instead of adjusting weights according to the actual state action visitation distribution shift.
\section{More experiments}\label{app:exp}

\subsection{More Error Curves for Dynamics Model Learned by MBPO}\label{app: more_local}

In this section, we provide the local error curves for global dynamics model in four MoJoCo environments: Hopper, HalfCheetah, Walker2d, and Humanoid.
The curves are shown in Figure~\ref{fig: more_local}.

\begin{figure}[!htbp]
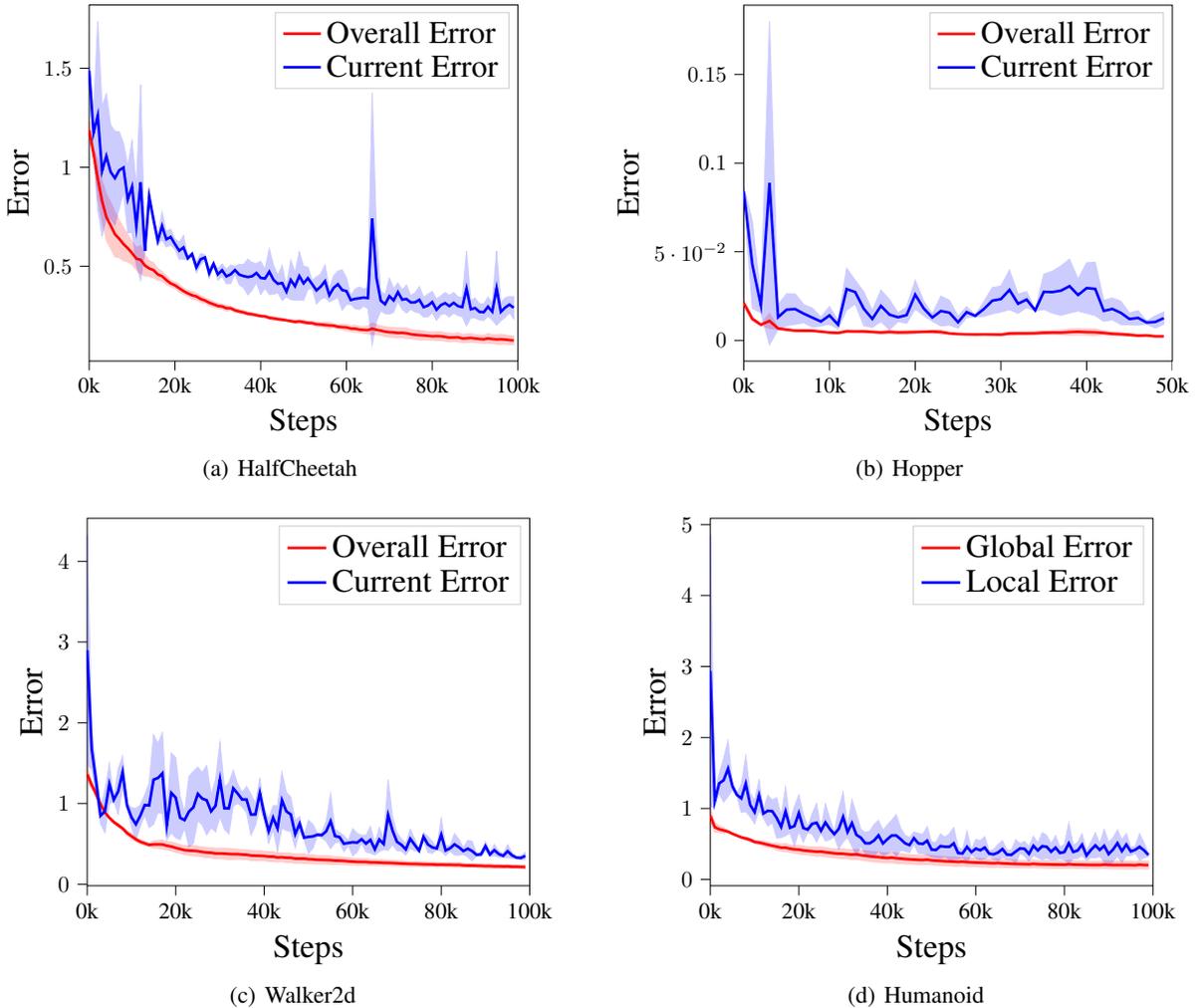

\centering
\subfigure[HalfCheetah]{
\resizebox{0.43\textwidth}{!}{\input{exp_results/PMAC_half_local_error}}
\label{local_a}
}
\hfil
\subfigure[Hopper]{
\resizebox{0.46\textwidth}{!}{\input{exp_results/PMAC_hopper_local_error}}
\label{local_b}
}
\hfil
\subfigure[Walker2d]{
\resizebox{0.43\textwidth}{!}{\input{exp_results/PMAC_walker_local_error}}
\label{local_c}
}
\hfil
\subfigure[Humanoid]{
% \resizebox{0.223\textwidth}{!}{\input{exp_results/PMAC_hopper_distribution}}
\resizebox{0.43\textwidth}{!}{\input{exp_results/PMAC_human_local_error}}
\label{local_e}
}
% \vspace{-10pt}
\caption{The global error curve and the local error curve of MBPO in four MuJoCo environments.}
\label{fig: more_local}
\end{figure}

\subsection{Experiment about the Convergence of the Dynamics Model on Recent Data}\label{app: converge_model}
To prove that the error gap in Figure~\ref{Fig_delta_distribution} is not caused by the dynamics model not having converged on the recent data,
we checkpoint the real sample buffer and dynamics model at multiple points during training, then train the dynamics model for a long time until convergence at these checkpointed locations.
We conduct the experiment on Walker2d-v2 and HalfCheetah-v2, and checkpoint the data and model at environment step 20k, 40k, 60k, 80k, and 100k.
The one-step model prediction error on newly generated samples during rollout of this converged model, MBPO and our method (PDML) are shown in Table~\ref{tab: error walker} and Table~\ref{tab: error half}. 
Each method runs 8 random seeds.

\begin{table}[!htb]
\setlength{\abovecaptionskip}{0.2cm}
\centering
\caption{Prediction error of the converged model, MBPO and PDML in Walker2d.}
\begin{tabular}{c|ccccc}
\toprule 
 &Env step 20k&Env step 40k&Env step 60k&Env step 80k&Env step 100k \\
\midrule
Converged Model & 0.946 $\pm $0.201 &	0.578  $\pm$  0.029  & 0.328  $\pm$ 0.037 & 0.317  $\pm$ 0.013 & 0.272  $\pm$ 0.024\\
MBPO & 1.022  $\pm$ 0.496 & 0.819  $\pm$ 0.076 & 0.471  $\pm$ 0.167 & 0.563  $\pm$ 0.204 & 0.302  $\pm$ 0.043\\
PDML & \textbf{0.673 $\pm$ 0.117} & \textbf{0.422 $\pm$ 0.169} & \textbf{0.207 $\pm$  0.072} & \textbf{0.128 $\pm$ 0.019} & \textbf{0.134 $\pm$ 0.034}  \\

\bottomrule
\end{tabular}
\label{tab: error walker}
% \vspace{-10pt}
\end{table}

\begin{table}[!htb]
\setlength{\abovecaptionskip}{0.2cm}
\centering
\caption{Prediction error of the converged model, MBPO and PDML in HalfCheetah.}
\begin{tabular}{c|ccccc}
\toprule 
 &Env step 20k&Env step 40k&Env step 60k&Env step 80k&Env step 100k \\
\midrule
Converged Model & 0.478 $\pm$ 0.027 & 0.345 $\pm$ 0.088 & 0.268 $\pm$ 0.053 & 0.245 $\pm$  0.051 & 0.219 $\pm$ 0.087  \\
MBPO & 0.561 $\pm$ 0.026 &0.391 $\pm$ 0.083 &0.324 $\pm$ 0.066 & 0.269 $\pm$ 0.053 & 0.241 $\pm$  0.055\\
PDML & \textbf{0.363 $\pm$ 0.066} & \textbf{0.232 $\pm$ 0.053} & \textbf{0.206 $\pm$ 0.163} & \textbf{0.149 $\pm$ 0.038} & \textbf{0.127 $\pm$ 0.014}\\

\bottomrule
\end{tabular}
\label{tab: error half}
% \vspace{-10pt}
\end{table}

From these results, we can find that training a dynamics model to convergence on the current real samples can actually reduce the model prediction error during rollouts, but the effect is not obvious. 
Especially when the number of samples in the real sample buffer becomes very large (Env step 60k ~ 100k), the model prediction error obtained by training a converged model is almost the same as MBPO. 
This experimental result further proves our claim that the main reason for the model prediction error during model rollouts is not that the model does not converge on new samples. 
The main reason is the mismatch of model learning and model rollouts.

\subsection{Visualization of State-Action Visitation Distribution of Different Historical Policies}\label{app: more distribution shift}

Due to the limited space of main paper, we provide detailed visualization of the state-action visitation distribution of policies under different environment steps in this section.
We conduct the experiment on HalfCheetah and Hopper, the results are shown in Figure~\ref{halfcheetah distribution shift all} and Figure~\ref{hopper distribution shift all}.
(a) in each figure is the comparison of different policies in the same figure, from (b) to (f) are the figures presenting the state-action visitation distribution of each policy individually.
We can see that the state-action visitation distribution of policies under different environment steps is very different.

\begin{figure}[!htbp]
\centering
\subfigure[]{
\includegraphics[width=0.3\textwidth]{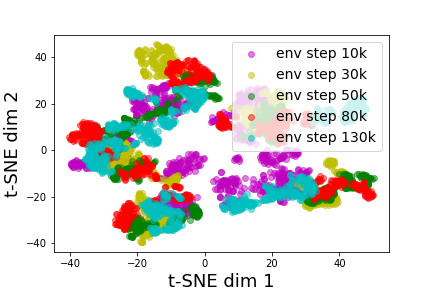}
\label{ablation:half distribition shift all}
}
\hfil
\subfigure[Environment step 10k]{
\includegraphics[width=0.3\textwidth]{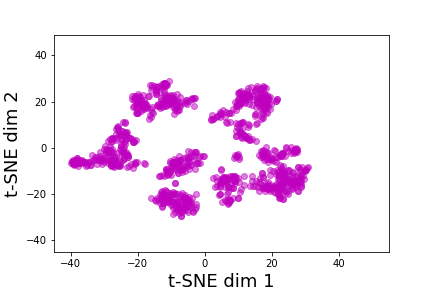}
\label{ablation:half distribition shift 10k}
}
\hfil
\subfigure[Environment step 30k]{
\includegraphics[width=0.3\textwidth]{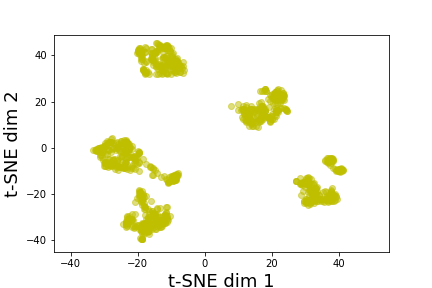}
\label{ablation:half distribition shift 30k}
}
\hfil
\subfigure[Environment step 50k]{
\includegraphics[width=0.3\textwidth]{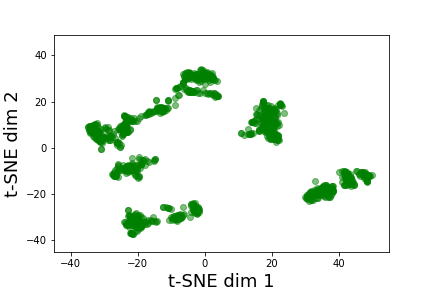}
\label{ablation:half distribition shift 50k}
}
\hfil
\subfigure[Environment step 80k]{
\includegraphics[width=0.3\textwidth]{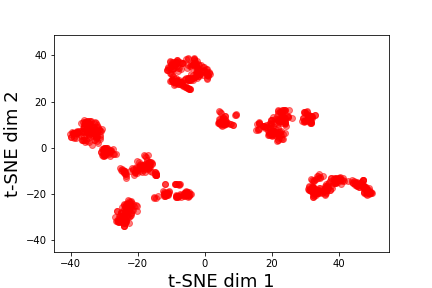}
\label{ablation:half distribition shift 80k}
}
\hfil
\subfigure[Environment step 130k]{
\includegraphics[width=0.3\textwidth]{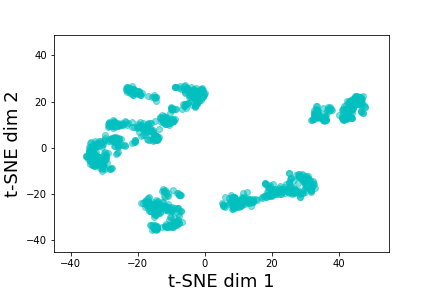}
\label{ablation:half distribition shift 130k}
}
\caption{Visualization of state-action visitation distribution of policies at different environment steps in HalfCheetah.}
\label{halfcheetah distribution shift all}
\end{figure}

\begin{figure}[!htbp]
\centering
\subfigure[]{
\includegraphics[width=0.3\textwidth]{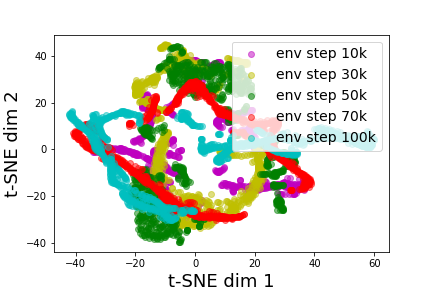}
\label{ablation:hopper distribition shift all}
}
\hfil
\subfigure[Environment step 10k]{
\includegraphics[width=0.3\textwidth]{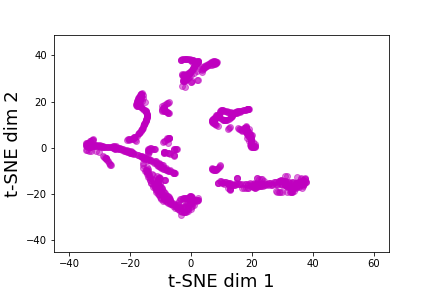}
\label{ablation:hopper distribition shift 10k}
}
\hfil
\subfigure[Environment step 30k]{
\includegraphics[width=0.3\textwidth]{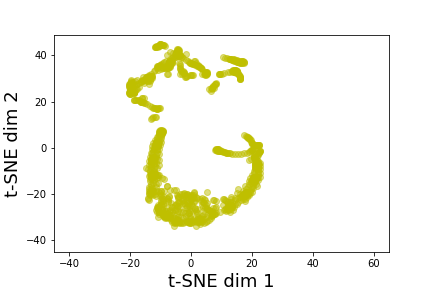}
\label{ablation:hopper distribition shift 30k}
}
\hfil
\subfigure[Environment step 50k]{
\includegraphics[width=0.3\textwidth]{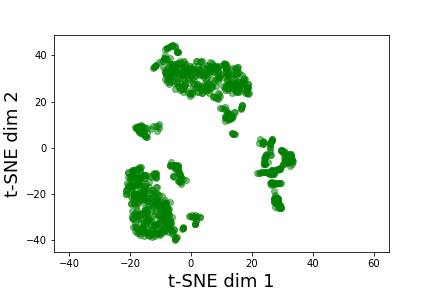}
\label{ablation:hopper distribition shift 50k}
}
\hfil
\subfigure[Environment step 70k]{
\includegraphics[width=0.3\textwidth]{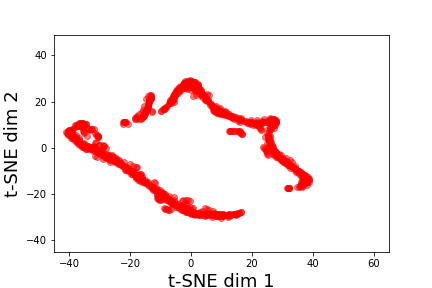}
\label{ablation:hopper distribition shift 70k}
}
\hfil
\subfigure[Environment step 100k]{
\includegraphics[width=0.3\textwidth]{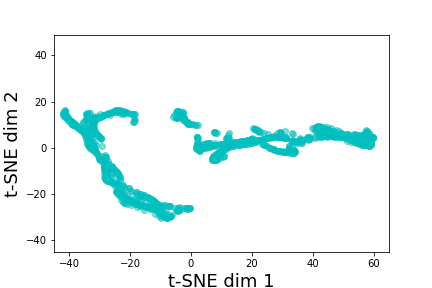}
\label{ablation:hopper distribition shift 100k}
}
\caption{Visualization of state-action visitation distribution of policies at different environment steps in Hopper.}
\label{hopper distribution shift all}
\end{figure}

\subsection{Visualization of Adjusted Policy Mixture Distribution}
\label{app: distribution visualization}

To provide a further understanding of out method, we visualize the adjusted policy mixture distribution at different training steps on Humanoid in Figure~\ref{policy mixture distribution}.
We take Figure~\ref{ablation:distribition 50k} as an example to explain the origin and meaning of the policy ID on the horizontal axis. Each of policies interacts with the environment for 250 steps, so a 50k environment step has 200 historical policies. 
The policy ID of 0 indicates the oldest policy. The larger the ID, the newer the policy.
We can see that the policy mixture distribution is totally different at different training steps.
The weight of policy is not a simple exponentially decay or linearly decay, which indicates our proposed method is non-trivial.

\begin{figure}[!htbp]
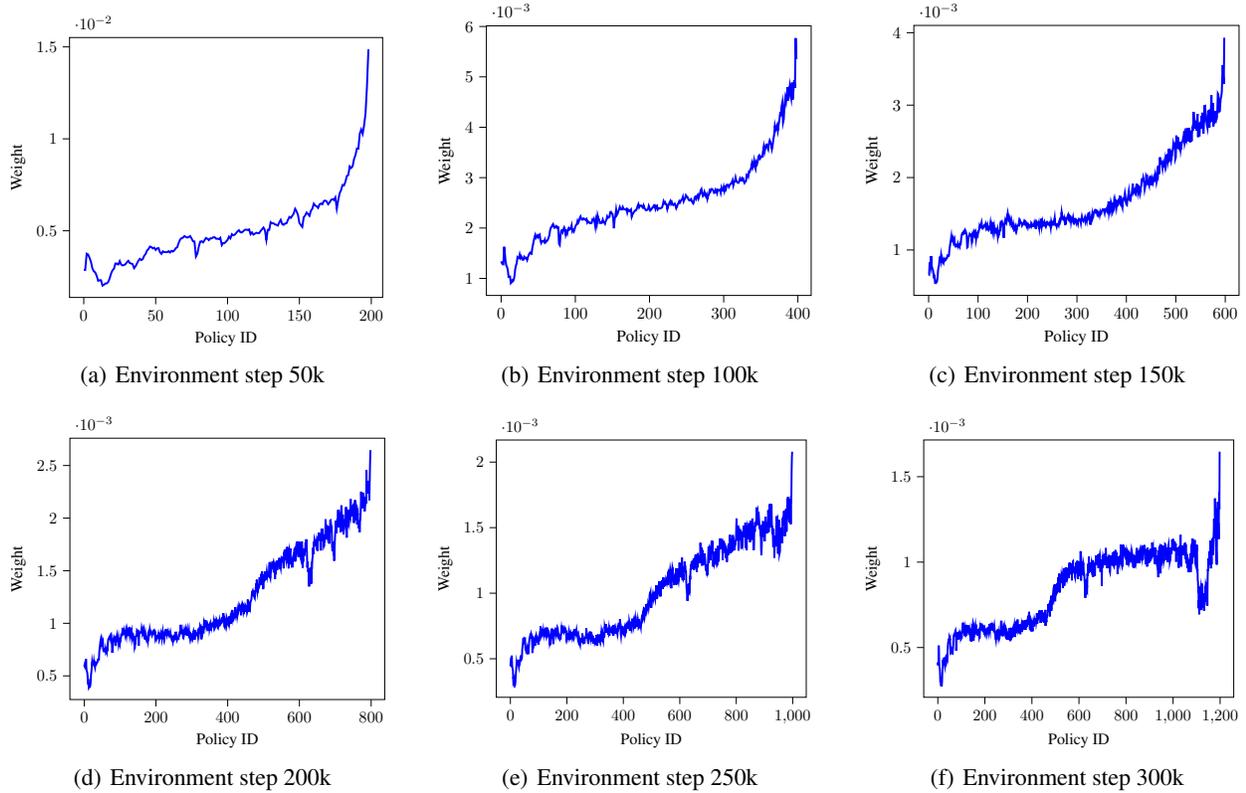

\centering
\subfigure[Environment step 50k]{
\resizebox{0.3\textwidth}{!}{% This file was created with tikzplotlib v0.10.1.
\begin{tikzpicture}

\definecolor{darkgray176}{RGB}{176,176,176}

\begin{axis}[
tick align=outside,
tick pos=left,
x grid style={darkgray176},
xlabel={Policy ID},
xmin=-9.9, xmax=207.9,
xtick style={color=black},
y grid style={darkgray176},
ylabel={Weight},
ymin=0.00137467992572508, ymax=0.0155090291610368,
ytick style={color=black}
]
\addplot [very thick, blue]
table {%
0 0.00288043873263332
1 0.00286892191530345
2 0.0037534862235529
3 0.00371092513588269
4 0.00352981375416081
5 0.00330859964102451
6 0.00294970979410906
7 0.00279872958457458
8 0.00273111772024408
9 0.00254578545386172
10 0.00227753501456324
11 0.0023119774538465
12 0.00230951980587299
13 0.00201715034551198
14 0.0020644455365379
15 0.00213227890734196
16 0.00213768915482909
17 0.00221733516361959
18 0.00241186161110301
19 0.00267811307393475
20 0.00285278835292129
21 0.00293336056283703
22 0.00322397835245979
23 0.00321340250086056
24 0.00316653595112971
25 0.00335139996632513
26 0.00316117602763492
27 0.00311477136996799
28 0.00313635503439869
29 0.00318081488891345
30 0.00329111858679749
31 0.00336938505200087
32 0.00324363932887885
33 0.00321920244689555
34 0.00318308225712549
35 0.00296653771285727
36 0.00306559161035743
37 0.00322088199810704
38 0.00337386766884719
39 0.00347495963337015
40 0.00340546840144097
41 0.00345408217117457
42 0.00363824297609021
43 0.00373586781630654
44 0.00392266486857843
45 0.00400432554646908
46 0.00414417640307775
47 0.00411584533929056
48 0.00404187064565292
49 0.00408236921708411
50 0.00395694033043281
51 0.00404274377911302
52 0.00405707467479944
53 0.00383322751718794
54 0.00374375031024588
55 0.00388407283037457
56 0.00384407961327306
57 0.00383581793919518
58 0.00383997652989177
59 0.00389167211420433
60 0.0038193360458437
61 0.00390895694208449
62 0.00386662514306351
63 0.00393247935322093
64 0.00405602573711647
65 0.00410027966472034
66 0.00433920621008032
67 0.00451835451282406
68 0.0046000284387862
69 0.00462220750422763
70 0.00471417609401019
71 0.00466305851189713
72 0.00463013885185657
73 0.00466177417072111
74 0.0047144892041873
75 0.00457715049105667
76 0.00443033816763881
77 0.00441649731113314
78 0.00360057171537218
79 0.00375149208636855
80 0.00413869444056379
81 0.00440986954222828
82 0.0044320486517645
83 0.00436798336581049
84 0.00448776592692768
85 0.00460883754665531
86 0.00456604482615173
87 0.00448898366561522
88 0.00458637818073277
89 0.00458656809621952
90 0.00464319600942846
91 0.0046538901372109
92 0.00456601330921084
93 0.00456878715174809
94 0.00454580375977697
95 0.0046193644043518
96 0.00421818595644537
97 0.00433414405002026
98 0.00433285631049581
99 0.00441788565877716
100 0.00449902138966311
101 0.00468104257173724
102 0.00458630663525603
103 0.00481571668235359
104 0.00475138733697962
105 0.00478560041369072
106 0.00487792752740459
107 0.00497754562930684
108 0.00503019030812365
109 0.0049127452127447
110 0.00491142298488472
111 0.00498156945318745
112 0.0050481083260727
113 0.00506471066822325
114 0.00504037112420086
115 0.00507621586101652
116 0.00481244049196175
117 0.00491397790659162
118 0.00479105028402552
119 0.00484088831581658
120 0.00489303148713263
121 0.00487499226486778
122 0.00487139279668284
123 0.00498881401838833
124 0.00507972247678665
125 0.00517665539638577
126 0.00502977755684197
127 0.0045000522066578
128 0.00504894688325087
129 0.00504770011858526
130 0.00546795512441429
131 0.00540389777528709
132 0.0054088373230476
133 0.0053400710875056
134 0.00532732893783669
135 0.00528168509718166
136 0.00545585890817361
137 0.00559545218071824
138 0.00536504888389889
139 0.00532081865569958
140 0.00539221792018929
141 0.00540542979377748
142 0.0055851715564749
143 0.00575438394710547
144 0.00569981079205298
145 0.00581243536871262
146 0.00593014529578028
147 0.00619276553631691
148 0.00605700956192264
149 0.0058634125550774
150 0.00541953032397448
151 0.00531427089950141
152 0.00521174682775996
153 0.00573562423953183
154 0.00583689102757661
155 0.00599338629411064
156 0.00590685089111225
157 0.0058092950733633
158 0.00609160459307418
159 0.00620068068598947
160 0.00647592873233897
161 0.00637501191968818
162 0.0063313591155775
163 0.00624196410535112
164 0.00643528448402204
165 0.006481675246013
166 0.00644242534444648
167 0.00661938714772661
168 0.00664294952730729
169 0.00659218024977921
170 0.00644215434080537
171 0.00653684357460868
172 0.00667543819700545
173 0.00672387359325433
174 0.00667054481612241
175 0.00681242341351232
176 0.00619396068297019
177 0.00667079753666911
178 0.00703887110767572
179 0.00721505646405499
180 0.00748965236217973
181 0.00746437544270482
182 0.00766503802043608
183 0.00797858741746287
184 0.00803311873621059
185 0.00850431940641471
186 0.00841327846290366
187 0.00850927410331877
188 0.00889834560484283
189 0.00911554690135104
190 0.00947688744638328
191 0.00946041745106854
192 0.0102827567434676
193 0.010506743649602
194 0.010274717746438
195 0.0107371151193516
196 0.0113973226514144
197 0.0128698334128608
198 0.0148665587412499
};
\end{axis}

\end{tikzpicture}}
\label{ablation:distribition 50k}
}
\hfil
\subfigure[Environment step 100k]{
\resizebox{0.3\textwidth}{!}{% This file was created with tikzplotlib v0.10.1.
\begin{tikzpicture}

\definecolor{darkgray176}{RGB}{176,176,176}

\begin{axis}[
tick align=outside,
tick pos=left,
x grid style={darkgray176},
xlabel={Policy ID},
xmin=-19.9, xmax=417.9,
xtick style={color=black},
y grid style={darkgray176},
ylabel={Weight},
ymin=0.00066230758779945, ymax=0.00601280130193871,
ytick style={color=black}
]
\addplot [very thick, blue]
table {%
0 0.00133876140693082
1 0.00130477718260944
2 0.00128061252650863
3 0.00127884319742088
4 0.00162791261913056
5 0.00149612310672128
6 0.00133110211891538
7 0.0012509878380661
8 0.00122533815900211
9 0.00116368821652896
10 0.00102978988149059
11 0.00103151135380272
12 0.00103620121777447
13 0.000905511847533053
14 0.000921771734452139
15 0.000961314673442469
16 0.000947185458468048
17 0.000974158252814018
18 0.00105029556217558
19 0.00118937225824349
20 0.00125868894401634
21 0.00130176463654121
22 0.0014190550749426
23 0.00140252769316126
24 0.00140311281961199
25 0.001472421020916
26 0.00140678317906784
27 0.00143770821287318
28 0.0013976380091419
29 0.00136921945257214
30 0.00137136649862904
31 0.0014247832216341
32 0.00140118631391709
33 0.00141315656771486
34 0.00140780662970353
35 0.00136730078208397
36 0.00142167157020635
37 0.00145620875528342
38 0.00149257867012911
39 0.00151140843852225
40 0.00149384285765838
41 0.00148755259035773
42 0.00160388961570663
43 0.00167373761671098
44 0.00175168728208954
45 0.00181066637672466
46 0.00186199526366382
47 0.00180616951109287
48 0.00182869920583667
49 0.00185402984951332
50 0.00181877134191061
51 0.00185410050642148
52 0.00186936846242694
53 0.00175376324962948
54 0.00170519518101133
55 0.00173927007429881
56 0.00174572985519721
57 0.00174035532189683
58 0.00171936469759908
59 0.00176861830061266
60 0.00171759024762141
61 0.00173737981027459
62 0.00174646140239048
63 0.00177127876924329
64 0.00180949938123825
65 0.00185756898059327
66 0.00195657946252102
67 0.00200667850156085
68 0.00202407427667787
69 0.00205309293071051
70 0.00207660171697313
71 0.00202889659927893
72 0.00200347948395811
73 0.00201156838412616
74 0.00204285555374426
75 0.00203532625569769
76 0.00199863180047812
77 0.00200832140243395
78 0.00168767866417555
79 0.00166338825712669
80 0.00184361412630552
81 0.00195030063362003
82 0.00195394558171176
83 0.00191739533246685
84 0.00197704900898872
85 0.00201126232940549
86 0.00196735800833624
87 0.00186688974778564
88 0.00192344933707368
89 0.0018877377630606
90 0.00197034660757898
91 0.00199886014627132
92 0.00197040017181587
93 0.00198788019420279
94 0.00198815654299856
95 0.00204926080695193
96 0.00193487728285997
97 0.00195845464076223
98 0.00193168419579836
99 0.00197948242010144
100 0.00202664687226202
101 0.00206134994652405
102 0.00207263383860535
103 0.00214120880640292
104 0.00210695262592878
105 0.0021412970230029
106 0.0021964151335233
107 0.00222914394809678
108 0.00218673202278497
109 0.00216548301313838
110 0.00216729032838093
111 0.00219197148141403
112 0.00218329827823712
113 0.00216560669363103
114 0.00216353504001293
115 0.00218021854286319
116 0.0021094135160382
117 0.00213361295713848
118 0.00211042527415509
119 0.00211245903434809
120 0.00209449641061348
121 0.00209672978578126
122 0.00210053790915128
123 0.00216811150940412
124 0.002157856737169
125 0.00218566373068592
126 0.00211749802490396
127 0.00202514437090011
128 0.00216048229313117
129 0.00208373055937004
130 0.00224868613027176
131 0.00218804217766586
132 0.00217191171417043
133 0.00210185686670815
134 0.00210669100301684
135 0.00209880192313366
136 0.00210622374289085
137 0.00216199651165091
138 0.00211465787125988
139 0.00210705283471399
140 0.00208421989606338
141 0.00214183164196624
142 0.00216519477358807
143 0.00223138644173297
144 0.00220315185164686
145 0.00221174873100487
146 0.00221945529458041
147 0.00231409607030164
148 0.00232090770819266
149 0.00229186726873116
150 0.00228345136978695
151 0.00219457860095896
152 0.00199461542499869
153 0.00217281106190554
154 0.00224301701113571
155 0.00236279223432862
156 0.00225975605766026
157 0.00228504810108532
158 0.00233534358971279
159 0.00237207015406034
160 0.00241722324412908
161 0.00237775977958981
162 0.00236869012487206
163 0.0023688818210917
164 0.00241722543771389
165 0.00238606928778517
166 0.00237008853402354
167 0.00236913570834206
168 0.00236333755743028
169 0.00237872752219812
170 0.00230817403951139
171 0.00236699517496662
172 0.00235340829697981
173 0.00233542242079194
174 0.00234831960520287
175 0.00236759846945994
176 0.00222563579653373
177 0.00226875899507666
178 0.00228772550401519
179 0.00232413693585127
180 0.00233186701383599
181 0.0023536136447868
182 0.00235478637319841
183 0.00239021609928403
184 0.00243775146326873
185 0.00247342742221331
186 0.00241210168323247
187 0.00244329888798685
188 0.00244617151207186
189 0.00244586799988241
190 0.00242436823592462
191 0.0024053607207577
192 0.00245007023037144
193 0.00240284615333104
194 0.00238046352095842
195 0.0023543542341832
196 0.00235370905273135
197 0.00236135086954912
198 0.00237421062184699
199 0.00243864391056853
200 0.00239734637771801
201 0.00234659733800078
202 0.00236030178318059
203 0.00238186814932273
204 0.00240863772784338
205 0.002415539745833
206 0.00245248283307359
207 0.00245508375074435
208 0.00238324938087648
209 0.00243271953502846
210 0.0024444763465775
211 0.00236644474642795
212 0.00238521623433534
213 0.00239584666714632
214 0.00238682429739192
215 0.00239432811639972
216 0.00239894005583592
217 0.00241158711799157
218 0.00246448093640561
219 0.00253028085165742
220 0.00249329291093878
221 0.00251196040755889
222 0.002469848201298
223 0.00246638522701596
224 0.00237642306417212
225 0.00242893836632618
226 0.00246621224756701
227 0.00250486670045068
228 0.00253468554896555
229 0.00248083973640777
230 0.00251544187781707
231 0.00247347134885594
232 0.00245502152432358
233 0.00246818639608398
234 0.00248369772644476
235 0.00244753248669716
236 0.0023919656915012
237 0.00242504749806246
238 0.00246889444235711
239 0.00249292509703412
240 0.00253162907479339
241 0.00250489084492763
242 0.00253282970560933
243 0.00256107775972393
244 0.00257061369782368
245 0.00261755845229137
246 0.0025296707907689
247 0.0025627884521125
248 0.002589912957862
249 0.00264869831725671
250 0.00262402091055341
251 0.00261926738095188
252 0.00258487582143738
253 0.00255519038984149
254 0.00254660923769377
255 0.00255807725362658
256 0.00257406712717974
257 0.00253682826783572
258 0.00250185582668091
259 0.00248631889328757
260 0.00253805892675594
261 0.00261103016525982
262 0.00262571984882458
263 0.00258167121931427
264 0.00262656196857965
265 0.00265495471153624
266 0.00258973197462291
267 0.00264070893776011
268 0.00268485044533155
269 0.0027474973018914
270 0.00277344721150448
271 0.00267941573827534
272 0.00263122441797283
273 0.00264835382595536
274 0.0027167779791544
275 0.00267555714196347
276 0.00263075632403239
277 0.00266570688837065
278 0.00257704902836194
279 0.00263509035877576
280 0.00266969095093349
281 0.00274377055825804
282 0.00272310681902578
283 0.00277440338789373
284 0.00277688916946313
285 0.00276269284274742
286 0.00273722463642321
287 0.00269670106964015
288 0.00270147806078084
289 0.00276702207666943
290 0.0027518619829966
291 0.00269567492475991
292 0.00276273188441549
293 0.00275544470943276
294 0.00274713424762254
295 0.00275784111732261
296 0.00269935037729284
297 0.00279640362834738
298 0.00274459077958346
299 0.00274094612294967
300 0.00281285588094749
301 0.0027353385371141
302 0.00280154734579878
303 0.00277122163746618
304 0.00285449724353181
305 0.00286084422372998
306 0.00285534828212841
307 0.00281694053939068
308 0.00285258931997417
309 0.00287244373414597
310 0.002881782635834
311 0.00291252558519278
312 0.00295076485894855
313 0.00284923801138709
314 0.00291159000046927
315 0.00292207648033812
316 0.00300361895206966
317 0.00294037853641426
318 0.00291338489462354
319 0.00289759649129881
320 0.0028967234330906
321 0.00293204681322825
322 0.00298014570650669
323 0.00294134428779937
324 0.00294116484831079
325 0.00299746098357434
326 0.0029888414617882
327 0.00292809587036424
328 0.00289414278820033
329 0.00292367452319001
330 0.00292965398432213
331 0.00305827307407301
332 0.00304758015261267
333 0.00301944577483115
334 0.00311942951450607
335 0.00311323099620527
336 0.00316234259270714
337 0.00319387510508419
338 0.00327979411337887
339 0.00319698395496586
340 0.00324433619008121
341 0.00333585333102788
342 0.00322240542442184
343 0.00317387879589195
344 0.00330830460448062
345 0.0032749591613959
346 0.00333384064679208
347 0.00342604108515578
348 0.00344318419046729
349 0.00337660566429795
350 0.00339625765686533
351 0.00342206762225445
352 0.00341710579044714
353 0.0035196398131045
354 0.00361457583025153
355 0.00350465311924814
356 0.00357002129746882
357 0.00365036221695969
358 0.0036572709862823
359 0.00359873935764227
360 0.00355995685209156
361 0.00367036990780969
362 0.00373722029443255
363 0.00370128445716025
364 0.0036599340132235
365 0.00356901498276475
366 0.00364360845863127
367 0.00376319809791491
368 0.00385253871645578
369 0.00402695982789144
370 0.00394991295893766
371 0.00399369690083824
372 0.0040379282475793
373 0.0039355280885405
374 0.00404521951550711
375 0.00414306221301459
376 0.00426537796272698
377 0.00425020752066398
378 0.00441125212353756
379 0.00449139942550803
380 0.00409539528490473
381 0.0041626795320731
382 0.00445884901785292
383 0.00458173829858511
384 0.00449921394740807
385 0.00464458351310284
386 0.00479342816357644
387 0.0045243310735875
388 0.00468418011213367
389 0.00480713362868255
390 0.00465103901983663
391 0.00456991439632394
392 0.00488537944668707
393 0.00454864831824048
394 0.00471819748954136
395 0.00492877205412884
396 0.00478238748319366
397 0.00576959704220511
398 0.00535131954984926
};
\end{axis}

\end{tikzpicture}}
\label{ablation:distribition 100k}
}
\hfil
\subfigure[Environment step 150k]{
\resizebox{0.3\textwidth}{!}{\input{exp_results/PDML_weight_150k_humanoid.tex}}
\label{ablation:distribition 150k}
}
\hfil
\subfigure[Environment step 200k]{
\resizebox{0.3\textwidth}{!}{\input{exp_results/PDML_weight_200k_humanoid.tex}}
\label{ablation:distribition 200k}
}
\hfil
\subfigure[Environment step 250k]{
\resizebox{0.3\textwidth}{!}{\input{exp_results/PDML_weight_250k_humanoid.tex}}
\label{ablation:distribition 250k}
}
\hfil
\subfigure[Environment step 300k]{
\resizebox{0.3\textwidth}{!}{\input{exp_results/PDML_weight_300k_humanoid.tex}}
\label{ablation:distribition 300k}
}
\caption{Visualization of adjusted policy mixture distribution at different training steps on Humanoid.}
\label{policy mixture distribution}
\end{figure}

\subsection{Ablation Study of PDML}

% In this section, we provide the ablation study of adjusted policy mixture distribution in PDML.
As we described in Section~\ref{sec4}, we use the adjusted policy mixture distribution for both model learning and sampling initial states for model rollouts.
In this section, we provide the ablation study to show the impact of the adjusted policy mixture distribution in these two parts respectively.
We conducted our experiments in Hopper and Walker2d, and the performance curves are shown in Figures \ref{ablation:learning and rollout hopper} and \ref{ablation:learning and rollout walker}.

\begin{figure}[!htbp]
\centering
\subfigure[Hopper]{
\resizebox{0.38\textwidth}{!}{% This file was created with tikzplotlib v0.10.1.
\begin{tikzpicture}

\definecolor{darkgray176}{RGB}{176,176,176}
\definecolor{green}{RGB}{0,128,0}
\definecolor{lightgray204}{RGB}{204,204,204}
\definecolor{purple}{RGB}{128,0,128}

\begin{axis}[
legend cell align={left},
legend style={
  fill opacity=0.8,
  draw opacity=1,
  text opacity=1,
  at={(0.97,0.03)},
  anchor=south east,
  draw=lightgray204
},
tick align=outside,
tick pos=left,
x grid style={darkgray176},
xlabel={Steps},
xmajorgrids,
xmin=-0.55, xmax=11.55,
xtick style={color=black},
xtick={0,2.75,5.5,8.25,11},
xticklabels={0k,30k,60k,90k,120k},
y grid style={darkgray176},
ylabel={Returns},
ymajorgrids,
ymin=126.525490255497, ymax=3841.1904784217,
ytick style={color=black}
]
\path [draw=red, fill=red, opacity=0.2]
(axis cs:0,521.018486614765)
--(axis cs:0,421.691404266603)
--(axis cs:1,1361.44645814367)
--(axis cs:2,2636.06766687043)
--(axis cs:3,2997.33788857376)
--(axis cs:4,3007.56904162707)
--(axis cs:5,3169.32658525878)
--(axis cs:6,3274.3902374211)
--(axis cs:7,3458.50841599406)
--(axis cs:8,3433.84049735409)
--(axis cs:9,3457.93389782643)
--(axis cs:10,3588.50200234617)
--(axis cs:11,3610.85435072886)
--(axis cs:11,3672.34206986869)
--(axis cs:11,3672.34206986869)
--(axis cs:10,3660.12388483112)
--(axis cs:9,3519.22263339582)
--(axis cs:8,3487.16421335081)
--(axis cs:7,3487.71578460847)
--(axis cs:6,3377.13543831177)
--(axis cs:5,3329.82133616751)
--(axis cs:4,3116.81768337101)
--(axis cs:3,3126.58260533071)
--(axis cs:2,2900.43036324941)
--(axis cs:1,1675.09967919274)
--(axis cs:0,521.018486614765)
--cycle;

\path [draw=purple, fill=purple, opacity=0.2]
(axis cs:0,338.064918576741)
--(axis cs:0,297.077994653278)
--(axis cs:1,402.911192591146)
--(axis cs:2,585.842374867145)
--(axis cs:3,917.801952392439)
--(axis cs:4,1468.16772297057)
--(axis cs:5,2166.71352628612)
--(axis cs:6,2722.36201512216)
--(axis cs:7,2834.35094042148)
--(axis cs:8,2911.62421455698)
--(axis cs:9,3354.03303179512)
--(axis cs:10,2840.07455772443)
--(axis cs:11,2982.25689670709)
--(axis cs:11,3268.94815483551)
--(axis cs:11,3268.94815483551)
--(axis cs:10,3215.23288254169)
--(axis cs:9,3437.27830785428)
--(axis cs:8,3308.06571918307)
--(axis cs:7,3051.0199247979)
--(axis cs:6,2899.04419938988)
--(axis cs:5,2538.42842380477)
--(axis cs:4,1757.01283590296)
--(axis cs:3,1637.79417735214)
--(axis cs:2,930.797934762353)
--(axis cs:1,460.981547393268)
--(axis cs:0,338.064918576741)
--cycle;

\path [draw=blue, fill=blue, opacity=0.2]
(axis cs:0,368.950594247811)
--(axis cs:0,318.848890905284)
--(axis cs:1,583.301345164708)
--(axis cs:2,1678.07373475525)
--(axis cs:3,2430.03197604138)
--(axis cs:4,2808.76459995931)
--(axis cs:5,3185.0460519981)
--(axis cs:6,3234.68329644659)
--(axis cs:7,3320.24403972489)
--(axis cs:8,3148.3182012721)
--(axis cs:9,3407.88335845776)
--(axis cs:10,3419.35452427472)
--(axis cs:11,3450.85724288969)
--(axis cs:11,3479.92883192555)
--(axis cs:11,3479.92883192555)
--(axis cs:10,3465.22612050728)
--(axis cs:9,3456.59613776927)
--(axis cs:8,3366.74313105187)
--(axis cs:7,3408.20223988503)
--(axis cs:6,3358.73150091917)
--(axis cs:5,3352.53664887245)
--(axis cs:4,2977.06899602295)
--(axis cs:3,2752.14953107079)
--(axis cs:2,2160.97440955589)
--(axis cs:1,774.101536696708)
--(axis cs:0,368.950594247811)
--cycle;

\path [draw=green, fill=green, opacity=0.2]
(axis cs:0,311.798346874858)
--(axis cs:0,295.373898808506)
--(axis cs:1,616.843996785332)
--(axis cs:2,1205.94988503364)
--(axis cs:3,2017.44668766832)
--(axis cs:4,2325.91469741269)
--(axis cs:5,2887.02611934338)
--(axis cs:6,3007.59081670307)
--(axis cs:7,2764.05934975739)
--(axis cs:8,3314.95377481645)
--(axis cs:9,3395.59242144556)
--(axis cs:10,3217.7923000614)
--(axis cs:11,3284.26510146315)
--(axis cs:11,3474.78432522464)
--(axis cs:11,3474.78432522464)
--(axis cs:10,3449.62592500413)
--(axis cs:9,3529.38968630605)
--(axis cs:8,3508.43849110435)
--(axis cs:7,3020.26162219881)
--(axis cs:6,3190.91964184028)
--(axis cs:5,2990.60597523099)
--(axis cs:4,2692.33628158984)
--(axis cs:3,2367.36867323046)
--(axis cs:2,1299.45685738669)
--(axis cs:1,726.89975821062)
--(axis cs:0,311.798346874858)
--cycle;

\addplot [very thick, red]
table {%
0 472.976519654385
1 1533.51307834557
2 2759.1852450854
3 3066.58287143861
4 3063.01969900472
5 3253.4610920782
6 3325.89292161874
7 3472.72841045454
8 3462.30206133435
9 3491.34622144468
10 3622.10235058772
11 3640.77412058743
};
\addlegendentry{PMDL-MBPO}
\addplot [very thick, purple]
table {%
0 317.611799328093
1 429.310981045254
2 739.226557165846
3 1246.68163108696
4 1605.78219580592
5 2350.14659028516
6 2805.59108130551
7 2945.26037064353
8 3128.08950336613
9 3396.90277269079
10 3024.52630072487
11 3124.87190288444
};
\addlegendentry{MBPO}
\addplot [very thick, blue]
table {%
0 341.846746032218
1 677.918209409273
2 1940.52758086628
3 2595.85490603975
4 2890.96969186837
5 3267.24288270005
6 3297.44675335104
7 3364.91307649144
8 3261.60509402062
9 3433.15629941712
10 3442.06276928919
11 3465.96043926154
};
\addlegendentry{Model Learning Only}
\addplot [very thick, green]
table {%
0 303.781469203601
1 671.577784334347
2 1251.44381659134
3 2188.01203726222
4 2513.5478169527
5 2939.4070555617
6 3101.34640617605
7 2891.17588679202
8 3415.57340466607
9 3469.21508140392
10 3330.28904396541
11 3378.83718046318
};
\addlegendentry{Model Rollouts Only}
\end{axis}

\end{tikzpicture}}
\label{ablation:learning and rollout hopper}
}
\hfil
\subfigure[Walker2d]{
\resizebox{0.38\textwidth}{!}{% This file was created with tikzplotlib v0.10.1.
\begin{tikzpicture}

\definecolor{darkgray176}{RGB}{176,176,176}
\definecolor{green}{RGB}{0,128,0}
\definecolor{lightgray204}{RGB}{204,204,204}
\definecolor{purple}{RGB}{128,0,128}

\begin{axis}[
legend cell align={left},
legend style={
  fill opacity=0.8,
  draw opacity=1,
  text opacity=1,
  at={(0.97,0.03)},
  anchor=south east,
  draw=lightgray204
},
tick align=outside,
tick pos=left,
x grid style={darkgray176},
xlabel={Steps},
xmajorgrids,
xmin=-0.95, xmax=19.95,
xtick style={color=black},
xtick={0,4.75,9.5,14.25,19},
xticklabels={0k,50k,100k,150k,200k},
y grid style={darkgray176},
ylabel={Returns},
ymajorgrids,
ymin=-44.872426589316, ymax=5790.39043262598,
ytick style={color=black}
]
\path [draw=red, fill=red, opacity=0.2]
(axis cs:0,310.869080430046)
--(axis cs:0,274.26783970509)
--(axis cs:1,374.439700657643)
--(axis cs:2,691.226589379165)
--(axis cs:3,1403.95426029567)
--(axis cs:4,1611.41996049657)
--(axis cs:5,2587.44518646806)
--(axis cs:6,3254.79856989235)
--(axis cs:7,3486.48263272935)
--(axis cs:8,3589.93343868015)
--(axis cs:9,3710.48140946969)
--(axis cs:10,3980.74670286248)
--(axis cs:11,3858.27722473211)
--(axis cs:12,4545.07957256603)
--(axis cs:13,4254.61062037242)
--(axis cs:14,4545.71354263577)
--(axis cs:15,4329.24431128295)
--(axis cs:16,5005.97693195506)
--(axis cs:17,4672.89823627628)
--(axis cs:18,5144.24731049092)
--(axis cs:19,5068.15770528662)
--(axis cs:19,5525.15121175255)
--(axis cs:19,5525.15121175255)
--(axis cs:18,5470.05309106649)
--(axis cs:17,5249.62477962359)
--(axis cs:16,5278.95831932772)
--(axis cs:15,4787.04863075718)
--(axis cs:14,4792.51993744146)
--(axis cs:13,4684.89691586051)
--(axis cs:12,4817.17484260129)
--(axis cs:11,4451.01315021625)
--(axis cs:10,4286.57893036521)
--(axis cs:9,3964.61781864422)
--(axis cs:8,4100.30285449661)
--(axis cs:7,3772.6035207106)
--(axis cs:6,3718.6857786969)
--(axis cs:5,2844.12704798606)
--(axis cs:4,2089.96698932163)
--(axis cs:3,1721.42116904253)
--(axis cs:2,952.656430253483)
--(axis cs:1,397.438752012594)
--(axis cs:0,310.869080430046)
--cycle;

\path [draw=purple, fill=purple, opacity=0.2]
(axis cs:0,254.82271141427)
--(axis cs:0,220.366794284106)
--(axis cs:1,427.032892349599)
--(axis cs:2,484.613387256256)
--(axis cs:3,574.907004616918)
--(axis cs:4,846.124717270676)
--(axis cs:5,1330.64082020444)
--(axis cs:6,1881.72192525012)
--(axis cs:7,2691.14593214432)
--(axis cs:8,2910.29607998693)
--(axis cs:9,2789.79284663236)
--(axis cs:10,3540.89053101747)
--(axis cs:11,3372.91223858445)
--(axis cs:12,3702.886584907)
--(axis cs:13,3732.51658957207)
--(axis cs:14,3902.89590052245)
--(axis cs:15,3767.79375786435)
--(axis cs:16,3591.49278882001)
--(axis cs:17,3643.05099190985)
--(axis cs:18,3562.58961897768)
--(axis cs:19,4216.83888831274)
--(axis cs:19,4502.87205006605)
--(axis cs:19,4502.87205006605)
--(axis cs:18,4145.77354140132)
--(axis cs:17,3769.82200675689)
--(axis cs:16,3889.38035810719)
--(axis cs:15,4114.82169190061)
--(axis cs:14,4109.40375560722)
--(axis cs:13,4090.54290751113)
--(axis cs:12,3993.41608580747)
--(axis cs:11,3953.67051646616)
--(axis cs:10,3824.43328094269)
--(axis cs:9,3297.24897249677)
--(axis cs:8,3032.93659993295)
--(axis cs:7,3008.61655263654)
--(axis cs:6,2411.25785070395)
--(axis cs:5,1814.61422775855)
--(axis cs:4,954.263531774537)
--(axis cs:3,682.161856768515)
--(axis cs:2,584.6086031204)
--(axis cs:1,497.614693815498)
--(axis cs:0,254.82271141427)
--cycle;

\path [draw=blue, fill=blue, opacity=0.2]
(axis cs:0,331.172573017621)
--(axis cs:0,291.242711835463)
--(axis cs:1,392.90807689135)
--(axis cs:2,526.912730016574)
--(axis cs:3,791.804726836289)
--(axis cs:4,1431.83242276152)
--(axis cs:5,1842.42952599738)
--(axis cs:6,2183.71662288518)
--(axis cs:7,2942.13226493787)
--(axis cs:8,3030.28123167865)
--(axis cs:9,3068.05323784002)
--(axis cs:10,3445.40129974877)
--(axis cs:11,3683.99020964184)
--(axis cs:12,3772.34104642845)
--(axis cs:13,4061.82422924225)
--(axis cs:14,4071.34953769551)
--(axis cs:15,4297.79642460245)
--(axis cs:16,4181.51204450798)
--(axis cs:17,4387.10917599612)
--(axis cs:18,4687.97647262004)
--(axis cs:19,4450.37011275131)
--(axis cs:19,4972.61257084581)
--(axis cs:19,4972.61257084581)
--(axis cs:18,5108.57050872684)
--(axis cs:17,5096.3731634519)
--(axis cs:16,4751.38525149273)
--(axis cs:15,4649.06038593777)
--(axis cs:14,4640.92877861778)
--(axis cs:13,4739.64445294636)
--(axis cs:12,4520.80687985868)
--(axis cs:11,4202.01223448307)
--(axis cs:10,4044.35505032601)
--(axis cs:9,4038.8718259073)
--(axis cs:8,3988.16718262547)
--(axis cs:7,3751.30567010609)
--(axis cs:6,3177.1128210018)
--(axis cs:5,2344.51375174342)
--(axis cs:4,1529.77330786202)
--(axis cs:3,1000.19176085357)
--(axis cs:2,736.23889372502)
--(axis cs:1,417.835440850095)
--(axis cs:0,331.172573017621)
--cycle;

\path [draw=green, fill=green, opacity=0.2]
(axis cs:0,342.871576551904)
--(axis cs:0,283.16959569536)
--(axis cs:1,464.485016894777)
--(axis cs:2,516.705896590553)
--(axis cs:3,706.641500999438)
--(axis cs:4,1241.87798685192)
--(axis cs:5,1777.98577823834)
--(axis cs:6,1981.13729494057)
--(axis cs:7,2528.81958570036)
--(axis cs:8,3032.06737282946)
--(axis cs:9,2880.27823388116)
--(axis cs:10,2937.09991276886)
--(axis cs:11,3530.50373672499)
--(axis cs:12,3308.49848092175)
--(axis cs:13,3650.16303848067)
--(axis cs:14,3885.32316223042)
--(axis cs:15,4251.1589104812)
--(axis cs:16,4122.07112576746)
--(axis cs:17,4284.11273952989)
--(axis cs:18,3885.71414304067)
--(axis cs:19,4173.34490780867)
--(axis cs:19,4704.07867185562)
--(axis cs:19,4704.07867185562)
--(axis cs:18,4579.1457105358)
--(axis cs:17,4434.36092453478)
--(axis cs:16,4257.88256459369)
--(axis cs:15,4480.30475207679)
--(axis cs:14,4424.92239071875)
--(axis cs:13,4134.47981281874)
--(axis cs:12,4140.4294606592)
--(axis cs:11,4246.35436959057)
--(axis cs:10,3798.20609851281)
--(axis cs:9,3828.4226644539)
--(axis cs:8,3728.43748367667)
--(axis cs:7,3242.75925594867)
--(axis cs:6,2777.74101472918)
--(axis cs:5,2339.47266459337)
--(axis cs:4,1628.51797064464)
--(axis cs:3,1029.23521024211)
--(axis cs:2,559.079245714948)
--(axis cs:1,504.670436224474)
--(axis cs:0,342.871576551904)
--cycle;

\addplot [very thick, red]
table {%
0 292.903994564138
1 385.528246400112
2 819.495060554688
3 1569.81599103574
4 1861.84048362766
5 2722.05254653316
6 3486.05824883771
7 3641.62612467158
8 3840.89503483001
9 3842.78204579975
10 4147.19113382598
11 4155.81283303558
12 4691.34130015397
13 4475.3042125548
14 4666.08919632027
15 4564.55569538872
16 5148.89392215021
17 4971.14442976961
18 5305.2327448193
19 5303.89545994349
};
\addlegendentry{PMDL-MBPO}
\addplot [very thick, purple]
table {%
0 237.916438540833
1 460.842878989599
2 530.009961600383
3 623.771807628081
4 903.754364812231
5 1567.35766554305
6 2132.96822837009
7 2842.49772045823
8 2969.72452103888
9 3038.65952564612
10 3677.97660156502
11 3684.59657145359
12 3851.11553378733
13 3927.04314270838
14 4004.24143412764
15 3947.54920524616
16 3742.3185405308
17 3706.18911006491
18 3856.27554856244
19 4364.64897888617
};
\addlegendentry{MBPO}
\addplot [very thick, blue]
table {%
0 312.422754044766
1 405.136507413595
2 617.753810480895
3 893.794047159743
4 1479.43566195095
5 2090.10564601079
6 2665.18049713203
7 3350.40858553358
8 3509.38559566757
9 3523.37583841559
10 3739.16582621701
11 3922.33783756146
12 4138.7837745142
13 4400.97661295461
14 4358.58278729687
15 4471.10473615433
16 4470.84500708523
17 4747.69522396346
18 4899.61658507946
19 4697.57497930253
};
\addlegendentry{Model Learning Only}
\addplot [very thick, green]
table {%
0 313.384381629578
1 484.432487324031
2 539.196094162528
3 857.519555869709
4 1422.04950504526
5 2065.71616026072
6 2353.09332624717
7 2860.98843145748
8 3396.92489513459
9 3405.4858619843
10 3361.23331889625
11 3922.39154154857
12 3750.32110089175
13 3896.75692077359
14 4161.07869437638
15 4367.64942865645
16 4183.19744871102
17 4362.0604083282
18 4224.65991784835
19 4460.54427821301
};
\addlegendentry{Model Rollouts Only}
\end{axis}

\end{tikzpicture}}
\label{ablation:learning and rollout walker}
}
\hfil
\subfigure[Hopper]{
\resizebox{0.38\textwidth}{!}{% This file was created with tikzplotlib v0.10.1.
\begin{tikzpicture}

\definecolor{darkgray176}{RGB}{176,176,176}
\definecolor{darkorange}{RGB}{255,140,0}
\definecolor{green}{RGB}{0,128,0}
\definecolor{lightgray204}{RGB}{204,204,204}
\definecolor{purple}{RGB}{128,0,128}

\begin{axis}[
legend cell align={left},
legend style={
  fill opacity=0.8,
  draw opacity=1,
  text opacity=1,
  at={(0.97,0.03)},
  anchor=south east,
  draw=lightgray204
},
tick align=outside,
tick pos=left,
x grid style={darkgray176},
xlabel={Steps},
xmajorgrids,
xmin=-0.55, xmax=11.55,
xtick style={color=black},
xtick={0,2.75,5.5,8.25,11},
xticklabels={0k,30k,60k,90k,120k},
y grid style={darkgray176},
ylabel={Returns},
ymajorgrids,
ymin=118.4119943202, ymax=3884.14343778805,
ytick style={color=black}
]
\path [draw=red, fill=red, opacity=0.2]
(axis cs:0,521.544028497425)
--(axis cs:0,415.793455964697)
--(axis cs:1,1381.49672129023)
--(axis cs:2,2638.68888512989)
--(axis cs:3,2997.69862739266)
--(axis cs:4,3010.52490039728)
--(axis cs:5,3162.47122451836)
--(axis cs:6,3274.27685443674)
--(axis cs:7,3460.28248554781)
--(axis cs:8,3434.22887136231)
--(axis cs:9,3460.74028393233)
--(axis cs:10,3586.39546479916)
--(axis cs:11,3606.97912923734)
--(axis cs:11,3670.4827662473)
--(axis cs:11,3670.4827662473)
--(axis cs:10,3657.24304801792)
--(axis cs:9,3520.06039553331)
--(axis cs:8,3487.55593780662)
--(axis cs:7,3486.11628355875)
--(axis cs:6,3377.21882228184)
--(axis cs:5,3336.03206231296)
--(axis cs:4,3120.84787876797)
--(axis cs:3,3128.8140737327)
--(axis cs:2,2894.98542820982)
--(axis cs:1,1685.55657448333)
--(axis cs:0,521.544028497425)
--cycle;

\path [draw=blue, fill=blue, opacity=0.2]
(axis cs:0,382.465539195593)
--(axis cs:0,338.287904797133)
--(axis cs:1,938.83182440536)
--(axis cs:2,2519.08640133425)
--(axis cs:3,2862.55846478993)
--(axis cs:4,2911.40948946963)
--(axis cs:5,2892.76385572879)
--(axis cs:6,3063.7956334496)
--(axis cs:7,3275.04324873293)
--(axis cs:8,3255.3913319044)
--(axis cs:9,3300.08244236984)
--(axis cs:10,3504.98024185566)
--(axis cs:11,3555.59611461029)
--(axis cs:11,3712.97382672132)
--(axis cs:11,3712.97382672132)
--(axis cs:10,3610.23782328735)
--(axis cs:9,3514.53655596141)
--(axis cs:8,3476.03916296573)
--(axis cs:7,3441.11532088992)
--(axis cs:6,3478.52383056901)
--(axis cs:5,3163.34513893133)
--(axis cs:4,3045.18249959508)
--(axis cs:3,2989.3357100179)
--(axis cs:2,2754.98141876094)
--(axis cs:1,1309.38639523687)
--(axis cs:0,382.465539195593)
--cycle;

\path [draw=purple, fill=purple, opacity=0.2]
(axis cs:0,351.482414284978)
--(axis cs:0,323.573739753112)
--(axis cs:1,851.894572089973)
--(axis cs:2,1673.82658942546)
--(axis cs:3,2237.13646967555)
--(axis cs:4,2458.81458164664)
--(axis cs:5,2351.68387899034)
--(axis cs:6,2376.35143131338)
--(axis cs:7,2478.68660985713)
--(axis cs:8,2303.89033248336)
--(axis cs:9,2614.33379175823)
--(axis cs:10,2988.52464721779)
--(axis cs:11,3108.25114169325)
--(axis cs:11,3372.71492380159)
--(axis cs:11,3372.71492380159)
--(axis cs:10,3236.37841937531)
--(axis cs:9,3102.13977907657)
--(axis cs:8,2944.54165910939)
--(axis cs:7,2792.76044671798)
--(axis cs:6,2680.94240529891)
--(axis cs:5,2738.2013168403)
--(axis cs:4,2643.79696678082)
--(axis cs:3,2668.01359941407)
--(axis cs:2,1900.0781543639)
--(axis cs:1,1139.00306068484)
--(axis cs:0,351.482414284978)
--cycle;

\path [draw=green, fill=green, opacity=0.2]
(axis cs:0,307.583211512793)
--(axis cs:0,289.58160538692)
--(axis cs:1,974.971909823793)
--(axis cs:2,1574.65484449469)
--(axis cs:3,2107.19847328914)
--(axis cs:4,2106.63480633636)
--(axis cs:5,2252.44510093119)
--(axis cs:6,2320.51522270452)
--(axis cs:7,2113.97949392871)
--(axis cs:8,2659.97226863466)
--(axis cs:9,2719.86101526518)
--(axis cs:10,2702.25251821974)
--(axis cs:11,2515.13445081507)
--(axis cs:11,3095.70541781028)
--(axis cs:11,3095.70541781028)
--(axis cs:10,2989.41417789902)
--(axis cs:9,2923.98133332112)
--(axis cs:8,2705.80660648342)
--(axis cs:7,2532.37998719288)
--(axis cs:6,2770.89659679854)
--(axis cs:5,2553.11178038202)
--(axis cs:4,2400.91331403565)
--(axis cs:3,2388.19924361549)
--(axis cs:2,2073.52744274569)
--(axis cs:1,1195.46032147117)
--(axis cs:0,307.583211512793)
--cycle;

\path [draw=darkorange, fill=darkorange, opacity=0.2]
(axis cs:0,371.972622068612)
--(axis cs:0,307.782289152145)
--(axis cs:1,572.433045240202)
--(axis cs:2,1007.10292325416)
--(axis cs:3,1582.42326266462)
--(axis cs:4,1894.790073539)
--(axis cs:5,1842.52193752654)
--(axis cs:6,1779.78257960308)
--(axis cs:7,1582.74662169948)
--(axis cs:8,1830.32664796369)
--(axis cs:9,1613.72861999931)
--(axis cs:10,1899.8116293251)
--(axis cs:11,2027.16539903316)
--(axis cs:11,3277.71092071071)
--(axis cs:11,3277.71092071071)
--(axis cs:10,2942.34890818112)
--(axis cs:9,2646.37168492688)
--(axis cs:8,2926.09435456071)
--(axis cs:7,2569.19605123419)
--(axis cs:6,2095.43294137784)
--(axis cs:5,2368.40175724488)
--(axis cs:4,2191.6466153536)
--(axis cs:3,1757.23407068877)
--(axis cs:2,1349.7251832139)
--(axis cs:1,813.500246905887)
--(axis cs:0,371.972622068612)
--cycle;

\addplot [very thick, red]
table {%
0 472.795918513284
1 1532.76886841273
2 2760.53860482674
3 3066.85442348338
4 3062.08467228887
5 3257.88476557084
6 3325.08628815651
7 3472.77187749469
8 3462.49976707919
9 3492.24666302517
10 3622.88706403089
11 3640.77287939842
};
\addlegendentry{PMDL-MBPO (rate = 0.02)}
\addplot [very thick, blue]
table {%
0 360.256319190529
1 1116.56842867588
2 2643.21146818282
3 2923.61572004125
4 2981.27687144963
5 3031.96809330851
6 3277.46500553233
7 3360.00267086016
8 3373.89845608834
9 3417.98939177545
10 3558.69226221491
11 3631.935386132
};
\addlegendentry{rate = 0.1}
\addplot [very thick, purple]
table {%
0 337.117501155928
1 985.102317837034
2 1791.9471185318
3 2444.72783877875
4 2555.4302884194
5 2546.27730248893
6 2525.53489247772
7 2638.94267016007
8 2646.59202939922
9 2864.8805438536
10 3118.2257604335
11 3250.29097161096
};
\addlegendentry{rate = 0.3}
\addplot [very thick, green]
table {%
0 298.52849885021
1 1081.31807066517
2 1838.23084454075
3 2248.5280387946
4 2265.41357916215
5 2408.01723528751
6 2533.37655110686
7 2321.08121028329
8 2683.24180732098
9 2830.70294583019
10 2855.54298138365
11 2831.54246125811
};
\addlegendentry{rate = 0.5}
\addplot [very thick, darkorange]
table {%
0 337.314447456326
1 692.472962673412
2 1166.97199357868
3 1666.14513649715
4 2036.4600012737
5 2086.84910643989
6 1945.00234218886
7 2119.33271900705
8 2365.59598385101
9 2185.83904864251
10 2400.47175316989
11 2640.51197642269
};
\addlegendentry{rate = 0.7}
\end{axis}

\end{tikzpicture}}
\label{ablation: alpha}
}
\caption{(a) and (b): Ablation study of adjusted policy mixture distribution on model learning and sampling initial states for model rollouts.
(c): Ablation study of current policy proportion rate.}
\end{figure}
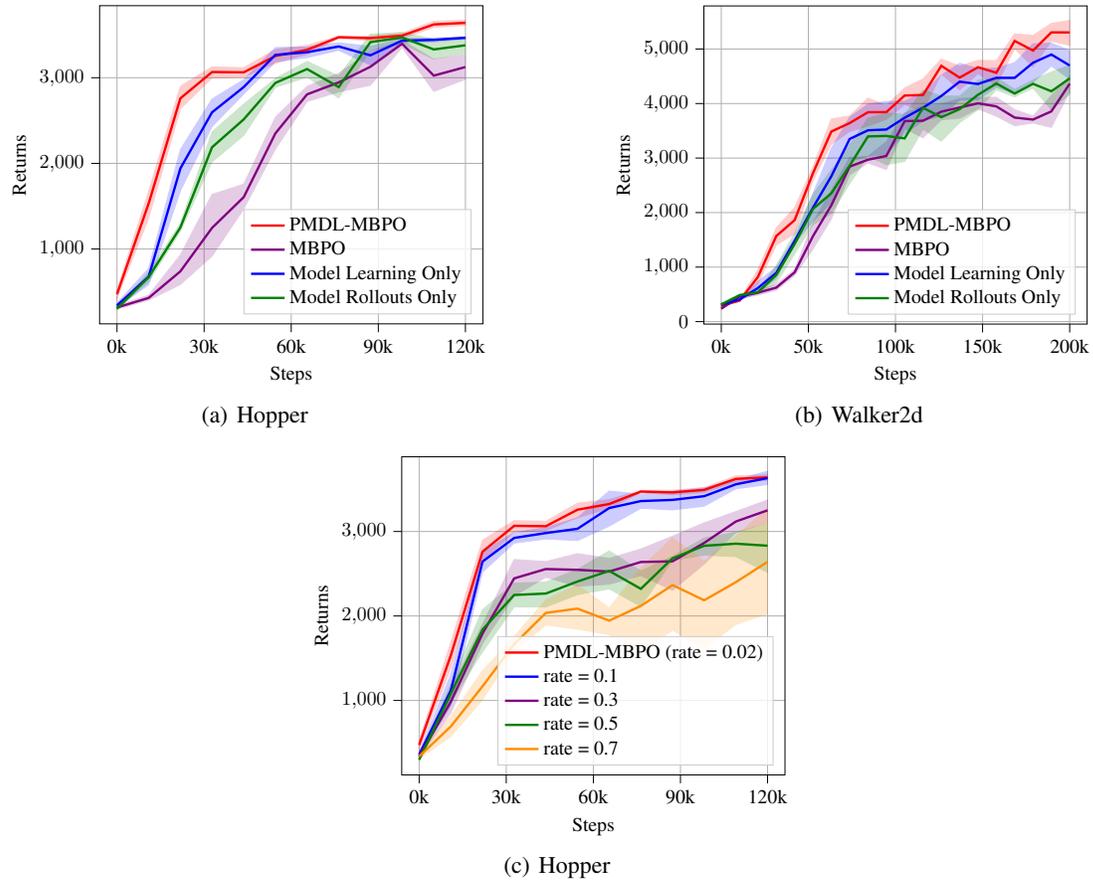

We find that using the adjusted policy mixture distribution only for model learning or model rollouts initial states sampling both improves the performance in Hopper and Walker2d compared to MBPO.
However, the improvement of only using the adjusted policy mixture distribution for model rollouts initial states sampling in Walker2d is not very significant.
Besides, the improvement of using the adjusted policy mixture distribution for model learning is better than using that for model rollouts initial states sampling, but both of them are worse than PDML.
This indicates two things.
First, model learning is more important than model rollouts initial states sampling, because even the initial state distribution obeys the state-action visitation distribution of current policy, the model-generated samples will still be inaccurate if the learned dynamics model is not accurate enough for the current policy.
Second, to achieve the best performance, sample distribution for model learning and sample distribution for model rollouts initial states should be synergistic; that is, the training data for training the dynamics model and the initial states of model rollouts should obey the same distribution, so that the model prediction error can be minimized.

\subsection{Ablation Study of Current Policy Proportion Rate}

We conducted experiments to explore the impact of current policy proportion rate on the performance of our method.
The $\alpha$ in Eq.~\ref{current_policy_weight} equals to current policy proportion rate divided by 1 minus current policy proportion rate.
As shown in Figure~\ref{ablation: alpha}, when the current policy proportion rate is small (0.02 and 0.1), the policy mixture distribution will not be too inclined to the current policy, so the model can learn a good transition dynamics. 
When the current policy proportion rate is too large (0.3, 0.5, and 0.7), the learned dynamics capture information about the underlying transition too locally, resulting in performance decrease. 
Therefore, we recommend that the selection of the current policy proportion rate should not be greater than 0.1.

\subsection{One-step Error in Four Environments}

As an extension of Figure~\ref{model error fig}, we provide the one-step model prediction error curve in this section.
The results are shown in Figure~\ref{app: model error fig}.

\begin{figure}[!htbp]
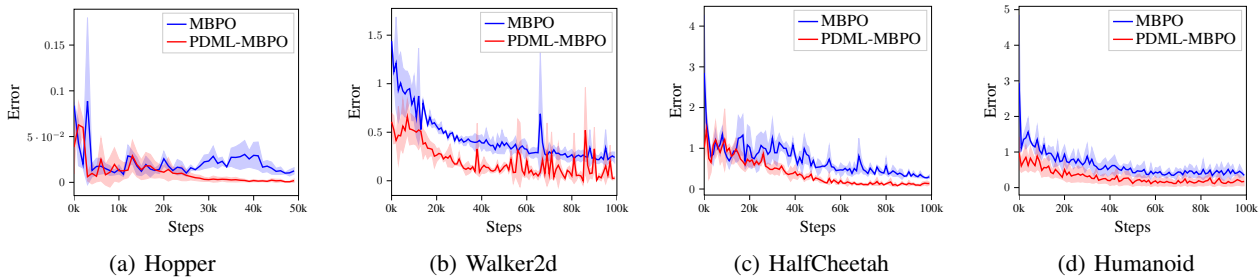

\centering
\subfigure[Hopper]{
\resizebox{0.24\textwidth}{!}{\input{exp_results/PDML_hopper_one_step}}
\label{app:hopper_onestep_error}
}
\hfil
\subfigure[Walker2d]{
\resizebox{0.225\textwidth}{!}{\input{exp_results/PDML_half_one_step}}
\label{app:half_onestep_error}
}
\hfil
\subfigure[HalfCheetah]{
\resizebox{0.22\textwidth}{!}{\input{exp_results/PDML_walker_one_step}}
\label{app:walker_onestep_error}
}
\hfil
\subfigure[Humanoid]{
\resizebox{0.22\textwidth}{!}{\input{exp_results/PDML_human_one_step.tex}}
\label{app:human_onestep_error}
}
\caption{One-step error curves in Hopper, Walker2d, HalfCheetah, and Humanoid.
}
\label{app: model error fig}
\end{figure}
\section{Implementation}

\subsection{Implementation Details}\label{app:subsec:imp-details}
%\textbf{Implementation Details.}

We implement \ouralgo-MBPO based on the PyTorch-version MBPO \cite{liu2020re}.
We also set the ensemble size of \ouralgo-MBPO to be the same as MBPO, which is 7. 
The warm-up samples are collected through interaction with the real environment for 5000 steps using a randomly chosen policy.
After the warm-up, we train the dynamics model and update the lifetime weight every 250 interaction steps.
% The size of real sample buffer and model sample buffer are 1 million and 100k, respectively. 
We set the current policy proportion to be 0.02 and $\alpha$ equals $0.02/0.98$.
One thing that needs to be noticed is the rollout horizon setting.
As introduced in MBPO \cite{janner2019trust}, the rollout horizon should start at a short horizon and increase linearly with the interaction epoch.
$[a, b, x, y]$ denotes a thresholded linear function, $i.e$. at epoch $e$, rollout horizon is $h = \min(\max(x+ \frac{e-a}{b-a}(y-x),x),y)$.
We set the rollout horizon to be the same as used in the MBPO paper, as shown in Table \ref{tab_rollout}.
Other hyper-parameter settings are shown in Table \ref{tab2}.
For MBPO\footnote{\href{https://github.com/Xingyu-Lin/mbpo_pytorch}{https://github.com/Xingyu-Lin/mbpo\_pytorch}}, AMPO\footnote{\href{https://github.com/RockySJ/ampo}{ https://github.com/RockySJ/ampo}}, 
% PETS\footnote{\href{https://github.com/kchua/handful-of-trials}{https://github.com/kchua/handful-of-trials}},
VaGraM\footnote{\href{https://github.com/pairlab/vagram}{https://github.com/pairlab/vagram}}
SAC\footnote{\href{https://github.com/pranz24/pytorch-soft-actor-critic}{https://github.com/pranz24/pytorch-soft-actor-critic}}, 
and REDQ\footnote{\href{https://github.com/watchernyu/REDQ}{https://github.com/watchernyu/REDQ}},  we use their open source implementations.
We evaluate \ouralgo-MBPO and other baselines on four MuJoCo-v2 continuous control environments \cite{todorov2012mujoco} with a maximum horizon of 1000, including HalfCheetah, Hopper, Walker2d, and Humanoid.
For Humanoid, we use the modified version introduced by MBPO \cite{janner2019trust}.
All experiments are conducted using a single NVIDIA TITAN X Pascal GPU.

\begin{table*}[!htbp]
\setlength{\abovecaptionskip}{0.2cm}
\centering
\caption{Rollout horizon settings for \ouralgo}
\begin{tabular}{c|c|c|c|c|c|c|c}
\toprule 
Walker2d & Hopper & Humanoid & HalfCheetah & Pusher & Ant\\
\midrule
1 & [1, 15, 20, 100] & [1, 25, 20, 300] & 1 & 1 & [1, 25, 20, 100] \\
\bottomrule
\end{tabular}
\label{tab_rollout}
\end{table*}

\begin{table*}[!htbp]
\setlength{\abovecaptionskip}{0.2cm}
\centering
\caption{Hyper-parameter settings for \ouralgo}
\begin{tabular}{c|c}
\toprule 
Parameter&Value \\
\midrule
Dynamics model ensemble size & 7 \\
Dynamics model layers & 4 \\
Actor and critic layers  & 3 \\
Dynamics model hidden units & 200 \\
Actor and critic hidden units & 256 \\
Learning rate & $3 \cdot 10^{-4}$ \\
Batch size & 256 \\
Optimizer & Adam \\
Activation function & ReLU \\
Real sample buffer size & $10^6$ \\
Model sample buffer size & $10^6$ \\
Real sample ratio & 0.05 \\
Policy updates per environment step & 20 \\
Environment steps between model training & 250 \\
\bottomrule
\end{tabular}
\label{tab2}
\end{table*}

For the experiment of MaPER in Sec \ref{exp5.3}, we use their open-source code in the supplementary material on openreview \footnote{\href{https://openreview.net/forum?id=WuEiafqdy9H}{ https://openreview.net/forum?id=WuEiafqdy9H}}.
% However, we find a bug in their code that comes from the PyTorch-version MBPO implementation, i.e., the same environment is used for policy training and policy evaluation. 
% This causes the rollout length to exceed the 1000-step limit during evaluation, resulting in the performance of the policy being much higher than the 1000-step performance. 
% We fix this bug and conduct the experiment, so the results of MaPER are lower than those reported in their paper.

%%%%%%%%%%%%%%%%%%%%%%%%%%%%%%%%%%%%%%%%%%%%%%%%%%%%%%%%%%%%%%%%%%%%%%%%%%%%%%%
%%%%%%%%%%%%%%%%%%%%%%%%%%%%%%%%%%%%%%%%%%%%%%%%%%%%%%%%%%%%%%%%%%%%%%%%%%%%%%%

\end{document}